\documentclass[12pt]{article}
\usepackage{amsmath}
\usepackage{floatpag}
\usepackage{amssymb}
\usepackage{graphicx}
\usepackage{caption}
\usepackage{subcaption}
\usepackage{enumerate}
\usepackage{natbib}
\usepackage{amsthm}
\usepackage{url} 
\usepackage{mathtools}
\usepackage{bbm}
\usepackage[normalem]{ulem}
\usepackage{threeparttable}
\usepackage{footnote}
\usepackage{tablefootnote}
\usepackage{adjustbox}
\usepackage{makecell}
\usepackage{booktabs}
\usepackage{multirow}
\usepackage{dsfont}
\usepackage{amsfonts}
\usepackage[ruled,linesnumbered]{algorithm2e}
\usepackage{bm}
\usepackage{tocloft}
\usepackage[title]{appendix}
\usepackage{textcomp}
\usepackage{eufrak}
\usepackage[mathscr]{eucal}
\usepackage{tikz}
\usetikzlibrary{arrows,chains,matrix,positioning,scopes}
\usepackage{hyperref}
\hypersetup{colorlinks=true,linkcolor=blue, citecolor=blue, linktocpage}
\mathtoolsset{showonlyrefs}
\allowdisplaybreaks
\newcommand{\blind}{1}

\makeatletter
\tikzset{join/.code=\tikzset{after node path={%
\ifx\tikzchainprevious\pgfutil@empty\else(\tikzchainprevious)%
edge[every join]#1(\tikzchaincurrent)\fi}}}
\makeatother

\tikzset{>=stealth',every on chain/.append style={join},
         every join/.style={}}
\tikzstyle{labeled}=[execute at begin node=$\scriptstyle,
   execute at end node=$]

\makeatletter
\newcommand*{\Rom}[1]{\expandafter\@slowromancap\romannumeral #1@}
\makeatother
\newcommand*{\rom}[1]{\romannumeral #1}

\DeclareMathOperator*{\argmin}{arg\,min}

\newcommand{\tp}{\mathbb{P}}

\newcommand{\te}{\mathbb{E}}
\newcommand{\cov}{\textup{Cov}}

\newcommand{\var}{\textup{Var}}

\newcommand{\bxi}{\bm{\xi}}
\newcommand{\hb}{\hat{b}}
\newcommand{\lambdak}[1]{\lambda^{(#1)}}
\newcommand{\bdelta}{\bm{\delta}}
\newcommand{\bdeltak}[1]{\bm{\delta}^{(#1)}}

\newcommand{\bdeltaah}{\bm{\delta}^{\mathcal{A}_h}}
\newcommand{\hdeltaa}{\hat{\bm{\delta}}^{\mathcal{A}}}
\newcommand{\hdeltaah}{\hat{\bm{\delta}}^{\mathcal{A}_h}}

\newcommand{\infnorma}[1]{\left\|#1\right\|_{\infty}}

\newcommand{\onenorma}[1]{\left\|#1\right\|_{1}}
\newcommand{\twonorma}[1]{\left\|#1\right\|_{2}}
\newcommand{\infnorm}[1]{\|#1\|_{\infty}}
\newcommand{\twonorm}[1]{\|#1\|_{2}}
\newcommand{\onenorm}[1]{\|#1\|_{1}}
\newcommand{\zeronorm}[1]{\|#1\|_{0}}
\newcommand{\diag}{\text{diag}}
\newcommand{\maxnorma}[1]{\left\|#1\right\|_{\max}}
\newcommand{\maxnorm}[1]{\|#1\|_{\max}}

\newcommand{\norma}[1]{\left|#1\right|}
\newcommand{\norm}[1]{|#1|}
\newcommand{\bbeta}{\bm{\beta}}
\newcommand{\hbeta}{\hat{\bm{\beta}}}

\newcommand{\htau}{\hat{\tau}}

\newcommand{\brho}{\bm{\varrho}}

\newcommand{\hrho}{\hat{\bm{\varrho}}}

\newcommand{\tlambda}{\tilde{\lambda}}
\newcommand{\bgamma}{\bm{\gamma}}
\newcommand{\bgammaa}{\bm{\gamma}^{\ma}}
\newcommand{\bgammaah}{\bm{\gamma}^{\mah}}
\newcommand{\bgammak}[1]{\bm{\gamma}^{(#1)}}

\newcommand{\hgammak}[1]{\hat{\gamma}^{(#1)}}
\newcommand{\hbgammak}[1]{\hat{\bm{\gamma}}^{(#1)}}
\newcommand{\hgammaa}{\hat{\bm{\gamma}}^{\ma}}
\newcommand{\hgammaah}{\hat{\bm{\gamma}}^{\mah}}
\newcommand{\hgamma}{\hat{\bm{\gamma}}}

\newcommand{\bbetak}[1]{\bm{\beta}^{(#1)}}
\newcommand{\hbetak}[1]{\hat{\bm{\beta}}^{(#1)}}
\newcommand{\bx}{\bm{x}}
\newcommand{\bxk}[1]{\bm{x}^{(#1)}}
\newcommand{\yk}[1]{y^{(#1)}}
\newcommand{\byk}[1]{\bm{y}^{(#1)}}
\newcommand{\bXk}[1]{\bm{X}^{(#1)}}
\newcommand{\bu}{\bm{u}}

\newcommand{\hua}{\hat{\bm{u}}^{\mathcal{A}}}
\newcommand{\huah}{\hat{\bm{u}}^{\mathcal{A}_h}}
\newcommand{\bv}{\bm{v}}

\newcommand{\hvah}{\hat{\bm{v}}^{\mathcal{A}_h}}

\newcommand{\bmetak}[1]{\bm{\eta}^{(#1)}}
\newcommand{\bw}{\bm{w}}
\newcommand{\bwk}[1]{\bm{w}^{(#1)}}
\newcommand{\bWk}[1]{\bm{W}^{(#1)}}
\newcommand{\bwa}{\bm{w}^{\mathcal{A}}}
\newcommand{\bwah}{\bm{w}^{\mathcal{A}_h}}

\newcommand{\hwa}{\hat{\bm{w}}^{\mathcal{A}}}
\newcommand{\hwah}{\hat{\bm{w}}^{\mathcal{A}_h}}
\newcommand{\ma}{\mathcal{A}}
\newcommand{\mah}{\mathcal{A}_h}
\newcommand{\mach}{\mathcal{A}^c_h}
\newcommand{\mac}{\mathcal{A}^c}
\newcommand{\md}{\mathcal{D}}

\def\boxit#1{\vbox{\hrule\hbox{\vrule\kern6pt\vbox{\kern6pt#1\kern6pt}\kern6pt\vrule}\hrule}}

\newcommand{\na}{n_{\mathcal{A}}}
\newcommand{\nah}{n_{\mathcal{A}_h}}
\newcommand{\transet}{\{0\}\cup \mah}
\newcommand{\bpsi}{\bm{\psi}}

\newcommand{\Kah}{K_{\mah}}
\newcommand{\by}{\bm{y}}
\newcommand{\bX}{\bm{X}}

\newcommand{\<}{\langle}
\renewcommand{\>}{\rangle}
\newcommand{\bTheta}{\bm{\Theta}}
\newcommand{\hTheta}{\widehat{\bm{\Theta}}}
\newcommand{\bSigma}{\bm{\Sigma}}
\newcommand{\hSigma}{\widehat{\bm{\Sigma}}}
\newcommand{\bSigmak}[1]{\bm{\Sigma}^{(#1)}}
\newcommand{\hSigmak}[1]{\widehat{\bm{\Sigma}}^{(#1)}}

\newcommand\smallo{
  \mathchoice
    {{\scriptstyle\mathcal{O}}}
    {{\scriptstyle\mathcal{O}}}
    {{\scriptscriptstyle\mathcal{O}}}
    {\scalebox{.7}{$\scriptscriptstyle\mathcal{O}$}}
}

\newtheorem{theorem}{Theorem}
\newtheorem{lemma}{Lemma} 
\newtheorem{proposition}{Proposition} 
\newtheorem{remark}{Remark}
\newtheorem{corollary}{Corollary}

\newtheorem{assumption}{Assumption}

\begin{document}

\def\spacingset#1{\renewcommand{\baselinestretch}%
{#1}\small\normalsize} \spacingset{1}

 \title{\bf Transfer Learning under High-dimensional Generalized Linear Models}
  \date{}
\if1\blind
{
 
  \author{Ye Tian\\
    Department of Statistics\\ Columbia University\\
    and \\
    Yang Feng \\
    Department of Biostatistics, School of Global Public Health\\New York University}
   
  \maketitle
} \fi

\if0\blind
{
  \bigskip
  \bigskip
  \bigskip

  \maketitle
} \fi

\bigskip
\begin{abstract}
	In this work, we study the transfer learning problem under high-dimensional generalized linear models (GLMs), which aim to improve the fit on \textit{target} data by borrowing information from useful \textit{source} data. Given which sources to transfer, we propose a transfer learning algorithm on GLM, and derive its $\ell_1/\ell_2$-estimation error bounds as well as a bound for a prediction error measure. The theoretical analysis shows that when the target and source are sufficiently close to each other, these bounds could be improved over those of the classical penalized estimator using only target data under mild conditions. When we don't know which sources to transfer, an \textit{algorithm-free} transferable source detection approach is introduced to detect informative sources. The detection consistency is proved under the high-dimensional GLM transfer learning setting. We also propose an algorithm to construct confidence intervals of each coefficient component, and the corresponding theories are provided. Extensive simulations and a real-data experiment verify the effectiveness of our algorithms. We implement the proposed GLM transfer learning algorithms in a new R package \texttt{glmtrans}, which is available on CRAN.
\end{abstract}

\vfill

\newpage
\spacingset{1.5} 

\addtocontents{toc}{\protect\setcounter{tocdepth}{0}}
\section{Introduction}\label{sec:intro}

Nowadays, a great deal of machine learning algorithms has been successfully applied in our daily life. Many of these algorithms require sufficient training data to perform well, which sometimes can be limited. For example, from an online merchant's view, it could be difficult to collect enough personal purchase data for predicting the customers' purchase behavior and recommending corresponding items. However, in many cases, some related datasets may be available in addition to the limited data for the original task. In the merchant-customer example, we may also have the customers' clicking data in hand, which is not exactly the same as but shares similarities with the purchase data. How to use these additional data to help with the original target task motivates a well-known concept in computer science: \textit{transfer learning} \citep{torrey2010transfer, weiss2016survey}. As its name indicates, in a transfer learning problem, we aim to transfer some useful information from similar tasks (\textit{sources}) to the original task (\textit{target}), in order to boost the performance on the target. To date, transfer learning has been widely applied in a number of machine learning applications, including the customer review classification \citep{pan2009survey}, medical diagnosis \citep{hajiramezanali2018bayesian}, and ride dispatching in ride-sharing platforms \citep{wang2018deep}, etc. Compared with the rapidly growing applications, there has been little discussion about the theoretical guarantee of transfer learning. Besides, although transfer learning has been prevailing in computer science community for decades, far less attention has been paid to it among statisticians. More specifically, transfer learning can be promising in the \textit{high-dimensional} data analysis, where the sample size is much less than the dimension with some sparsity structure in the data \citep{tibshirani1996regression}. The impact of transfer learning in high-dimensional generalized linear models (GLMs) with sparsity structure is not quite clear up to now. In this paper, we are trying to fill the gap by developing  transfer learning tools in high-dimensional GLM inference problem, and providing corresponding theoretical guarantees.

Prior to our paper, there are a few pioneering works exploring transfer learning under the high-dimensional setting. \cite{bastani2021predicting} studied the single-source case when the target data comes from a high-dimensional GLM with limited sample size while the source data size is sufficiently large than the dimension. A two-step transfer learning algorithm was developed, and the estimation error bound was derived when the contrast between target and source coefficients is $\ell_0$-sparse. \cite{li2021transfer} further explored the multi-source high-dimensional linear regression problem where both target and source samples are high-dimensional. The $\ell_2$-estimation error bound under $\ell_q$-regularization ($q \in [0, 1]$) was derived and proved to be minimax optimal under some conditions. In \cite{li2020transfer2}, the analysis was extended to the Gaussian graphical models  with false discovery rate control. Other related research on transfer learning with theoretical guarantee includes the non-parametric classification model \citep{cai2021transfer, reeve2021adaptive} and the analysis under general functional classes via  transfer exponents \citep{hanneke2020no, hanneke2020value}, etc. In addition, during the past few years, there have been some related works studying parameter sharing under the regression setting. For instance, \cite{chen2015data} and \cite{zheng2019data} developed the so-called ``data enriched model" for linear and logistic regression under a single-source setting, where the properties of the oracle tuned estimator with a quadratic penalty were studied. \cite{gross2016data} and \cite{ollier2017regression} explored the so-called ``data shared Lasso" under the multi-task learning setting, where $\ell_1$ penalties of all contrasts are considered.

In this work, we contribute to transfer learning under a high-dimensional context from three perspectives. First, we extend the results of \cite{bastani2021predicting} and \cite{li2021transfer}, by proposing multi-source transfer learning algorithms on generalized linear models (GLMs) and we assume both target and source data to be high-dimensional. We assume the contrast between target and source coefficients to be $\ell_1$-sparse, which differs from the $\ell_0$-sparsity considered in \cite{bastani2021predicting}. The theoretical analysis shows that when the target and source are sufficiently close to each other, the estimation error bound of target coefficients could be improved over that of the classical penalized estimator using only target data under mild conditions. Moreover, the error rate is shown to be minimax optimal under certain conditions. To the best of our knowledge, this is the first study of the multi-source transfer learning framework under the high-dimensional GLM setting. Second, as we mentioned, transferring sources that are close to the target can bring benefits. However, some sources might be far away from the target, and transferring them can be harmful. This phenomenon is often called \textit{negative transfer} in literature \citep{torrey2010transfer}. We will show the impact of negative transfer in simulation studies in Section \ref{subsec:simulations}. To avoid this issue, we develop an \textit{algorithm-free} transferable source detection algorithm, which can help identify informative sources. And with certain conditions satisfied, the algorithm is shown to be able to distinguish useful sources from useless ones. Third, all aforementioned works of transfer learning on high-dimensional regression only focus on the point estimate of the coefficient, which is not sufficient for statistical inference. How transfer learning can benefit the confidence interval construction remains unclear. We propose an algorithm on the basis of our two-step transfer learning procedure and nodewise regression \citep{van2014asymptotically}, to construct the confidence interval for each coefficient component. The corresponding asymptotic theories are established.

The rest of this paper is organized as follows. Section \ref{sec:method} first introduces GLM basics and transfer learning settings under high-dimensional GLM, then presents a general algorithm (where we know which sources are useful) and the transferable source detection algorithm (where useful sources are automatically detected). At the end of Section \ref{sec:method}, we develop an algorithm to construct confidence intervals. Section \ref{sec:theory} provides the theoretical analysis on the algorithms, including $\ell_1$ and $\ell_2$-estimation error bounds of the general algorithm, detection consistency property of the transferable source detection algorithm, and asymptotic theories for the confidence interval construction. We conduct extensive simulations and a real-data study in Section \ref{sec:numerical}, and the results demonstrate the effectiveness of our GLM transfer learning algorithms. In Section \ref{sec:discussions}, we review our contributions and shed light on some interesting future research directions. Additional simulation results and theoretical analysis, as well as all the proofs, are relegated to supplementary materials.

\section{Methodology}\label{sec:method}
We first introduce some notations to be used throughout the paper. We use bold capitalized letters (e.g. $\bX$, $\bm{A}$) to denote matrices, and use bold little letters (e.g. $\bx$, $\by$) to denote vectors. For a $p$-dimensional vector $\bx = (x_1,\ldots, x_p)^T$, we denote its $\ell_q$-norm as $\|\bx\| = (\sum_{i=1}^p \norm{x_i}^q)^{1/q}$ ($q \in (0, 2]$), and $\ell_0$-``norm" $\zeronorm{\bx} = \#\{j: x_j \neq 0\}$. For a matrix $\bm{A}_{p \times q} = [a_{ij}]_{p \times q}$, its 1-norm, 2-norm, $\infty$-norm and max-norm are defined as $\onenorm{\bm{A}} = \sup_{j}\sum_{i=1}^p |a_{ij}|$, $\twonorm{\bm{A}} = \max_{\bx: \twonorm{\bx}=1}\twonorm{\bm{A}\bx}$, $\infnorm{\bm{A}} = \sup_{i}\sum_{j=1}^q |a_{ij}|$ and $\maxnorm{\bm{A}} = \sup_{i,j}|a_{ij}|$, respectively. For two non-zero real sequences $\{a_n\}_{n=1}^{\infty}$ and $\{b_n\}_{n=1}^{\infty}$, we use $a_n \ll b_n$, $b_n \gg a_n$ or $a_n = \smallo(b_n)$ to represent $|a_n/b_n| \rightarrow 0$ as $n \rightarrow \infty$. And $a_n \lesssim b_n$ or $a_n = \mathcal{O}(b_n)$ means $\sup_n|a_n/b_n| < \infty$. Expression $a_n \asymp b_n$ means that $a_n/b_n$ converges to some positive constant. For two random variable sequences $\{x_n\}_{n=1}^{\infty}$ and $\{y_n\}_{n=1}^{\infty}$, notation $x_n \lesssim_p y_n$ or $x_n = \mathcal{O}_p(y_n)$ means that for any $\epsilon > 0$, there exists a positive constant $M$ such that $\sup_n\tp(|x_n/y_n|> M) \leq \epsilon$. And for two real numbers $a$ and $b$, we use $a \vee b$ and $a\wedge b$ to represent $\max(a, b)$ and $\min(a, b)$, respectively. Without specific notes, the expectation $\te$, variance $\var$, and covariance $\cov$ are calculated based on all randomness.

\subsection{Generalized linear models (GLMs)}\label{subsec:glm}
Given the predictors $\bx \in \mathbb{R}^p$, if the response $y$ follows the generalized linear models (GLMs), then its conditional distribution takes the form
\begin{equation}\label{eq: glm def}
	y|\bx \sim \tp(y|\bx) = \rho(y)\exp\{y\bx^T\bw - \psi(\bx^T\bw)\},
\end{equation}
where $\bw \in \mathbb{R}^p$ is the coefficient, $\rho$ and $\psi$ are some known univariate functions. $\psi'(\bx^T\bw) = \te(y|\bx)$ is called the \textit{inverse link function} \citep{mccullagh1989generalized}. Another important property is that $\var(y|\bx)=\psi''(\bx^T\bw)$, which follows from the fact that the distribution belongs to the exponential family. It is $\psi$ that characterizes different GLMs. For example, in linear model with Gaussian noise, we have a continuous response $y$ and $\psi(u)=\frac{1}{2}u^2$; in the logistic regression model, $y$ is binary and $\psi(u)=\log(1+e^u)$; and in Poisson regression model, $y$ is a nonnegative integer and $\psi(u)=e^u$. For most GLMs, $\psi$ is strictly convex and infinitely differentiable.

\subsection{Target data, source data, and transferring level}\label{subsec:target source trasnfer}
In this paper, we consider the following multi-source transfer learning problem. Suppose we have the \textit{target} data set $(\bX^{(0)}, \by^{(0)})$ and $K$ \textit{source} data sets with the $k$-th source denoted as $(\bX^{(k)}, \by^{(k)})$, where $\bX^{(k)} \in \mathbb{R}^{n_k \times p}$, $\by^{(k)} \in \mathbb{R}^{n_k}$ for $k = 0, \ldots, K$. The $i$-th row of $\bX^{(k)}$ and the $i$-th element of $\by^{(k)}$ are denoted as $\bx_{i}^{(k)}$ and $y_{i}^{(k)}$, respectively. The goal is to transfer useful information from source data to obtain a better model for the target data. We assume the responses in the target and source data all follow the generalized linear model:
\begin{equation}\label{eq: glm}
	y^{(k)}|\bx \sim \tp(y|\bx) = \rho(y)\exp\{y\bx^T\bwk{k} - \psi(\bx^T\bwk{k})\},
\end{equation}
for $k = 0, \ldots, K$, with possibly different coefficient $\bwk{k} \in \mathbb{R}^p$, the predictor $\bx \in \mathbb{R}^p$, and some known univariate functions $\rho$ and $\psi$. Denote the target parameter as $\bbeta = \bwk{0}$. Suppose the target model is $\ell_0$-sparse, which satisfies $\zeronorm{\bbeta} =s \ll p$. This means that only $s$ of the $p$ variables contribute to the target response. Intuitively, if $\bwk{k}$ is close to $\bbeta$, the $k$-th source could be useful for transfer learning.

Define the $k$-th contrast $\bdeltak{k} = \bbeta - \bwk{k}$ and we say $\onenorm{\bdeltak{k}}$ is the \textit{transferring level} of source $k$. And we define the \textit{level-$h$ transferring set} $\mah = \{k: \onenorm{\bdeltak{k}} \leq h\}$ as the set of sources which has transferring level lower than $h$. Note that in general, $h$ can be any positive values and different $h$ values define different $\mah$ set. However, in our regime of interest, $h$ shall be reasonably small to guarantee that transferring sources in $\mah$ is beneficial. Denote $\nah = \sum_{k \in \mah}n_k$, $\alpha_k = \frac{n_k}{\nah + n_0}$ for $k \in \{0\}\cup \mah$ and $\Kah = |\mah|$.

Note that in \eqref{eq: glm}, we assume GLMs of the target and all sources share the same inverse link function $\psi$. After a careful examination of our proofs for theoretical properties in Section \ref{sec:theory}, we find that these theoretical results still hold even when the target and each source have their own function $\psi$, as long as these GLMs satisfy Assumptions \ref{asmp: convexity} and \ref{asmp: second derivative} (to be presented in Section \ref{subsec:oracle alg}). It means that transferring information across different GLM families is possible. For simplicity, in the following discussion, we assume all these GLMs belong to the same family and hence have the same function $\psi$.

\subsection{Two-step GLM transfer learning}\label{subsec:glm tranfer alg}
We first introduce a general transfer learning algorithm on GLMs, which can be applied to transfer all sources in a given index set $\ma$. The algorithm is motivated by the ideas in \cite{bastani2021predicting} and \cite{li2021transfer}, which we call a \textit{two-step transfer learning algorithm}. The main strategy is to first transfer the information from those sources by pooling all the data to obtain a rough estimator, then correct the bias in the second step using the target data. More specifically, we fit a GLM with $\ell_1$-penalty by pooled samples first, then fit the contrast in the second step using only the target by another $\ell_1$-regularization. The detailed algorithm ($\ma$-Trans-GLM) is presented in Algorithm \ref{algo: merging}.  The transferring step could be understood as to solve the following equation w.r.t. $\bw \in \mathbb{R}^p$:
\begin{equation}\label{eq: sample version equation wa}
	\sum\limits_{k\in\{0\}\cup \ma}\left[(\bXk{k})^T\byk{k} - \sum\limits_{i=1}^{n_k}\psi'((\bw)^T\bxk{k}_i)\bxk{k}_i\right] = \bm{0}_{p},
\end{equation}
which converges to the solution of its population version under certain conditions
\begin{equation}\label{eq: wa in population merging}
	\sum\limits_{k\in\{0\}\cup \ma}\alpha_k\te\left\{[\psi'((\bwa)^T\bxk{k})-\psi'((\bwk{k})^T\bxk{k})]\bxk{k}\right\} = \bm{0}_{p},
\end{equation}
where $\alpha_k = \frac{n_k}{\na + n_0}$. Notice that in the linear case, $\bwa$ can be explicitly expressed as a linear transformation of the true parameter $\bwk{k}$, i.e. $
	\bwa = \bm{\Sigma}^{-1}\sum_{k \in \{0\}\cup \ma} \alpha_k \bm{\Sigma}^{(k)}\bwk{k}$,
where $\bm{\Sigma}^{(k)} = \te[\bxk{k}(\bxk{k})^T]$ and $\bm{\Sigma} = \sum_{k \in \{0\}\cup \ma}\alpha_k\te[\bxk{k}(\bxk{k})^T]$ \citep{li2021transfer}. 

To help readers better understand the algorithm, we draw a schematic in Section \ref{subsec: schematic} of supplements. We refer interested readers who want to get more intuitions to that.

\begin{algorithm}[!h]
\caption{$\ma$-Trans-GLM}
\label{algo: merging}
\KwIn{target data $(\bX^{(0)}, \by^{(0)})$, source data $\{(\bXk{k}, \byk{k})\}_{k=1}^K$, penalty parameters $\lambda_{\bm{w}}$ and $\lambda_{\bm{\delta}}$, transferring set $\ma$}
\KwOut{the estimated coefficient vector $\hat{\bbeta}$}
\underline{\textbf{Transferring step}}: Compute
$\hwa \leftarrow \argmin\limits_{\bm{w}}\left\{\frac{1}{\na + n_0}\sum\limits_{k\in\{0\}\cup \ma}\left[-(\byk{k})^T\bXk{k}\bm{w} + \sum\limits_{i=1}^{n_k}\psi(\bm{w}^T\bxk{k}_i)\right] + \lambda_{\bm{w}}\onenorm{\bm{w}}\right\}$ \\
\underline{\textbf{Debiasing step}}: Compute
$\hdeltaa \leftarrow \argmin\limits_{\bm{\delta}}\left\{-\frac{1}{n_0}(\by^{(0)})^T\bX^{(0)}(\hwa+\bm{\delta}) + \frac{1}{n_0}\sum\limits_{i=1}^{n_0}\psi((\hwa+\bm{\delta})^T\bx^{(0)}_i) + \lambda_{\bm{\delta}}\onenorm{\bm{\delta}}\right\}$\\
Let $\hbeta \leftarrow \hwa + \hdeltaa$\\
Output $\hbeta$ \\
\end{algorithm}

\subsection{Transferable source detection}\label{subsec:detection}
As we described, Algorithm \ref{algo: merging} can be applied only if we are certain about which sources to transfer, which in practice may not be known as a priori. Transferring certain sources may not improve the performance of the fitted model based on only target, and can even lead to worse performance. In transfer learning, we say \textit{negative transfer} happens when the source data leads to an inferior performance on the target task \citep{pan2009survey, torrey2010transfer, weiss2016survey}. How to avoid negative transfer has become an increasingly popular research topic. 

Here we propose a simple, \textit{algorithm-free}, and \textit{data-driven} method to determine an informative transferring set $\widehat{\ma}$. We call this approach a transferable source \textit{detection} algorithm and refer to it as Trans-GLM. 

We sketch this detection algorithm as follows. First, divide the target data into three folds, that is, $\{(\bX^{(0)[r]},\by^{(0)[r]})\}_{r=1}^3$. Note that we choose three folds only for convenience. We also explored other fold number choices in the simulation. See Section \ref{subsubsec:source detection} in the supplementary materials. Second, run the transferring step on each source data and every two folds of target data. Then, for a given loss function, we calculate its value on the left-out fold of target data and compute the average cross-validation loss $\hat{L}_0^{(k)}$ for each source. As a benchmark, we also fit Lasso on every choice of two folds of target data and calculate the loss on the remaining fold. The average cross-validation loss $\hat{L}_0^{(0)}$ is viewed as the loss of target. Finally, the difference between $\hat{L}_0^{(k)}$ and $\hat{L}_0^{(0)}$ is calculated and compared with some threshold, and sources with a difference less than the threshold will be recruited into $\widehat{\ma}$. Under the GLM setting, a natural loss function is the negative log-likelihood. For convenience, suppose $n_0$ is divisible by 3. According to \eqref{eq: glm}, for any coefficient estimate $\bw$, the average of negative log-likelihood on the $r$-th fold of target data $(\bX^{(0)[r]}, \by^{(0)[r]})$ is 
\begin{equation}\label{eq: l0 hat def}
	\hat{L}_0^{[r]}(\bw) =-\frac{1}{n_0/3}\sum\limits_{i=1}^{n_0/3}\log \rho(y_i^{(0)[r]})  -\frac{1}{n_0/3}(\by^{(0)[r]})^T\bX^{(0)}\bm{w} + \frac{1}{n_0/3}\sum\limits_{i=1}^{n_0/3}\psi(\bm{w}^T\bx^{(0)[r]}_i).
\end{equation}
The detailed algorithm is presented as Algorithm \ref{algo: a unknown}.

\begin{algorithm}[!h]
\caption{Trans-GLM}
\label{algo: a unknown}
\KwIn{target data $(\bX^{(0)}, \by^{(0)})$, all source data $\{(\bX^{(k)}, \by^{(k)})\}_{k=1}^K$, a constant $C_0>0$, penalty parameters $\{\{\lambda^{(k)[r]}\}_{k=0}^K\}_{r=1}^3$}
\KwOut{the estimated coefficient vector $\hat{\bbeta}$, and the determined transferring set $\widehat{\mathcal{A}}$}
\underline{\textbf{Transferable source detection}}:
Randomly divide $(\bX^{(0)}, \by^{(0)})$ into three sets of equal size as $\{(\bX^{(0)[i]}, \by^{(0)[i]})\}_{i=1}^3$ \\
\For{$r = 1$ \KwTo $3$}{
	$\hbeta^{(0)[r]} \leftarrow$ fit the Lasso on $\{(\bX^{(0)[i]}, \by^{(0)[i]})\}_{i=1}^3 \backslash (\bX^{(0)[r]}, \by^{(0)[r]})$ with penalty parameter $\lambda^{(0)[r]}$\\
	$\hbeta^{(k)[r]} \leftarrow$ run step 1 in Algorithm \ref{algo: merging} with $(\{(\bX^{(0)[i]}, \by^{(0)[i]})\}_{i=1}^3 \allowbreak \backslash (\bX^{(0)[r]}, \by^{(0)[r]}))\cup (\bX^{(k)}, \by^{(k)})$ and penalty parameter $\lambda^{(k)[r]}$ for all $k \neq 0$\\
	Calculate the loss function $\hat{L}_0^{[r]}(\hbeta^{(k)[r]})$ on $(\bX^{(0)[r]}, \by^{(0)[r]})$ for $k = 1, \ldots, K$ \\
}
$\hat{L}_0^{(k)} \leftarrow \sum_{r=1}^3\hat{L}_0^{[r]}(\hbeta^{(k)[r]})/3$, $\hat{L}_0^{(0)} \leftarrow \sum_{r=1}^3 \hat{L}_0^{[r]}(\hbeta^{(0)[r]})/3$, $\hat{\sigma} = \sqrt{\sum_{r=1}^3 (\hat{L}_0^{[r]}(\hbeta^{(0)[r]})-\hat{L}_0^{(0)})^2/2}$\\
$\widehat{\mathcal{A}} \leftarrow \{k \neq 0: \hat{L}_0^{(k)} - \hat{L}_0^{(0)} \leq C_0(\hat{\sigma} \vee 0.01)\}$ \\
\underline{\textbf{$\widehat{\mathcal{A}}$-Trans-GLM}}: $\hbeta \leftarrow$ run Algorithm \ref{algo: merging} using  $\{(\bX^{(k)}, \by^{(k)})\}_{k\in \{0\} \cup \widehat{\mathcal{A}}}$\\
Output $\hbeta$ and $\widehat{\mathcal{A}}$ \\
\end{algorithm}

It's important to point out that Algorithm \ref{algo: a unknown} \textbf{does not} require the input of $h$. We will show that $\widehat{\ma} = \mah$ for some specific $h$ if certain conditions hold, in Section \ref{subsec:detection theory}. Furthermore, under these conditions, transferring with $\widehat{\ma}$ will lead to a faster convergence rate compared to Lasso fitted on only the target data, when target sample size $n_0$ falls into some regime. This is the reason that this algorithm is called the \textit{transferable} source detection algorithm.

\subsection{Confidence intervals}\label{subsec:inf}
In previous sections, we've discussed how to obtain a point estimator of the target coefficient vector $\bbeta$ from the two-step transfer learning approach. In this section, we would like to construct the asymptotic confidence interval (CI) for each component of $\bbeta$ based on that point estimate.

As described in the introduction, there have been quite a few works on high-dimensional GLM inference in the literature. In the following, we will propose a transfer learning procedure to construct CI based on the desparsified Lasso \citep{van2014asymptotically}. Recall that desparsified Lasso contains two main steps. The first step is to learn the inverse Fisher information matrix of GLM by nodewise regression \citep{meinshausen2006high}. The second step is to ``debias" the initial point estimator and then construct the asymptotic CI. Here, the estimator $\hbeta$ from Algorithm \ref{algo: merging} can be used as an initial point estimator. Intuitively, if the predictors from target and source data are similar and satisfy some sparsity conditions, it might be possible to use Algorithm \ref{algo: merging} for learning the inverse Fisher information matrix of target data, which effectively combines the information from target and source data.

Before  formalizing the procedure to construct the CI, let's first define several additional notations. For any $\bw \in \mathbb{R}^{n}$, denote $\bWk{k}_{\bw} = \diag\left(\sqrt{\psi''((\bxk{k}_1)^T\bw)}, \ldots, \sqrt{\psi''((\bxk{k}_{n_k})^T\bw)}\right)$, $\bXk{k}_{\bw} = \bWk{k}_{\bw}\bXk{k}$, $\bSigmak{k}_{\bw} = \te[\bxk{k}(\bxk{k})^T\psi''((\bxk{k})^T\bw)]$ and $\widehat{\bm{\Sigma}}^{(k)}_{\bw} = n_k^{-1}(\bXk{k}_{\bw})^T\bXk{k}_{\bw}$. $\bXk{k}_{\bw, j}$ represents the $j$-th column of $\bXk{k}_{\bw}$ and $\bXk{k}_{\bw, -j}$ represents the matrix $\bXk{k}_{\bw}$ without the $j$-th column. $\widehat{\bm{\Sigma}}^{(k)}_{\bw, j, -j}$ represents the $j$-th row of $\widehat{\bm{\Sigma}}^{(k)}_{\bw}$ without the diagonal $(j,j)$ element, and $\widehat{\bm{\Sigma}}^{(k)}_{\bw, j, j}$ is the diagonal $(j,j)$ element of $\widehat{\bm{\Sigma}}^{(k)}_{\bw}$.

Next, we explain the details of the CI construction procedure in Algorithm \ref{algo: inf}. In step 1, we obtain a point estimator $\hbeta$ from $\ma$-Trans-GLM (Algorithm \ref{algo: merging}), given a specific transferring set $\ma$. Then in steps 2-4, we estimate the target inverse Fisher information matrix $(\bSigmak{0}_{\bbeta})^{-1}$ as 
\begin{equation}\label{eq: theta formula}
	\hTheta = \text{diag}(\htau_1^{-2}, \ldots, \htau_p^{-2})
	\begin{pmatrix}
		1 &-\hgammak{0}_{1,2} &\hdots  &-\hgammak{0}_{1,p} \\
		-\hgammak{0}_{2,1} &1 &\hdots  &-\hgammak{0}_{2,p} \\
		\vdots &\vdots &\ddots &\vdots \\
		-\hgammak{0}_{p,1} &-\hgammak{0}_{p,2} &\hdots &1
	\end{pmatrix}.
\end{equation}
Finally in step 5, we ``debias" $\hbeta$ using the target data to get a new point estimator $\hat{\bm{b}}$ which is asymptotically unbiased as
\begin{equation}\label{eq: b hat}
	\hat{\bm{b}} = \hbeta +\frac{1}{n_0}\hTheta(\bXk{0})^T[\bm{Y}^{(0)} - \bpsi'(\bXk{0}\hbeta)],
\end{equation}
where $\bpsi'(\bXk{0}\hbeta) \coloneqq (\psi'((\bxk{0}_1)^T\hbeta), \ldots, \psi'((\bxk{0}_{n_0})^T\hbeta))^T \in \mathbb{R}^{n_0}$. 

It's necessary to emphasize that the confidence level $(1-\alpha)$ is for every single CI rather than for all $p$ CIs simultaneously. As discussed in Sections 2.2 and 2.3 of \cite{van2014asymptotically}, it is possible to get simultaneous CIs for different coefficient components and do multiple hypothesis tests when the design is fixed. In other cases, e.g., random design in different replications (which we focus on in this paper), multiple hypothesis testing might be more challenging.

%

\begin{algorithm}[!h]
\caption{Confidence interval construction via nodewise regression}
\label{algo: inf}
\KwIn{target data $(\bX^{(0)}, \by^{(0)})$, source data $\{(\bXk{k}, \byk{k})\}_{k=1}^K$, penalty parameters $\{\lambda_j\}_{j=1}^p$ and $\{\tlambda_j\}_{j=1}^p$, transferring set $\ma$, confidence level $(1-\alpha)$}
\KwOut{Level-$(1-\alpha)$ confidence interval $\mathcal{I}_j$ for $\beta_j$ with $j = 1,\ldots, p$}
Compute $\hbeta$ via Algorithm \ref{algo: merging} \\
Compute $\hgammaa_j \leftarrow \argmin\limits_{\bm{\gamma}}\left\{-\frac{1}{2(\na+n_0)}\sum_{k \in \{0\}\cup \ma}\twonorm{\bXk{k}_{\hbeta, j} - \bXk{k}_{\hbeta, -j}\bgamma}^2 + \lambda_j\onenorm{\bgamma}\right\}$ for $j = 1, \ldots, p$\\
Compute
$\hrho_j \leftarrow \argmin\limits_{\bm{\varrho}}\left\{-\frac{1}{2n_0}\twonorm{\bXk{0}_{\hbeta, j} - \bXk{0}_{\hbeta, -j}(\hgammaa_j + \brho)} + \tlambda_j\onenorm{\bm{\varrho}}\right\}$\\
Compute $\hbgammak{0}_j \leftarrow \hgammaa_j + \hrho_j$, $\hSigma_{\hbeta} \leftarrow \sum_{k \in \{0\}\cup \ma}\frac{n_k}{\na+n_0}\widehat{\bm{\Sigma}}^{(k)}_{\hbeta}$, $\htau_j^2 = \hSigma_{\hbeta, j, j} - \hSigma_{\hbeta, j, -j}\hgamma_j$ and calculate $\hTheta$ via \eqref{eq: theta formula}, where $\hbgammak{0}_j = (\hgammak{0}_{j,1}, \ldots, \hgammak{0}_{j,j-1}, \hgammak{0}_{j,j+1}, \ldots, \hgammak{0}_{j,p})^T$.\\
Compute $\mathcal{I}_j \leftarrow [\hb_j-\hTheta^T_j\hSigma_{\hbeta}\hTheta_jq_{\alpha/2}/\sqrt{n_0}, \hb_j+\hTheta^T_j\hSigma_{\hbeta}\hTheta_jq_{\alpha/2}/\sqrt{n_0}]$ for $j = 1, \ldots, p$, where $\hb_j$ is the $j$-th component of $\hat{\bm{b}}$ in \eqref{eq: b hat}, and $q_{\alpha/2}$ is the $\alpha/2$-left tail quantile of $\mathcal{N}(0, 1)$\\
Output $\{\mathcal{I}_j\}_{j=1}^p$
\end{algorithm}

\section{Theory}\label{sec:theory}
In this section, we will establish theoretical guarantees on the three proposed algorithms. Section \ref{subsec:oracle alg} provides a detailed analysis of Algorithm \ref{algo: merging} with transferring set $\mah$, which we denote as $\mah$-Trans-GLM. Section \ref{subsec:detection theory} introduces certain conditions, under which we show that the transferring set $\widehat{\ma}$ detected by Algorithm \ref{algo: a unknown} (Trans-GLM) is equal to $\mah$ for some $h$ with high probability. Section \ref{subsec:inf theory} presents the analysis of Algorithm \ref{algo: inf} with transferring set $\mah$, where we prove a central limit theorem. For the proofs and some additional theoretical results, refer to supplementary materials.

\subsection{Theory on $\mah$-Trans-GLM}\label{subsec:oracle alg}
We first impose some common assumptions about GLM. 

\begin{assumption}\label{asmp: convexity}
	$\psi$ is infinitely differentiable and strictly convex. We call a second-order differentiable function $\psi$ strictly convex if $\psi''(x) > 0$.
\end{assumption}

\begin{assumption}\label{asmp: x}
	For any $\bm{a} \in \mathbb{R}^p$, $\bm{a}^T\bxk{k}_i$'s are i.i.d. $\kappa_u\twonorm{\bm{a}}^2$-subGaussian variables with zero mean for all $k = 0, \ldots, K$, where $\kappa_u$ is a positive constant. Denote the covariance matrix of $\bxk{k}$ as $\bm{\Sigma}^{(k)}$, with $\inf_{k}\lambda_{\min}(\bm{\Sigma}^{(k)}) \geq \kappa_l > 0$, where $\kappa_l$ is a positive constant.
\end{assumption}

\begin{assumption}\label{asmp: second derivative}
	At least one of the following assumptions hold: ($M_{\psi}$, $U$ and $\bar{U}$ are some positive constants)
	\begin{enumerate}[(i)]
		\item $\infnorm{\psi''} \leq M_{\psi} < \infty$ a.s.;
		\item $\sup\limits_k \infnorm{\bxk{k}} \leq U < \infty$ a.s., $\sup\limits_{k}\sup\limits_{|z|\leq \bar{U}} \psi''((\bxk{k})^T\bwk{k} + z) \leq M_{\psi}< \infty$ a.s.
	\end{enumerate}
\end{assumption}

Assumption \ref{asmp: convexity} imposes the \textit{strict convexity} and differentiability of $\psi$, which is satisfied by many popular distribution families, such as Gaussian, binomial, and Poisson distributions. Note that we do not require $\psi$ to be \textit{strongly convex} (that is, $\exists C > 0$, such that $\psi''(x) > C$), which relaxes Assumption 4 in \cite{bastani2021predicting}. It is easy to verify that $\psi$ in logistic regression is in general not strongly convex with unbounded predictors. Assumption \ref{asmp: x} requires the predictors in each source to be subGaussian with a well-behaved correlation structure. Assumption \ref{asmp: second derivative} is motivated by Assumption (GLM 2) in the full-length version of \cite{negahban2009unified}, which is imposed to restrict $\psi''$ in a bounded region in some sense. Note that linear regression and logistic regression satisfy condition (\rom{1}), while Poisson regression with coordinate-wise bounded predictors and $\ell_1$-bounded coefficients satisfies condition (\rom{2}). 

Besides these common conditions on GLM, as discussed in Section \ref{subsec:glm tranfer alg}, to guarantee the success of $\mah$-Trans-GLM, we have to make sure that the estimator from the transferring step is close enough to $\bbeta$. Therefore we introduce the following assumption, which guarantees $\bwah$ defined in \eqref{eq: wa in population merging} with $\ma = \mah$ is close to $\bbeta$. 

\begin{assumption}\label{asmp: merging}
Denote $\widetilde{\bm{\Sigma}}_h = \sum_{k \in \transet} \alpha_k \te\Big[\int_0^1 \psi''((\bxk{k})^T\bbeta+t(\bxk{k})^T(\bwah-\bbeta))dt \cdot \bxk{k}(\bxk{k})^T \Big]$ and $\widetilde{\bm{\Sigma}}_h^{(k)} = \te\Big[\int_0^1 \psi''((\bxk{k})^T\bbeta+t(\bxk{k})^T(\bwk{k}-\bbeta))dt \cdot \bxk{k}(\bxk{k})^T \Big]$. It holds that $\sup_{k \in \transet}\onenorm{\widetilde{\bm{\Sigma}}_h^{-1}\widetilde{\bm{\Sigma}}_h^{(k)}} < \infty$.
\end{assumption}
\begin{remark}
	A sufficient condition for Assumption \ref{asmp: merging} to hold is $(\widetilde{\bSigma}^{(k)}_{\bwah, \bbeta})^{-1}\widetilde{\bSigma}^{(k')}_{\bwk{k'}, \bbeta}$ has positive diagonal elements and is diagonally dominant, for any $k \neq k'$ in $\mah$, where $\widetilde{\bSigma}^{(k)}_{\bw, \bbeta} \coloneqq \te\Big[\int_0^1 \psi''((\bxk{k})^T\bbeta+t(\bxk{k})^T(\bw-\bbeta))dt \cdot \bxk{k}(\bxk{k})^T \Big]$ for any $\bw \in \mathbb{R}^p$.
\end{remark}
In the linear case, this assumption can be further simplified as a restriction on heterogeneity between target predictors and source predictors. More discussions can be found in Condition 4 of \cite{li2021transfer}. Now, we are ready to present the following main result for the $\mah$-Trans-GLM algorithm. Define the parameter space as
\begin{equation}
	\Xi(s, h) = \left\{\bbeta, \{\bwk{k}\}_{k\in \mah}: \zeronorm{\bbeta} \leq s, \sup_{k\in \mah}\onenorm{\bwk{k} - \beta} \leq h\right\}.
\end{equation}
Given $s$ and $h$, we compress parameters $\bbeta$, $\{\bwk{k}\}_{k\in \mah}$ into a parameter set $\bxi$ for simplicity.

\begin{theorem}[$\ell_1$/$\ell_2$-estimation error bound of $\mah$-Trans-GLM with Assumption \ref{asmp: merging}]\label{thm: l2 with assumption}
	Assume Assumptions \ref{asmp: convexity}, \ref{asmp: x} and \ref{asmp: merging} hold.  Suppose $h \ll \sqrt{\frac{n_0}{\log p}}$, $h \leq C\sqrt{s}$, $n_0 \geq C\log p$ and $\nah \geq Cs\log p$, where $C>0$ is a constant. Also assume Assumption \ref{asmp: second derivative}.(\rom{1}) holds or Assumption \ref{asmp: second derivative}.(\rom{2}) with $h \leq C'U^{-1}\bar{U}$ for some $C' > 0$ holds. We take $\lambda_{\bm{w}} = C_{\bw}\sqrt{\frac{\log p}{\nah+n_0}}$ and $\lambda_{\bdelta} = C_{\bdelta}\sqrt{\frac{\log p}{n_0}}$, where $C_{\bw}$ and $C_{\bdelta}$ are sufficiently large positive constants.  Then
	\begin{align}
		\sup_{\xi \in \Xi(s, h)}\tp\left(\twonorm{\hbeta - \bbeta} \lesssim \left(\frac{s\log p}{\nah+n_0}\right)^{1/2} +  \left[\left(\frac{\log p}{n_0}\right)^{1/4}h^{1/2}\right]\wedge h  \right)&\geq 1-n_0^{-1}, \label{eq: l2 bound 1 with asmp}\\
		\sup_{\xi \in \Xi(s, h)}\tp\left(\onenorm{\hbeta - \bbeta} \lesssim s\left(\frac{\log p}{\nah+n_0}\right)^{1/2} + h\right)&\geq 1-n_0^{-1}.\label{eq: l1 bound 1 with asmp}
	\end{align}
\end{theorem}

\begin{remark}
	When $h \ll s\sqrt{\frac{\log p}{n_0}}$, $\nah \gg n_0$, the upper bounds in \eqref{eq: l2 bound 1 with asmp} and \eqref{eq: l1 bound 1 with asmp} are better than the classical Lasso $\ell_2$-bound $\mathcal{O}_p\left(\sqrt{\frac{s\log p}{n_0}}\right)$ and $\ell_1$-bound $\mathcal{O}_p\left(s\sqrt{\frac{\log p}{n_0}}\right)$ using only target data.
\end{remark}

Similar to Theorem 2 in \cite{li2021transfer}, we can show the following lower bound of $\ell_1$/$\ell_2$-estimation error in regime $\Xi(s, h)$ in the minimax sense.

\begin{theorem}[$\ell_1$/$\ell_2$-minimax estimation error bound]\label{thm: minimax}
Assume Assumptions \ref{asmp: convexity}, \ref{asmp: x} and \ref{asmp: merging} hold. Also assume Assumption \ref{asmp: second derivative}.(\rom{1}) holds or Assumption \ref{asmp: second derivative}.(\rom{2}) with $n_0 \gtrsim s^2\log p$. Then
\begin{align}
	\inf_{\hat{\bbeta}}\sup_{\bxi \in \Xi(s, h)}\tp\left(\twonorm{\hbeta - \bbeta} \gtrsim \left(\frac{s\log p}{\nah+n_0}\right)^{1/2} + \left(\frac{s\log p}{n_0}\right)^{1/2} \wedge \left[\left(\frac{\log p}{n_0}\right)^{1/4}h^{1/2}\right]\wedge h\right)&\geq \frac{1}{2}, \\
	\inf_{\hat{\bbeta}}\sup_{\bxi \in \Xi(s, h)}\tp\left(\onenorm{\hbeta - \bbeta} \gtrsim s\left(\frac{\log p}{\nah+n_0}\right)^{1/2} + \left[s\left(\frac{\log p}{n_0}\right)^{1/2}\right] \wedge h\right)&\geq \frac{1}{2}.
\end{align}
\end{theorem}

\begin{remark}
	Theorem \ref{thm: minimax} indicates that under conditions on $h$ required by Theorem \ref{thm: l2 with assumption} ($h \lesssim s\sqrt{\log p/n_0}$), $\mah$-Trans-GLM achieves the minimax optimal rate of $\ell_1$/$\ell_2$-estimation error bound.
\end{remark}
Next, we present a similar upper bound, which is weaker than the bound above but holds without requiring Assumption \ref{asmp: merging}.

\begin{theorem}[$\ell_1$/$\ell_2$-estimation error  bound of $\mah$-Trans-GLM without Assumption \ref{asmp: merging}]\label{thm: l2 without assumption}
	Assume Assumptions \ref{asmp: convexity} and \ref{asmp: x} hold.  Suppose $h \ll \sqrt{\frac{n_0}{\log p}}$, $h \leq Cs^{-1/2}$, $n_0 \geq C\log p$ and $\nah \geq Cs\log p$, where $C>0$ is a constant. Also assume Assumption \ref{asmp: second derivative}.(\rom{1}) holds or Assumption \ref{asmp: second derivative}.(\rom{2}) with $h \leq C'U^{-1}\bar{U}$ for some $C' > 0$ holds. We take $\lambda_{\bm{w}} = C_{\bw}\left(\sqrt{\frac{\log p}{\nah+n_0}} + h\right)$ and $\lambda_{\bdelta} = C_{\bdelta}\sqrt{\frac{\log p}{n_0}}$, where $C_{\bw}$ and $C_{\bdelta}$ are sufficiently large positive constants. Then
		\begin{align}
			\sup_{\bxi \in \Xi(s, h)}\tp\left(\twonorm{\hbeta - \bbeta} \lesssim \left(\frac{s\log p}{\nah+n_0}\right)^{1/2} + \sqrt{s}h+ \left[\left(\frac{\log p}{n_0}\right)^{1/4}h^{1/2}\right]\wedge h\right)&\geq 1-n_0^{-1}, \\
			\sup_{\bxi \in \Xi(s, h)}\tp\left(\onenorm{\hbeta - \bbeta} \lesssim s\sqrt{\frac{\log p}{\nah+n_0}} + sh\right)&\geq 1-n_0^{-1}.\label{eq: l1 bound 1 without asmp}
		\end{align}
\end{theorem}

\begin{remark}\label{rmk: h beneficial regime}
	When $h \ll \sqrt{\frac{\log p}{n_0}}$ and $\nah \gg n_0$, the upper bounds in (\rom{1}) and (\rom{2}) are better than the classical Lasso bound $\mathcal{O}_p\left(\sqrt{\frac{\log p}{n_0}}\right)$ with target data.
\end{remark}

Comparing the results in Theorems \ref{thm: l2 with assumption} and \ref{thm: l2 without assumption}, we know that with Assumption \ref{asmp: merging}, we could get sharper $\ell_1/\ell_2$-estimation error bounds.

\subsection{Theory on the transferable source detection algorithm}\label{subsec:detection theory}

In this section, we will show that under certain conditions, our transferable set detection algorithm (Trans-GLM) can recover the level-$h$ transferring set $\mah$ for some specific $h$, that is, $\widehat{\ma} = \mah$ with high probability. Under these conditions, transferring with $\widehat{\ma}$ will lead to a faster convergence rate compared to Lasso fitted on the target data, when the target sample size $n_0$ falls into certain regime. But as we described in Section \ref{subsec:detection}, Algorithm \ref{algo: a unknown} does not require any explicit input of $h$.

The corresponding population version of $\hat{L}_0^{[r]}(\bw)$ defined in \eqref{eq: l0 hat def} is
\begin{align}
	L_0(\bw)&= -\te[\log \rho(y^{(0)})]-\te[\yk{0}\bw^T\bxk{0}] + \te[\psi(\bm{w}^T\bx^{(0)})] \\
	&= -\te[\log \rho(y^{(0)})] -\te[\psi'(\bm{w}^T\bx^{(0)})\bw^T\bxk{0}] + \te[\psi(\bm{w}^T\bx^{(0)})].
\end{align}

Based on Assumption \ref{asmp: merging detection}, similar to \eqref{eq: wa in population merging}, for $\{k\}$-Trans-GLM (Algorithm \ref{algo: merging} with $\ma = \{k\}$) used in Algorithm \ref{algo: a unknown}, consider the following population version of estimators from the transferring step with respect to target data and the $k$-th source data, which is the solution $\bbetak{k}$ of equation $\sum_{j\in\{0,k\}}\alpha_j^{(k)}\te\left\{[\psi'((\bbetak{k})^T\bxk{k})-\psi'((\bwk{k})^T\bxk{k})]\bxk{k}\right\} = 0$,
where $\alpha_0^{(k)} = \frac{2n_0/3}{2n_0/3 + n_k}$ and $\alpha_k^{(k)} = \frac{n_k}{2n_0/3 + n_k}$. Define $\bbetak{0} = \bbeta$. Next, let's impose a general assumption to ensure the identifiability of some $\mah$ by Trans-GLM.
\begin{assumption}[Identifiability of $\mah$]\label{asmp: idenfity a}
	Denote $\mach = \{1, \ldots,K\}\backslash \mah$. Suppose for some $h$, we have
	\begin{align}
		\tp\left(\sup_r\norm{\hat{L}_0^{[r]}(\hbeta^{(k)[r]}) - \hat{L}_0^{[r]}(\bbetak{k})} > \Upsilon_1^{(k)}+\zeta\Gamma_1^{(k)}\right) &\lesssim g_1^{(k)}(\zeta),\\ 
		 \tp\left(\sup_r\norm{\hat{L}_0^{[r]}(\bbetak{k}) -L_0(\bbetak{k})} > \zeta\Gamma_2^{(k)}\right) &\lesssim g_2^{(k)}(\zeta),
	\end{align}
	where $g_1^{(k)}(\zeta)$, $g_2^{(k)}(\zeta) \rightarrow 0$ as $\zeta \rightarrow \infty$.
	Assume $\inf\limits_{k \in \mach}\lambda_{\min}(\te[\int_0^1\psi''((1-t)(\bxk{0})^T\bbeta + t(\bxk{0})^T\bbetak{k})dt\cdot \bxk{0}(\bxk{0})^T]) \coloneqq \underline{\lambda} > 0$, and
	\begin{align}
		\twonorm{\bbetak{k} -\bbeta} &\geq \underline{\lambda}^{-1/2}\left[C_1\left(\sqrt{\Gamma_1^{(0)}}\vee \sqrt{\Gamma_2^{(0)}} \vee 1\right) + \sqrt{2\Upsilon_1^{(k)}}\right], \forall k \in \mach \label{eq: gap ac l0}\\
		\Upsilon_1^{(k)} + \Gamma_1^{(k)} + \Gamma_2^{(k)} + h^2 &= \smallo(1), \forall k \in \mah;\quad \Gamma_1^{(k)} = \smallo(1), \Gamma_2^{(k)} = \smallo(1), \forall k \in \mach, \label{eq: error term dominated}
	\end{align}
	where $C_1 > 0$ is sufficiently large.
\end{assumption}

\begin{remark}
	Here we use generic notations $\Upsilon_1^{(k)}$, $\Gamma_1^{(k)}$, $\Gamma_2^{(k)}$, $g_1^{(k)}(\zeta)$ and $g_2^{(k)}(\zeta)$. We show their explicit forms under linear, logistic, and Poisson regression models in Proposition \ref{prop: gamma_1 gamma_2} in Section \ref{subsubsec: explicit forms} of supplements.
\end{remark}

\begin{remark}
	Condition \eqref{eq: gap ac l0} guarantees that for the sources not in $\mah$, there is a sufficiently large gap between the population-level coefficient from the transferring step and the true coefficient of target data. Condition \eqref{eq: error term dominated} guarantees the variations of $\sup_r\norm{\hat{L}_0^{[r]}(\hbeta^{(k)[r]}) - \hat{L}_0^{[r]}(\bbetak{k})}$ and $\sup_r\norm{\hat{L}_0^{[r]}(\bbetak{k}) -L_0(\bbetak{k})}$ are shrinking as the sample sizes go to infinity.
\end{remark}

Based on Assumption \ref{asmp: idenfity a}, we have the following detection consistency property.
\begin{theorem}[Detection consistency of $\mah$]\label{thm: a detection}
	For Algorithm \ref{algo: a unknown} (Trans-GLM), with Assumption \ref{asmp: idenfity a} satisfied for some $h$, for any $\delta > 0$, there exist constants $C'(\delta)$ and $N = N(\delta) > 0$ such that when $C_0 = C'(\delta)$ and $\min_{k \in \transet} n_k > N(\delta)$,
	\begin{equation}
		\tp(\widehat{\mathcal{A}} = \mathcal{A}_h) \geq 1-\delta.
	\end{equation}
	Then Algorithm \ref{algo: a unknown} has the same high-probability upper bounds of $\ell_1/\ell_2$-estimation error as those in Theorems \ref{thm: l2 with assumption} and \ref{thm: l2 without assumption} under the same conditions, respectively. 
\end{theorem}

\begin{remark}
	We would like to emphasize again that Algorithm \ref{algo: a unknown} does not require the explicit input of $h$. Theorem \ref{thm: a detection} tells us that the transferring set $\widehat{\ma}$ suggested by Trans-GLM will be $\mah$ for some $h$, under certain conditions.
\end{remark}

Next, we attempt to provide a sufficient and more explicit condition (Corollary \ref{cor: explicit asmp}) to ensure that Assumption \ref{asmp: idenfity a} hold. Recalling the procedure of Algorithm \ref{algo: a unknown}, we note that it relies on using the negative log-likelihood as the similarity metric between target and source data, where the accurate estimation of coefficients or log-likelihood for GLM under the high-dimensional setting depends on the sparse structure. Therefore, in order to provide an explicit and sufficient condition for Assumption \ref{asmp: merging detection} to hold, we now impose a ``weak" sparsity assumption on both $\bwk{k}$ and $\bbetak{k}$ with $k \in \mach$ for some $h$. Note that the source data in $\mah$ automatically satisfy the sparsity condition due to the definition of $\mah$.

\begin{assumption}\label{asmp: merging detection}
For some $h$ and any $k \in \mach$, we assume $\bwk{k}$ and $\bbetak{k}$ can be decomposed as follows with some $s'$ and $h' > 0$:
\begin{enumerate}[(i)]
	\item $\bwk{k} = \bm{\varsigma}^{(k)} + \bm{\vartheta}^{(k)}$, where $\zeronorm{\bm{\varsigma}^{(k)}} \leq s'$ and  $\onenorm{\bm{\vartheta}^{(k)}} \leq h'$;
	\item $\bbetak{k} = \bm{\iota}^{(k)} + \bm{\varpi}^{(k)}$, where $\zeronorm{\bm{\iota}^{(k)}} \leq s'$ and  $\onenorm{\bm{\varpi}^{(k)}} \leq h'$.
\end{enumerate}
\end{assumption}

\begin{corollary}\label{cor: explicit asmp}
	Assume Assumptions \ref{asmp: convexity}, \ref{asmp: x}, \ref{asmp: merging detection} and $\inf\limits_{k \in \mach}\lambda_{\min}\big(\te[\int_0^1\psi''((1-t)(\bxk{0})^T\bbeta + t(\bxk{0})^T\bbetak{k}) \allowbreak dt\cdot \bxk{0}(\bxk{0})^T]\big) \coloneqq \underline{\lambda} > 0$ hold. Also assume $\sup_{k \in \mach} \infnorm{\bbetak{k}} < \infty$, $\sup_k \infnorm{\bwk{k}} < \infty$. Let $\lambda^{(k)[r]} = C\left(\sqrt{\frac{\log p}{n_k+n_0}} + h\right)$ when $k \in \mah$, $\lambda^{(k)[r]} = C\sqrt{\frac{\log p}{n_k+n_0}}\cdot (1 \vee \twonorm{\bbetak{k}-\bbeta} \vee \twonorm{\bwk{k}-\bbeta})$ when $k \in \mach$ and $\lambdak{0} = C\sqrt{\frac{\log p}{n_0}}$ for some sufficiently large constant $C > 0$. Then we have the following sufficient conditions to make Assumption \ref{asmp: idenfity a} hold for logistic, linear and Poisson regression models. Denote $\Omega = \sqrt{h'}\allowbreak \left(\frac{\log p}{\min_{k \in \mah}n_k + n_0}\right)^{1/4} + \left(\frac{s'\log p}{\min_{k \in \mah}n_k + n_0}\right)^{1/4}[(s\vee s')^{1/4}+\sqrt{h'}] + \left(\frac{\log p}{\min_{k \in \mah}n_k + n_0}\right)^{1/8} (h')^{1/4}[(s\vee s')^{1/8}+(h')^{1/4}]$.
	\begin{enumerate}[(i)]
		\item For logistic regression models, we require 
			\begin{align}
				\inf_{k \in \mah}n_k &\gg s\log p, \quad n_0 \gg \left\{[s\vee s' + (h')^2]\vee \Omega^2\right\}\cdot \log K, \\
				\inf_{k \in \mach}\twonorm{\bbetak{k}-\bbeta} &\gtrsim \left(\frac{s\log p}{n_0}\right)^{1/4} \vee 1 + \Omega, \quad h \ll s^{-1/2}.
			\end{align} 
		\item For linear models, we require 
			\begin{align}
				\inf_{k \in \mah}n_k &\gg s^2\log p, \quad n_0 \gg \left\{[(s\vee s')^2 + (h')^4]\vee [(s\vee s' + (h')^2)\Omega^2]\right\}\cdot \log K, \\
				\inf_{k \in \mach}\twonorm{\bbetak{k}-\bbeta} &\gtrsim \left(\frac{s^2\log p}{n_0}\right)^{1/4} \vee 1 + \left[(s')^{1/4} + \sqrt{h'}\right]\Omega, \quad h \ll s^{-1}.
			\end{align}
		\item For Poisson regression models, we require 
			\begin{align}
				\inf_{k \in \mah}n_k \gg s^2\log p, \quad n_0 &\gg \left[(s\vee s' + h')\vee \Omega^2\right]\cdot \log K, \quad U(s \vee s'+h \vee h') \lesssim 1,\\
				\inf_{k \in \mach}\twonorm{\bbetak{k}-\bbeta} &\gtrsim \left(\frac{s\log p}{n_0}\right)^{1/4} \vee 1 + \left[(s')^{1/4} + \sqrt{h'}\right]\Omega, \quad h \ll s^{-1}.
			\end{align}
	\end{enumerate}
\end{corollary}

Under Assumptions \ref{asmp: convexity}, \ref{asmp: x}, and the sufficient conditions derived  in Corollary \ref{cor: explicit asmp}, by Theorem \ref{thm: a detection}, we can conclude that $\widehat{\ma} = \mah$ for some $h$. Note that we don't impose Assumption \ref{asmp: merging} here. Remark \ref{rmk: h beneficial regime} indicates that, for $\mah$-Trans-GLM to have a faster convergence rate than Lasso on target data, we need $h \ll \sqrt{\frac{\log p}{n_0}}$ and $\nah \gg n_0$. Suppose $s' \asymp s$, $h' \lesssim s^{1/2}$. Then for logistic regression models, when $s\log K \ll n_0 \ll s\log p$, the sufficient condition implies $h \ll s^{-1/2} \ll \sqrt{\frac{\log p}{n_0}}$. For linear models, when $s^2\log K \ll n_0 \ll s^2\log p$, $h \ll s^{-1} \ll \sqrt{\frac{\log p}{n_0}}$. And for Poisson models, when $s\log K \ll n_0 \ll s^2\log p$, $h \ll s^{-1} \ll \sqrt{\frac{\log p}{n_0}}$. This implies that when target sample size $n_0$ is within certain regimes and there are many more source data points than target data points, Trans-GLM can lead to a better $\ell_2$-estimation error bound than the classical Lasso on target data.

\subsection{Theory on confidence interval construction}\label{subsec:inf theory}
In this section, we will establish the theory for our confidence interval construction procedure described in Algorithm \ref{algo: inf}. First, we would like to review and introduce some notations. In Section \ref{subsec:inf}, we defined $\bSigmak{k}_{\bbeta} = \te[\bxk{k}(\bxk{k})^T\psi''((\bxk{k})^T\bbeta)]$. Let $\bTheta = (\bSigmak{0}_{\bbeta})^{-1}$ and $\Kah  = \norm{\mah}$. Define
\begin{equation}
	\bgammak{k}_j = \argmin_{\bgamma \in \mathbb{R}^{p-1}} \te\left\{\psi''(\bbeta^T\bxk{k})\cdot [\bxk{k}_{j} - (\bxk{k}_{-j})^T\bgamma]^2 \right\} = (\bSigmak{k}_{\bbeta, -j, -j})^{-1}\bSigmak{k}_{\bbeta, -j, j},
\end{equation}
which is closely related to $(\bSigmak{k}_{\bbeta})^{-1}$ and $\bgammak{0}_j$ can be viewed as the population version of $\hbgammak{0}_j$. And $\bSigmak{k}_{\bbeta, j, -j}$ represents the $j$-th row without the $(j,j)$ diagonal element of $\bSigmak{k}_{\bbeta}$. $\bSigmak{k}_{\bbeta, -j, -j}$ denotes the submatrix of $\bSigmak{k}_{\bbeta}$ without the $j$-th row and $j$-th column.
Suppose 
\begin{align}
	\sup_{k \in \mah, j = 1:p}\onenorm{(\bSigmak{0}_{\bbeta, -j, -j})^{-1}\bSigmak{0}_{\bbeta, -j, j} - (\bSigmak{k}_{\bbeta, -j, -j})^{-1}\bSigmak{k}_{\bbeta, -j, j}} &\leq h_1, \\
	\sup_{k \in \mah, j = 1:p}\left[\norm{\bSigmak{k}_{\bbeta,j,j} - \bSigmak{0}_{\bbeta,j,j}}\vee \norm{(\bSigmak{k}_{\bbeta, j, -j} - \bSigmak{0}_{\bbeta, j, -j})\bgammak{0}_j}\right] &\leq h_{\max}.
\end{align}
Then by the definition of $\bgammak{k}_j$, 
\begin{equation}
	\sup_{k \in \mah, j = 1:p} \onenorm{\bgammak{k}_j-\bgammak{0}_j} \lesssim  h_1,
\end{equation}
which is similar to our previous setting $\sup_{k \in \mah}\onenorm{\bwk{k}-\bbeta} \leq h$. This motivates us to apply a similar two-step transfer learning procedure (steps 2-4 in Algorithm \ref{algo: inf}) to learn $\bgammak{0}_j$ for $j = 1, \ldots, p$. We impose the following set of conditions.
\begin{assumption}\label{asmp: inf}
	Suppose the following conditions hold:
	\begin{enumerate}[(i)]
		\item $\sup_{k \in \transet}\infnorm{\bxk{k}} \leq U < \infty$, $\sup_{k \in \transet}\norm{(\bxk{k})^T\bwk{k}} \leq U' < \infty$ a.s.;
		\item $\sup_j\zeronorm{\bgammak{0}_j}/s < \infty$, $\sup_{j \in 1:p, k \in \transet} |(\bxk{k})^T\bgammak{0}_j| \leq U'' < \infty$ a.s.;
		\item $\inf_{k \in \transet}\lambda_{\min}(\bSigmak{k}_{\bwk{k}}) \geq \underline{U} > 0$;
		\item $\sup\limits_{k \in \transet}\sup\limits_{|z|\leq \bar{U}} \psi'''((\bxk{k})^T\bwk{k} + z) \leq M_{\psi}< \infty$ a.s.
		\item $\sup_{k \in \transet}\onenorm{(\bm{\Sigma}_{\bbeta,-j,-j}^{\mah})^{-1}\bm{\Sigma}^{(k)}_{\bbeta, -j, -j}} < \infty$, where $\bm{\Sigma}^{\mah}_{\bbeta} = \sum_{k \in \transet}\alpha_k \bm{\Sigma}^{(k)}_{\bbeta}$;
		\item $\min_{k \in \mah}n_k \gtrsim n_0$, $n_0 \gg \frac{s^3(\log p)^2}{\Kah^2} \vee \Kah$, $\nah+n_0 \gg s^2\log p$;
		\item $h_1 \lesssim s^{-1/2} \wedge \left[\sqrt{\frac{n_0}{\log p}}\left(\frac{\sqrt{\Kah}}{s} \wedge 1\right)\right]$, $h_1 \vee h \ll \frac{\Kah n_0^{1/2}}{s^2(\log p)^{3/2}} \wedge \frac{n_0^{1/4}}{s^{1/2}(\log p)^{1/4}}$, $hh_1^{1/2} \ll n_0^{-1/4}(\log p)^{-1/4}\left(\frac{\Kah}{s} \wedge 1\right)$, $h^{5/2}h_1 \ll n_0^{-3/4}(s\log p)^{-1/4}$, $h_1 \ll \frac{\Kah^{1/2}n_0^{1/2}}{s^{3/2}(\log p)^{1/2}} \wedge \frac{\Kah^{3/2}n_0^{1/2}}{s^{5/2}(\log p)^{3/2}}$, $h_1h^{1/2} \ll \frac{n_0^{1/4}}{s(\log p)^{1/4}} \wedge \frac{\Kah n_0^{1/4}}{s^2(\log p)^{5/4}}$, $h \ll \frac{\Kah^{1/2}}{(s\log p)^{1/2}} \wedge \frac{1}{n_0^{1/4}(\log p)^{1/2}}$, $h_{\max} \ll s^{-1/2} \wedge \left(\frac{1}{s}\sqrt{\frac{\Kah}{\log p}}\right)$, $hh_{\max} \ll n_0^{-1/2}$.
	\end{enumerate}
\end{assumption}

\begin{remark}
	Conditions (\rom{1})-(\rom{3}) are motivated from conditions of Theorem 3.3 in \cite{van2014asymptotically}. Note that in \cite{van2014asymptotically}, they define $s_j = \zeronorm{\bgammak{0}_j}$ and treat $s_j$ and $s$ as two different parameters. Here we require $\sup_j s_j \lesssim s$ just for simplicity (otherwise condition (\rom{7}) would be more complicated). Condition (\rom{4}) requires the inverse link function to behave well, which is similar to Assumption \ref{asmp: second derivative}. Condition (\rom{5}) is similar to Assumption \ref{asmp: merging} to guarantee the success of the two-step transfer learning procedure to learn $\bgammak{0}$ in Algorithm \ref{algo: inf} with a fast rate. Without condition (\rom{5}), the conclusions in the following Theorem \ref{thm: inf} may still hold but under a stronger condition on $h$, $h_1$ and $h_{\max}$, and the rate \eqref{eq: inf var est bound} may be worse. We do not explore the details in this paper and leave them to interested readers. Conditions (\rom{6}) and (\rom{7}) require that the sample size is sufficiently large and the distance between target and source is not too large. In condition (\rom{6}), $\min_{k \in \mah}n_k \gtrsim n_0$ is not necessary and the only reason we add it here is to simplify condition (\rom{7}). 
\end{remark}

\begin{remark}
	When $\bxk{k}$'s are from the same distribution, $h_1 = h_{\max} = 0$. In this case, we can drop the debiasing step to estimate $\hbgammak{0}_j$ in Algorithm \ref{algo: inf} as well as condition (\rom{5}). Furthermore, condition (\rom{7}) can be significantly simplified and only $h \ll \frac{\Kah n_0^{1/2}}{s^2(\log p)^{3/2}} \wedge \frac{n_0^{1/4}}{s^{1/2}(\log p)^{1/4}} \wedge \frac{\Kah^{1/2}}{(s\log p)^{1/2}} \wedge \frac{1}{n_0^{1/4}(\log p)^{1/2}}$ is needed.  
\end{remark}

\begin{remark}
	From conditions (\rom{6}) and (\rom{7}), we can see that as long as $\Kah \lesssim s (\log p)^{2/3}$, the conditions become milder as $\Kah$ increases.
\end{remark}

Now, we are ready to present our main result for Algorithm \ref{algo: inf}.
\begin{theorem}\label{thm: inf}
	Under Assumptions \ref{asmp: convexity}-\ref{asmp: merging} and Assumption \ref{asmp: inf},
	\begin{equation}\label{eq: inf clt}
		\frac{\sqrt{n_0}(\hb_j - \beta_j)}{\sqrt{\hTheta_j^T\hSigma_{\hbeta}\hTheta_j}} \xrightarrow{d} \mathcal{N}(0, 1),
	\end{equation}
	and 
	\begin{align}
		\norm{\hTheta_j^T\hSigma_{\hbeta}\hTheta_j - \bTheta_{jj}} &\lesssim s\sqrt{\frac{\log p}{\nah+n_0}} + \sqrt{s}\left[h^{1/2}\left(\frac{\log p}{n_0}\right)^{1/4}\wedge h\right] + (sh_1)^{1/2}\left(\frac{\log p}{n_0}\right)^{1/4} \\
		&\quad + (sh_1)^{1/2}\left[\left(\frac{s\log p}{\nah+n_0}\right)^{1/4} + \left(h^{1/4}\left(\frac{\log p}{n_0}\right)^{1/8}\right) \wedge h^{1/2}\right] + \sqrt{s}h_{\max}, \label{eq: inf var est bound}
	\end{align}
	for $j = 1, \ldots, p$, with probability at least $1-\Kah n_0^{-1}$.
\end{theorem}

Theorem \ref{thm: inf} guarantees that under certain conditions, the $(1-\alpha)$-confidence interval for each coefficient component obtained in Algorithm \ref{algo: inf} has approximately level $(1-\alpha)$ when the sample size is large. Also, if we compare the rate of \eqref{eq: inf var est bound} with the rate $\mathcal{O}_p(s\sqrt{\log p/n_0})$ in \cite{van2014asymptotically} (see the proof of Theorem 3.1), we can see that when $h \ll s\sqrt{\frac{\log p}{n_0}}$, $h_1 \ll \sqrt{\frac{s\log p}{n_0}} \cdot \left[s^{1/2}\wedge \left(\frac{\nah+n_0}{n_0}\right)^{1/4}\right]$, $h_1^{1/2}h^{1/4} \ll s^{1/2}\left(\frac{\log p}{n_0}\right)^{3/8}$ and $h_{\max} \ll \sqrt{\frac{s\log p}{n_0}}$, the rate is better than that of desparsified Lasso using only target data.

\section{Numerical Experiments}\label{sec:numerical}
In this section, we demonstrate the power of our GLM transfer learning algorithms via extensive simulation studies and a real-data application. In the simulation part, we study the performance of different methods under various settings of $h$. The methods include \textit{Trans-GLM} (Algorithm \ref{algo: a unknown}), \textit{na\"ive-Lasso} (Lasso on target data), \textit{$\mah$-Trans-GLM} (Algorithm \ref{algo: merging} with $\ma = \mah$) and \textit{Pooled-Trans-GLM} (Algorithm \ref{algo: merging} with all sources). In the real-data study, besides na\"ive-Lasso, Pooled-Trans-GLM, and Trans-GLM, additional methods are explored for comparison, including support vector machines (SVM), decision trees (Tree), random forests (RF) and Adaboost algorithm with trees (Boosting). We run these benchmark methods twice. First, we fit the models on only the target data, then at the second time, we fit them a combined data of target and all sources, which is called a pooled version. We use the original method name to denote the corresponding method implemented on target data, and add a prefix ``Pooled" to denote the corresponding method implemented on target and all source data. For example, Pooled-SVM represents SVM fitted on all data from target and sources.

All experiments are conducted in R. We implement our GLM transfer learning algorithms in a new R package \texttt{glmtrans}, which is available on CRAN. More implementation details can be found in Section \ref{subsubsec: implementation supp} in the supplements.
%
\subsection{Simulations}\label{subsec:simulations}

\subsubsection{Transfer learning on $\mah$}\label{subsubsec: oracle}
In this section, we study the performance of $\mah$-Trans-GLM and compare it with that of na\"ive-Lasso. The purpose of the simulation is to verify that $\mah$-Trans-GLM can outperform na\"ive-Lasso in terms of the target coefficient estimation error, when $h$ is not too large.

Consider the simulation setting as follows. We set the target sample size $n_0 = 200$ and source sample sample size $n_k = 100$ for each $k \neq 0$. The dimension $p = 500$ for both target and source data. For the target, the coefficient is set to be $\bbeta = (0.5\cdot\bm{1}_{s}, \bm{0}_{p-s})^T$, where $\bm{1}_{s}$ has all $s$ elements 1 and $\bm{0}_{p-s}$ has all $(p-s)$ elements 0, where $s$ is set to be 5. Denote $\bm{\mathscr{R}}_p^{(k)}$ as $p$ independent Rademacher variables (being $-1$ or $1$ with equal probability) for any $k$. $\bm{\mathscr{R}}_p^{(k)}$ is independent with $\bm{\mathscr{R}}_p^{(k')}$ for any $k \neq k'$. For any source data $k$ in $\mah$, we set $\bwk{k} = \bbeta + (h/p)\bm{\mathscr{R}}_p^{(k)}$. For linear and logistic regression models, predictors from target $\bx^{(0)}_i \overset{i.i.d.}{\sim} \mathcal{N}(\bm{0}_p, \bm{\Sigma})$ with $\bm{\Sigma} = [\Sigma_{jj'}]_{p \times p}$ where $\Sigma_{jj'}=0.5^{|j-j'|}$, for all $i=1,\cdots, n$. And for $k\in\mah$, we generate $p$-dimensional predictors from $\mathcal{N}(\bm{0}_p, \bm{\Sigma}+\bm{\epsilon}\bm{\epsilon}^T)$, where $\bm{\epsilon} \sim \mathcal{N}(\bm{0}_p, 0.3^2\bm{I}_p)$ and is independently generated. For Poisson regression model, predictors are from the same Gaussian distributions as those in linear and binomial cases with coordinate-wise truncation at $\pm 0.5$.

Note that na\"ive-Lasso is fitted on only target data, and $\mah$-Trans-GLM denotes Algorithm \ref{algo: merging} on source data in $\mah$ as well as target data. We train na\"ive-Lasso and $\mah$-Trans-GLM models under different settings of $h$ and $\Kah$, then calculate the $\ell_2$-estimation error of $\bbeta$. All the experiments are replicated  200 times and the average $\ell_2$-estimation errors of $\mah$-Trans-GLM and na\"ive-Lasso under linear, logistic, and Poisson regression models are shown in Figure \ref{fig: oracle1}.

\begin{figure}[!h]
	\centering
	\includegraphics[width=\textwidth]{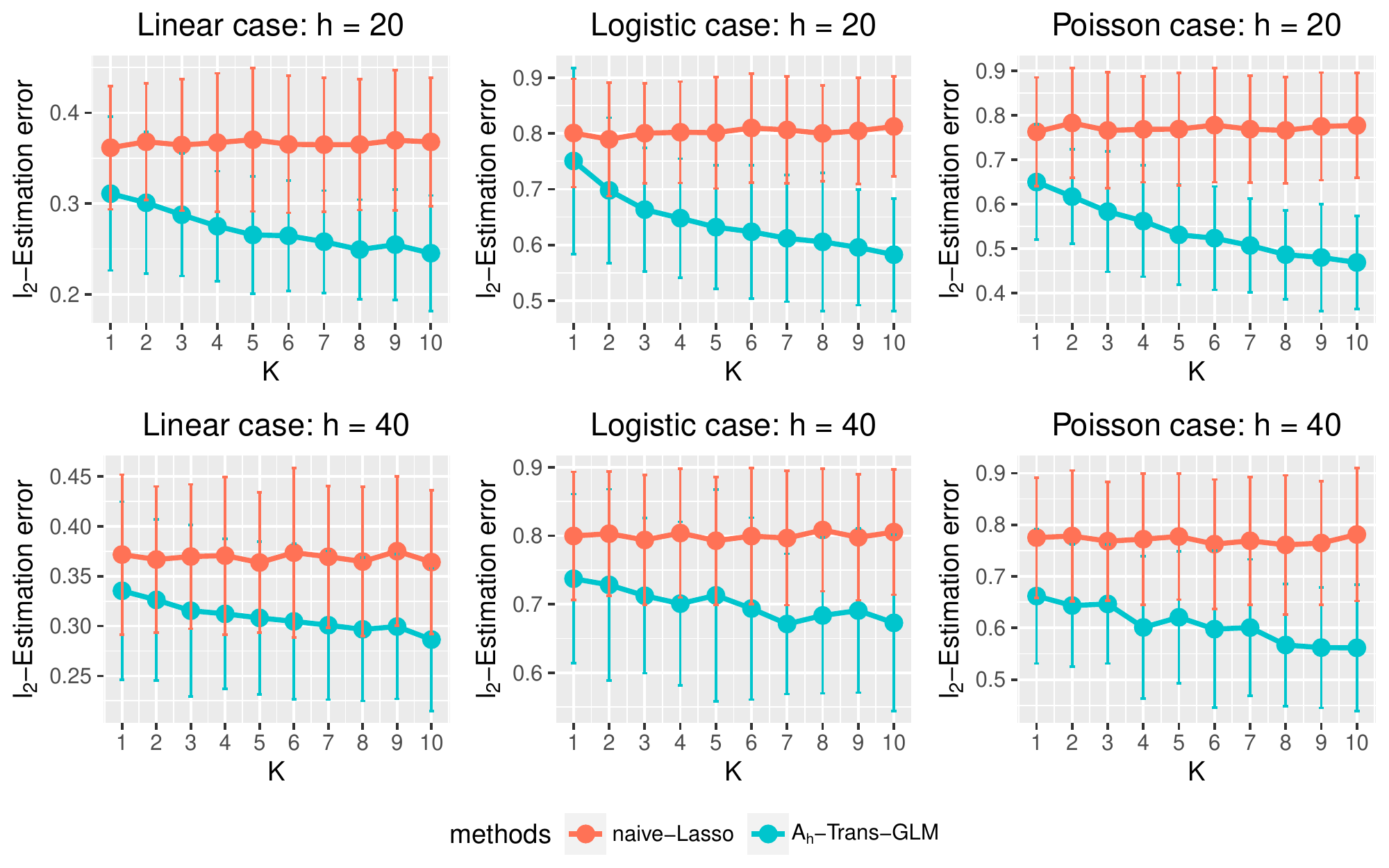}
	\caption{The average $\ell_2$-estimation error of $\mah$-Trans-GLM and na\"ive-Lasso under linear, logistic and Poisson regression models with different settings of $h$ and $K$. $n_0 = 200$ and $n_k = 100$ for all $k = 1, \ldots, p$, $p = 500$, $s =5$. Error bars denote the standard deviations.}
	\label{fig: oracle1}
\end{figure}

From Figure \ref{fig: oracle1}, it can be seen that $\mah$-Trans-GLM outperforms na\"ive-Lasso for most combinations of $h$ and $K$. As more and more source data become available, the performance of $\mah$-Trans-GLM improves significantly. This matches our theoretical analysis because the $\ell_2$-estimation error bounds in Theorems \ref{thm: l2 with assumption} and \ref{thm: l2 without assumption} become sharper as $\nah$ grows. When $h$ increases, the performance of $\mah$-Trans-GLM becomes worse. 

\begin{figure}[!h]
	\centering
	\includegraphics[width=\textwidth]{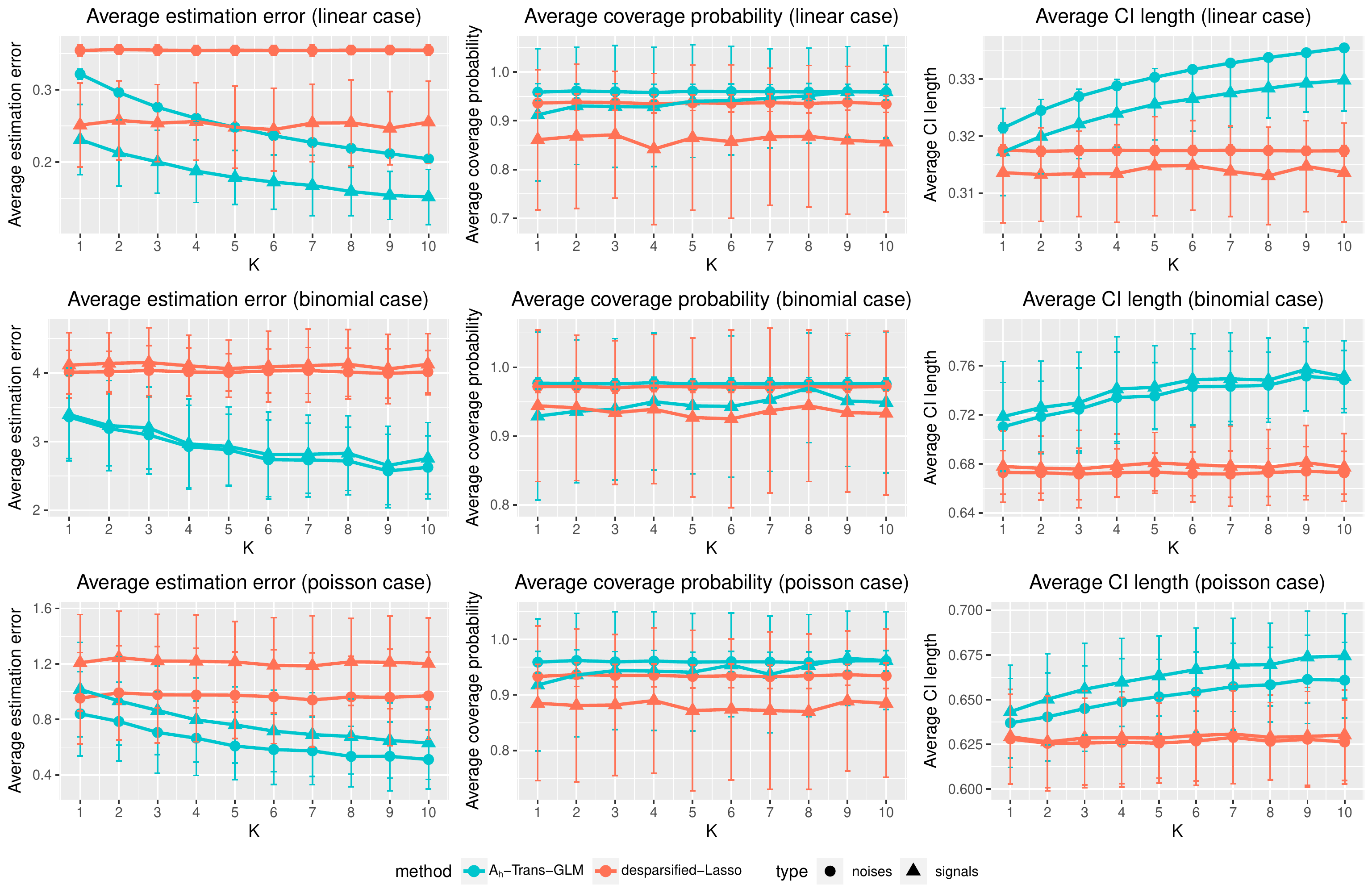}
	\caption{Three evaluation metrics of Algorithm \ref{algo: inf} with $\mah$ (we denote it as $\mah$-Trans-GLM) and desparsified Lasso on target data, under linear, logistic and Poisson regression models, with different settings of $K$. $h$ = 20. $n_0 = 200$ and $n_k = 100$ for all $k = 1, \ldots, p$, $p = 500$, $s =5$. Error bars denote the standard deviations.}
	\label{fig: oracle_inf}
\end{figure} 

We also apply the inference algorithm \ref{algo: inf} with $\mah$ and compare it with desparsified Lasso \citep{van2014asymptotically} on only target data. Recall the notations we used in Section \ref{subsec:inf theory}. Here we consider $95\%$ confidence intervals (CIs) for each component of coefficient $\bbeta$, and report three evaluation metrics in Figure \ref{fig: oracle_inf} when $h = 20$ under different $\Kah$: (\rom{1}) the average of estimation error of $\bTheta_{jj}$ over variables in the signal set $S$ and noise set $S^c$ (including the intercept), respectively (which we call ``average estimation error"); (\rom{2}) the average CI coverage probability over variables in the signal set $S$ and noise set $S^c$; (\rom{3}) the average CI length over $j \in$ signal set $S$ and noise set $S^c$. Note that there is no explicit formula of $\bTheta$ for logistic and Poisson regression models. Here we approximated it through $5 \times 10^6$ Monte-Carlo simulations. Notice that the average estimation error of $\mah$-Trans-GLM declines as $K$ increases, which agrees with our theoretical analysis in Section \ref{subsec:inf theory}. As for the coverage probability, although CIs obtained by desparsified Lasso can achieve $95\%$ coverage probability on $S^c$ in linear and binomial cases, it fails to meet the $95\%$ requirement of coverage probability on $S$ in all three cases. In contrast, CIs provided by $\mah$-Trans-GLM can achieve approximately $95\%$ level when $K$ is large on both $S$ and $S^c$. Finally, the results of average CI length reveal that the CIs obtained by $\mah$-Trans-GLM tend to be wider as $K$ increases. Considering this together with the average estimation error and coverage probability, a possible explanation could be that desparsified Lasso might down-estimate $\bTheta_{jj}$ which leads to too narrow CIs to cover the true coefficients. And $\mah$-Trans-GLM offers a more accurate estimate of $\bTheta_{jj}$ which results in wider CIs.

We also consider different $(\{n_k\}_{k=0}^K, p, s)$ settings with the results in the supplements.

\subsubsection{Transfer learning when $\mah$ is unknown}\label{subsubsec: unknown simulation}
Different from the previous subsection, now we fix the total number of sources as $K = 10$. There are two types of sources, which belong to either $\mah$ or $\mah^c$. Sources from $\mah$ have similar coefficients to the target one, while the coefficients of sources from $\mah$ can be quite different. Intuitively, using more sources from $\mah$ benefits the estimation of the target coefficient. But in practice, $\mah$ may not be known as a priori. As we argued before, Trans-GLM can detect useful sources automatically, therefore it is expected to be helpful in such a scenario. Simulations in this section aim to justify the effectiveness of Trans-GLM.

Here is the detailed setting. We set the target sample size $n_0 = 200$ and source sample sample size $n_k = 200$ for all $k \neq 0$. The dimension $p = 2000$. Target coefficient is the same as the one used in Section \ref{subsubsec: oracle} and we fix the signal number $s = 20$. Recall $\bm{\mathscr{R}}_p^{(k)}$ denotes $p$ independent Rademacher variables and $\bm{\mathscr{R}}_p^{(k')}$ are independent for any $k \neq k'$. Consider $h = 20$ and $40$. For any source data $k$ in $\mah$, we set $\bwk{k} = \bbeta + (h/p)\bm{\mathscr{R}}_p^{(k)}$. For linear and logistic regression models, predictors from target $\bx^{(0)}_i \overset{i.i.d.}{\sim} N(\bm{0}, \bm{\Sigma})$ with $\bm{\Sigma} = [\Sigma_{jj'}]_{p \times p}$ where $\Sigma_{jj'}=0.9^{|j-j'|}$, for all $i=1,\cdots, n_0$. For the source, we generate $p$-dimensional predictors from independent $t$-distribution with degrees of freedom 4. For the target and sources of Poisson regression model, we generate predictors from the same Gaussian distribution and $t$-distribution respectively, and truncate each predictor at $\pm 0.5$. 

To generate the coefficient $\bwk{k}$ for $k \notin \mah$, we randomly generate $S^{(k)}$ of size $s$ from $\{2s+1, \ldots, p\}$. Then, the $j$-th component of coefficient $\bw^{(k)}$ is set to be 
\begin{equation}
	w^{(k)}_j = \begin{cases}
		0.5 + 2hr_j^{(k)}/p, & j \in \{s+1, \ldots, 2s\}\cup S^{(k)},\\
		2hr_j^{(k)}/p, & \textrm{otherwise},
	\end{cases}
\end{equation}
where $r_j^{(k)}$ is a Rademacher variable. We also add an intercept $0.5$. The generating process of each source data is independent. Compared to the setting in Section \ref{subsubsec: oracle}, the current setting is more challenging because source predictors come from $t$-distribution with heavier tails than sub-Gaussian tails. However, although Assumption \ref{asmp: x} is violated, in the following analysis, we will see that Trans-GLM can still succeed in detecting informative sources.

As before, we fit na\"ive-Lasso  on only target data. $\mah$-Trans-GLM and Pooled-Trans-GLM represent Algorithm \ref{algo: merging} on source data in $\mah$ and target data or all sources and target data, respectively. Trans-GLM runs Algorithm \ref{algo: a unknown} by first identifying the informative source set $\widehat{\ma}$, then applying Algorithm \ref{algo: merging} to fit the model on sources in $\widehat{\ma}$. We vary the values of $\Kah$ and $h$, and repeat simulations in each setting 200 times. The average $\ell_2$-estimation errors are summarized in Figure \ref{fig: unknown3}.

\begin{figure}[!h]
	\centering
	\includegraphics[width=\textwidth]{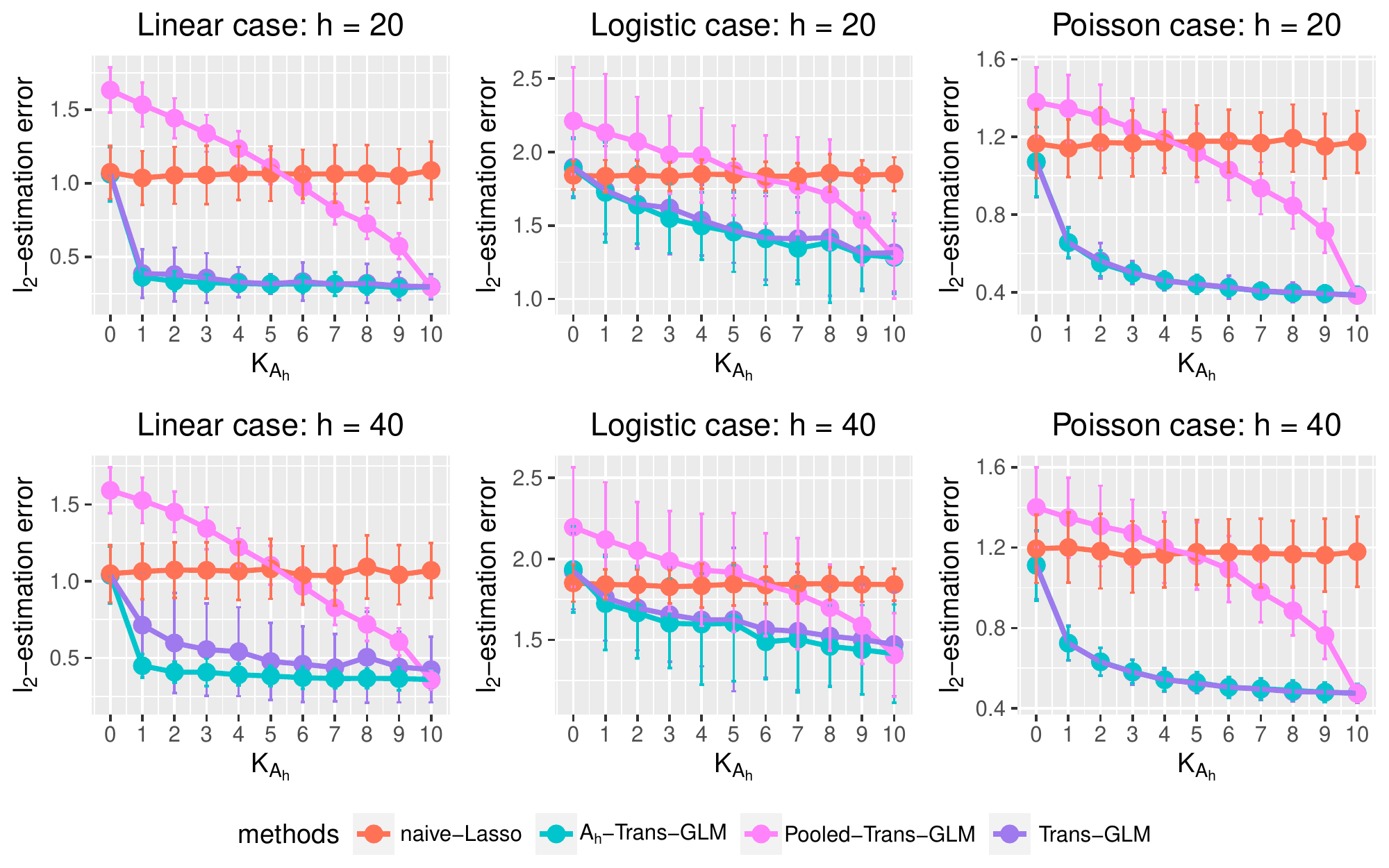}
	\caption{The average $\ell_2$-estimation error of various models with different settings of $h$ and $\Kah$ when $K = 10$. $n_k = 200$ for all $k = 0, \ldots, K$, $p = 2000$, $s = 20$. Error bars denote the standard deviations.}
	\label{fig: unknown3}
\end{figure} 

From Figure \ref{fig: unknown3}, it can be observed that in all three models, $\mah$-Trans-GLM always achieves the best performance as expected since it transfers information from sources in $\mah$. Trans-GLM mimics the behavior of $\mah$-Trans-GLM very well, implying that the transferable source detection algorithm can successfully recover $\mah$. When $\Kah$ is small, Pooled-Trans-GLM performs worse than na\"ive-Lasso because of the negative transfer. As $\Kah$ increases, the performance of Pooled-Trans-GLM improves and finally matches those of $\mah$-Trans-GLM and Trans-GLM when $\Kah = K = 10$.

\subsection{A real-data study}\label{subsec:real data}
\begin{figure}[!h]
\centering
  \includegraphics[width=0.7\textwidth]{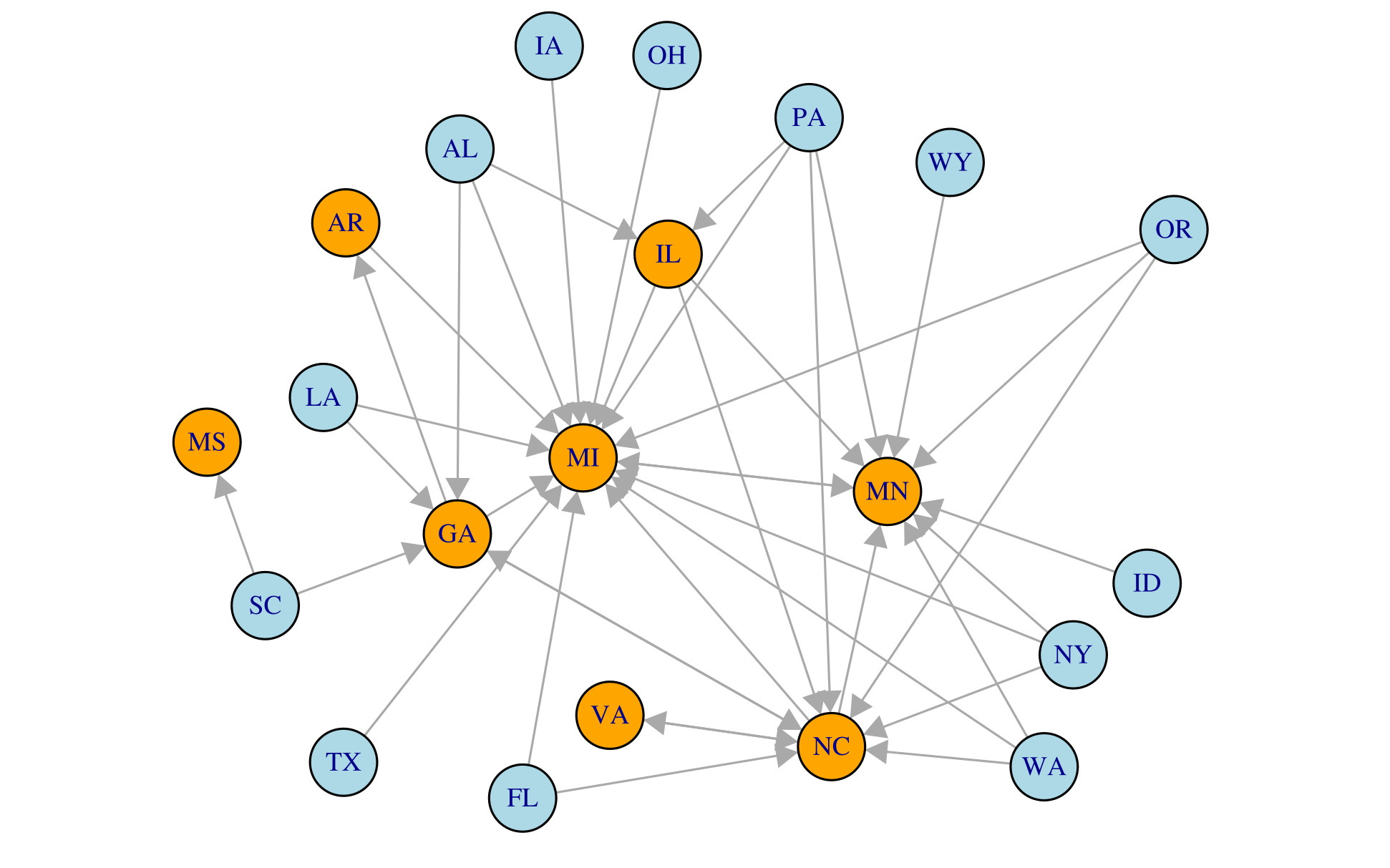}
  \caption{The transferability between different states for Trans-GLM.}
  \label{fig: network}
\end{figure}

In this section, we study the 2020 US presidential election results of each county. We only consider win or lose between two main parties, Democrats and Republicans, in each county. The 2020 county-level election result is available at \url{https://github.com/tonmcg/US_County_Level_Election_Results_08-20}. The response is the election result of each county. If Democrats win, we denote this county as class 1, otherwise, we denote it as class 0. And we also collect the county-level information as the predictors, including the population and race proportions, from \url{https://www.kaggle.com/benhamner/2016-us-election}. 

The goal is to explore the relationship between states in the election using transfer learning. We are interested in swing states with a large number of counties. Among 49 states (Alaska and Washington, D.C. excluded), we select the states where the proportion of counties voting Democrats falls in $[10\%, 90\%]$, and have at least 75 counties as target states. They include Arkansas (AR), Georgia (GA), Illinois (IL), Michigan (MI), Minnesota (MN), Mississippi (MS), North Carolina (NC), and Virginia (VA). 

The original data includes 3111 counties and 52 county-level predictors. We also consider the pairwise interaction terms between predictors. After pre-processing, there are 3111 counties and 1081 predictors in the final data, belonging to 49 US states.

We would like to investigate which states have a closer relationship with these target states by our transferable source detection algorithm. For each target state, we use it as the target data and the remaining 48 states as source datasets. Each time we randomly sample 80\% of target data as training data and the remaining 20\%  is used for testing. Then we run Trans-GLM (Algorithm \ref{algo: a unknown}) and see which states are in the estimated transferring set $\widehat{\mathcal{A}}$. We repeat the simulation 500 times and count the transferring frequency between every state pair. The 25 (directed) state pairs with the highest transferring frequencies are visualized in Figure \ref{fig: network}. Each orange node represents a target state we mentioned above and blue nodes are source states. States with the top 25 transferring frequencies are connected with a directed edge.

\begin{table}[!h]
\renewcommand{\arraystretch}{0.75}
\begin{center}
\begin{threeparttable}
\resizebox{\textwidth}{!}{\begin{tabular}{lccccccccc}
\Xhline{1.25pt}
\multirow{2}{*}{Methods} & \multicolumn{8}{c}{Target states} \\ 
\cline{2-9}
&AR&GA&IL&MI&MN&MS&NC&VA\\
\hline
na\"ive-Lasso& \multicolumn{1}{r}{4.79\textsubscript{3.36} }& 6.98\textsubscript{3.90} & 5.73\textsubscript{4.14} & 11.49\textsubscript{2.44} & 12.46\textsubscript{2.70} & 7.53\textsubscript{6.57} & 15.60\textsubscript{6.73} & 9.48\textsubscript{4.88} \\
Pooled-Lasso& \multicolumn{1}{r}{3.59\textsubscript{4.71} }& 9.98\textsubscript{4.22} & 7.89\textsubscript{5.56} & 7.04\textsubscript{5.80} & \textit{10.38}\textsubscript{5.18} & 22.01\textsubscript{7.18} & 12.73\textsubscript{5.35} & 21.44\textsubscript{5.46} \\
Pooled-Trans-GLM& \multicolumn{1}{r}{\textit{1.83}\textsubscript{3.12} }& \textit{4.86}\textsubscript{3.60} & \textit{2.52}\textsubscript{3.55} & \textit{5.62}\textsubscript{4.54} & 10.75\textsubscript{5.60} & \textit{7.23}\textsubscript{6.65} & \textit{9.71}\textsubscript{5.75} & \textbf{7.15}\textsubscript{4.23} \\
Trans-GLM& \multicolumn{1}{r}{\textbf{1.54}\textsubscript{2.94} }& \textit{4.74}\textsubscript{3.54} & \textbf{2.51}\textsubscript{3.45} & \textit{5.53}\textsubscript{4.73} & \textbf{10.34}\textsubscript{5.73} & 7.24\textsubscript{6.81} & \textbf{9.34}\textsubscript{5.57} & \textit{7.18}\textsubscript{4.67} \\
SVM& \multicolumn{1}{r}{6.71\textsubscript{1.70} }& 17.09\textsubscript{3.89} & 7.00\textsubscript{5.40} & 12.59\textsubscript{1.87} & 13.29\textsubscript{2.29} & 23.92\textsubscript{8.90} & 12.66\textsubscript{6.86} & 10.78\textsubscript{5.29} \\
Pooled-SVM& \multicolumn{1}{r}{7.84\textsubscript{6.32} }& 13.47\textsubscript{4.73} & 7.75\textsubscript{5.24} & 7.58\textsubscript{6.40} & 13.01\textsubscript{5.69} & 27.32\textsubscript{8.72} & 12.30\textsubscript{5.75} & 17.31\textsubscript{5.46} \\
Tree& \multicolumn{1}{r}{\textit{2.23}\textsubscript{3.58} }& 8.37\textsubscript{4.40} & 4.62\textsubscript{5.27} & 10.05\textsubscript{5.53} & 10.97\textsubscript{8.42} & \textit{5.97}\textsubscript{5.26} & 18.29\textsubscript{8.01} & 14.46\textsubscript{6.88} \\
Pooled-Tree& \multicolumn{1}{r}{7.81\textsubscript{6.89} }& 7.68\textsubscript{4.59} & 4.63\textsubscript{4.26} & 7.42\textsubscript{6.18} & 10.53\textsubscript{5.91} & 16.73\textsubscript{7.30} & 14.76\textsubscript{7.26} & 17.43\textsubscript{5.85} \\
RF& \multicolumn{1}{r}{3.60\textsubscript{3.57} }& 6.04\textsubscript{3.59} & 4.08\textsubscript{3.98} & 6.42\textsubscript{4.79} & \textit{10.51}\textsubscript{5.10} & 7.27\textsubscript{5.72} & 11.29\textsubscript{6.29} & 7.73\textsubscript{4.77} \\
Pooled-RF& \multicolumn{1}{r}{3.73\textsubscript{4.82} }& 7.49\textsubscript{3.90} & 4.35\textsubscript{3.63} & \textbf{5.34}\textsubscript{4.99} & 10.86\textsubscript{4.96} & 12.56\textsubscript{6.88} & 11.04\textsubscript{6.03} & 10.40\textsubscript{5.18} \\
Boosting& \multicolumn{1}{r}{2.23\textsubscript{3.58} }& \textbf{4.65}\textsubscript{3.77} & \textit{2.55}\textsubscript{3.82} & 7.79\textsubscript{5.52} & 10.64\textsubscript{6.51} & \textbf{5.28}\textsubscript{5.16} & 10.88\textsubscript{6.47} & \textit{7.53}\textsubscript{5.10} \\
Pooled-Boosting& \multicolumn{1}{r}{3.10\textsubscript{4.84} }& 5.71\textsubscript{3.53} & 3.82\textsubscript{3.85} & 5.81\textsubscript{5.27} & 11.21\textsubscript{5.13} & 14.31\textsubscript{7.42} & \textit{10.82}\textsubscript{5.99} & 11.95\textsubscript{5.25} \\
\Xhline{1.25pt}
\end{tabular}}
\end{threeparttable}
\caption{The average test error rate (in percentage) of various methods with different targets among 500 replications. The cutoff for all binary classification methods is set to be 1/2. Numbers in the subscript indicate the standard deviations. }
\label{table: error}
\end{center}
\end{table}

From Figure \ref{fig: network}, we observe that Michigan has a strong relationship with other states, since there is a lot of information transferable when predicting the county-level results in Michigan, Minnesota, and North Carolina. Another interesting finding is that states which are geographically close to each other may share more similarities. For instance, see the connection between Indiana and Michigan, Ohio and Michigan, North Carolina and Virginia, South Carolina and Georgia, Alabama and Georgia, etc. In addition, one can observe that states in the Rust Belt also share more similarities. As examples, see the edges among Pennsylvania, Minnesota, Illinois, Michigan, New York, and Ohio, etc. 

To further verify the effectiveness of our GLM transfer learning framework on this dataset and make our findings more convincing, we calculate the average test misclassification error rates for each of the eight target states. For comparison, we compare the performances of Trans-GLM and Pooled-Trans-GLM with na\"ive-Lasso, SVM, trees, random forests (RF), boosting trees (Boosting) as well as their pooled version. Average test errors and the standard deviations of various methods are summarized in Table \ref{table: error}. The best method and other top three methods for each target are highlighted in bold and italics, respectively.

Table \ref{table: error} shows that in four out of eight scenarios, Trans-GLM performs the best among all approaches. Moreover, Trans-GLM is always ranked in the top three except in the case of target state MS. This verifies the effectiveness of our GLM transfer learning algorithm. Besides, Pooled-Trans-GLM can always improve the performance of na\"ive-Lasso, while for other methods, pooling the data can sometimes lead to worse performance than that of the model fitted on only the target data. This marks the success of our two-step transfer learning framework and the importance of the debiasing step. Combining the results with Figure \ref{fig: network}, it can be seen that the performance improvement of Trans-GLM (compared to na\"ive-Lasso) for the target states with more connections (share more similarities with other states) are larger. For example, Trans-GLM outperforms na\"ive-Lasso a lot on Michigan, Minnesota and North Carolina, while it performs similarly to na\"ive-Lasso on Mississippi. 

We also try to identify significant variables by Algorithm \ref{algo: inf}. Due to the space limit, we put the results and analysis in Section \ref{subsubsec: election inf supp} of supplements. Interested readers can find the details there. Furthermore, since we have considered all main effects and their interactions, one reviewer pointed out that besides the classical Lasso penalty, there are other variants like group Lasso \citep{yuan2006model} or Lasso with hierarchy restriction \citep{bien2013lasso}, which may bring better practical performance and model interpretation. To be consistent with our theories, we only consider the Lasso penalty here and leave other options for future study.

\section{Discussions}\label{sec:discussions}
In this work, we study the GLM transfer learning problem. To the best of our knowledge, this is the first paper to study high-dimensional GLM under a transfer learning framework, which can be seen as an extension to \cite{bastani2021predicting} and \cite{li2021transfer}. We propose GLM transfer learning algorithms, and derive bounds for $\ell_1/\ell_2$-estimation error and a prediction error measure with fast and slow rates under different conditions. In addition, to avoid the negative transfer, an algorithm-free transferable source detection algorithm is developed and its theoretical properties are presented in detail. Moreover, we accommodate the two-step transfer learning method to construct confidence intervals of each coefficient component with theoretical guarantees. Finally, we demonstrate the effectiveness of our algorithms via  simulations and a real-data study.

There are several promising future avenues that are worth further research. The first interesting problem is how to extend the current framework and theoretical analysis to other models, such as multinomial regression and the Cox model. Second, Algorithm \ref{algo: merging} is shown to achieve the minimax $\ell_1/\ell_2$ estimation error rate when the homogeneity assumption (Assumption \ref{asmp: merging}) holds. Without homogeneity of predictors between target and source, only sub-optimal rates are obtained. This problem exists in the line of most existing high-dimensional transfer learning research \citep{bastani2021predicting, li2021transfer, li2020transfer2}. It remains unclear how to achieve the minimax rate when source predictors' distribution deviates a lot from the target one. Another promising direction is to explore similar frameworks for other machine learning models, including support vector machines, decision trees, and random forests.

\section*{Acknowledgments}
We thank the editor, the AE, and anonymous reviewers for their insightful comments which have greatly improved the scope and quality of the paper. \if1\blind
{This work was supported by NIH grant 1R21AG074205-01, NYU University Research Challenge Fund, and through the NYU IT High Performance Computing resources, services, and staff expertise.} \fi


\begin{appendices}
\end{appendices}

\spacingset{0.87} 
{\small
\bibliography{reference.bib}}
\bibliographystyle{apalike}
\spacingset{1.5} 

\newpage
\setcounter{page}{1}
\setcounter{section}{0}
\renewcommand{\thesection}{S.\arabic{section}}
\renewcommand{\theequation}{\thesection.\arabic{equation}}
\begin{center}
{\large\bf Supplementary Materials of ``Transfer Learning with High-dimensional Generalized Linear Models"}
\end{center}
\renewcommand{\cftsecnumwidth}{30pt}
\tableofcontents
\addtocontents{toc}{\protect\setcounter{tocdepth}{5}}

\section{More details}

\subsection{A schematic to illustrate $\ma$-Trans-GLM}\label{subsec: schematic}

To better illustrate Algorithm \ref{algo: merging}, we draw a schematic in Figure \ref{fig: schematic}. The blue point represents the target coefficient $\bbeta = \bwk{0}$ and the surrounding blue circle represents the estimation error $\mathcal{O}_p\left(\sqrt{\frac{s\log p}{n_0}}\right)$. The purple point denotes the estimator $\hat{\bbeta}_{\textup{na\"ive-Lasso}}$ from the classical Lasso with only target data. The orange point represents $\bwah$, which is the population version of the rough estimator we obtain from the transferring step by pooling target and source data in $\ma$ (see Section \ref{subsec:glm tranfer alg}), and the surrounding orange circle denotes its estimation error. It can be seen that $\bwah$ is a pooled version of $\{\bwk{k}\}_{k \in \{0\}\cup \mah}$, which is close to $\bbeta$ when $h$ is small. Starting from an initial estimate of $\bbeta$, the transferring step of $\ma$-Trans-GLM algorithm updates the estimate to $\hwah$ (an estimate of $\bwah$ based on source data in $\ma$ and the target data), then the debiasing step yields the final estimator $\hat{\bbeta}_{\ma\textup{-Trans-GLM}}$. 

\begin{figure}[!h]
	\centering
	\includegraphics[width=0.8\textwidth]{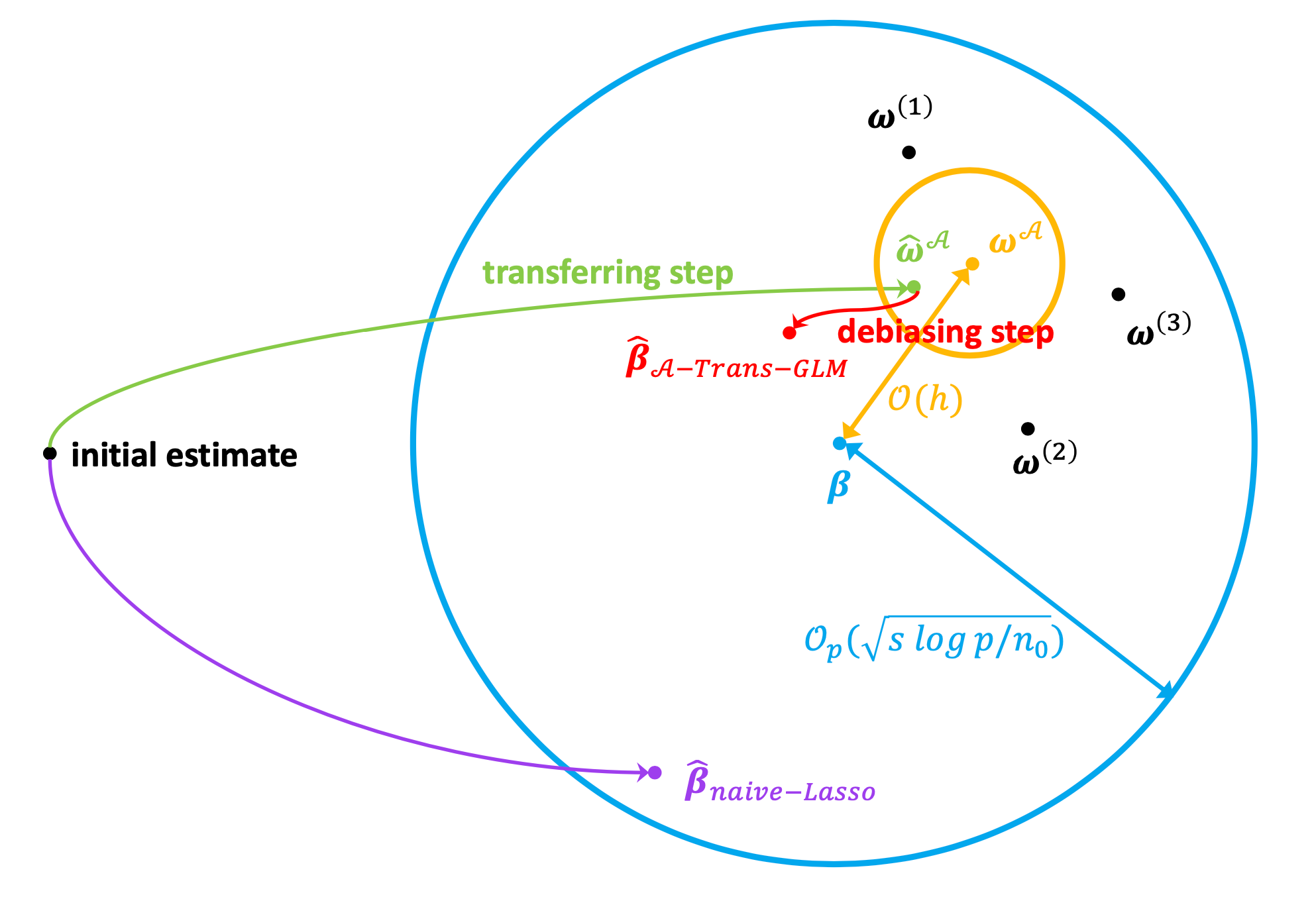}
	\caption{A schematic of $\ma$-Trans-GLM (Algorithm \ref{algo: merging}). $\ma = \{1, 2, 3\}$.}
	\label{fig: schematic}
\end{figure}

\subsection{Theory}\label{subsec: supp theory}

\subsubsection{Explicit forms of convergence rates in Assumption \ref{asmp: idenfity a}}\label{subsubsec: explicit forms}
\begin{proposition}[Explicit forms of $\Upsilon^{(k)}_1$, $\Gamma_1^{(k)}$, $\Gamma_2^{(k)}$, $g_1^{(k)}$ and $g_2^{(k)}$ for certain families]\label{prop: gamma_1 gamma_2} 
Denote
\begin{equation}
	\Omega_k = \begin{cases}
		\sqrt{\frac{s \log p}{n_0}}, &\quad k = 0,\\
		\sqrt{\frac{s\log p}{n_k + n_0}}  + \left(\frac{\log p}{ n_k+n_0}\right)^{1/4}\sqrt{h} + \sqrt{s}h,& \quad k \in \mah \\
		h'\sqrt{\frac{\log p}{n_k + n_0}} + \sqrt{\frac{s'\log p}{n_k + n_0}}\cdot W_k  + \left(\frac{\log p}{n_k+n_0}\right)^{1/4}\sqrt{h'W_k}, & \quad k \in \mach,
	\end{cases}
\end{equation}
where $W_k = 1 \vee \twonorm{\bbetak{k}-\bbeta} \vee \twonorm{\bbetak{k}-\bwk{k}}$.  Assume Assumptions \ref{asmp: convexity} and \ref{asmp: x} hold. For Poisson model, it is further required that $h \leq U^{-1}\bar{U}$ and $U\sup_{k \in \mac}\{\onenorm{\bbetak{k}-\bbeta}\vee \onenorm{\bbetak{k}-\bwk{k}}\} \leq \bar{U}$. With $\lambda^{(k)[r]} = C\left(\sqrt{\frac{\log p}{n_k+n_0}} + h\right)$ when $k \in \mah$, $\lambda^{(k)[r]} = C\sqrt{\frac{\log p}{n_k+n_0}}\cdot W_k$ when $k \in \mach$ and $\lambdak{0} = C\sqrt{\frac{\log p}{n_0}}$ for some sufficiently large constant $C > 0$, we have the following explicit forms of $\Gamma_1^{(k)}$, $\Gamma_2^{(k)}$, $\Upsilon_1^{(k)}$, $g_1^{(k)}$ and $g_2^{(k)}$ for logistic, linear and Poisson regression models.
	\begin{enumerate}[(i)]
		\item For the logistic regression model:
			\begin{align}
				\Gamma_1^{(0)} &= \sqrt{\frac{s\log p}{n_0}}, \quad \Gamma_2^{(0)} =\twonorm{\bbeta}/\sqrt{n_0},\\
				\Upsilon_1^{(k)} &= \Omega_k, \quad \Gamma_1^{(k)} = \sqrt{\frac{1}{n_0}}\Omega_k, \quad
				\Gamma_2^{(k)} = \sqrt{\frac{1}{n_0}} \cdot\left[\twonorm{\bwk{k}}\mathds{1}(k \in \mah) + \twonorm{\bbetak{k}}\mathds{1}(k \in \mach)\right], \\
				g_1^{(k)}(\zeta)  &= g_2^{(k)}(\zeta) = \exp(-\zeta^2).
			\end{align}
		\item For the linear model:
			\begin{align}
				\Gamma_1^{(0)} &= \sqrt{\frac{s\log p}{n_0}}\cdot \twonorm{\bbeta}, \quad \Gamma_2^{(0)} =(\twonorm{\bbeta}^2 \vee \twonorm{\bbeta})/\sqrt{n_0},\\
				\Upsilon_1^{(k)} &= \Omega_k\cdot \left[\twonorm{\bwk{k}}\mathds{1}(k \in \mah) + \twonorm{\bbetak{k}}\mathds{1}(k \in \mach)\right],\\
				\Gamma_1^{(k)} &= \sqrt{\frac{1}{n_0}}\Omega_k\cdot \left[\twonorm{\bwk{k}}\mathds{1}(k \in \mah) + \twonorm{\bbetak{k}}\mathds{1}(k \in \mach)\right],\\
				\Gamma_2^{(k)} &= \sqrt{\frac{1}{n_0}} \left[(\twonorm{\bwk{k}}^2 \vee \twonorm{\bwk{k}})\mathds{1}(k \in \mah) + (\twonorm{\bbetak{k}}^2 \vee \twonorm{\bbetak{k}})\mathds{1}(k \in \mach)\right], \\
				g_1^{(k)}(\zeta) &= g_2^{(k)}(\zeta) = \exp(-\zeta^2), k \neq 0; \quad 
				g_1^{(0)}(\zeta) = \exp(-\zeta^2), g_2^{(0)}(\zeta) = \exp\{-n_0\} + \exp(-\zeta^2).
			\end{align}
		\item For the Poisson regression model with bounded predictors ($\sup_k \infnorm{\bxk{k}} \leq U < \infty$): 
			\begin{align}
				\Gamma_1^{(0)} &= \sqrt{\frac{s\log p}{n_0}}\cdot \exp\left(U\onenorm{\bbeta}\right), \quad \Gamma_2^{(0)} = \exp\left(U\onenorm{\bbeta}\right)\cdot \frac{1 + \twonorm{\bbeta} + U \onenorm{\bbeta}}{\sqrt{n_0}},\\
				\Upsilon_1^{(k)} &= \Omega_k\cdot \exp\left\{U\onenorm{\bwk{k}}\cdot \mathds{1}(k \in \mah) + U\onenorm{\bbetak{k}}\cdot \mathds{1}(k \in \mach)\right\},\\
				\Gamma_1^{(k)} &= \sqrt{\frac{1}{n_0}}\Omega_k\cdot \exp\left\{U\onenorm{\bwk{k}}\cdot \mathds{1}(k \in \mah) + U\onenorm{\bbetak{k}}\cdot \mathds{1}(k \in \mach)\right\},\\
				\Gamma_2^{(k)} &= \sqrt{\frac{1}{n_0}}\left[\exp\left(U\onenorm{\bwk{k}}\right) \left(1 + \twonorm{\bwk{k}} + U \onenorm{\bwk{k}}\right)\cdot \mathds{1}(k \in \mah)  \right.\\
				&\quad \left.+ \exp\left(U\onenorm{\bbetak{k}}\right) \left(1 + \twonorm{\bbetak{k}} + U \onenorm{\bbetak{k}}\right)\cdot \mathds{1}(k \in \mach)\right], \\
				g_1^{(k)}(\zeta) &= g_2^{(k)}(\zeta) = \exp(-\zeta^2), k \neq 0; \quad g_1^{(0)}(\zeta)  = \exp(-\zeta^2), g_2^{(0)}(\zeta) = \zeta^{-2}.
			\end{align}
	\end{enumerate}
\end{proposition}

\subsubsection{Bound of a prediction error measure}
Denote $L^{(0)}_{n_0}(\bw) = -\frac{1}{n_0}\sum_{i=1}^{n_0}\log \rho(\bxk{0}_i)-\frac{1}{n_0}(\by^{(0)})^T\bX^{(0)}\bw + \frac{1}{n_0}\sum\limits_{i=1}^{n_0}\psi(\bw^T\bx^{(0)}_i)$. Suggested by \cite{loh2015regularized}, we consider a special measure of the prediction error, which is defined by $\<\nabla L^{(0)}_{n_0}(\hbeta) -\nabla L^{(0)}_{n_0}(\bbeta), \hbeta - \bbeta\>$, where $\nabla L^{(0)}_{n_0}(\bw) = -\frac{1}{n_0}(\bX^{(0)})^T\by^{(0)} + \frac{1}{n_0}\sum\limits_{i=1}^{n_0}\bx^{(0)}_i\psi(\bw^T\bx^{(0)}_i) \in \mathbb{R}^p$. Note that the loss function $L^{(0)}_{n_0}$ is convex, therefore this quantity is non-negative. As argued in their paper, $\<\nabla L^{(0)}_{n_0}(\hbeta) -\nabla L^{(0)}_{n_0}(\bbeta), \hbeta - \bbeta\>$ can be easily interpreted in GLMs. For example, in the case of linear models where $\psi(u) = u^2/2$, this measure equals to the in-sample square loss $\twonorm{\bX^{(0)}(\hbeta - \bbeta)}^2$. For general GLMs, it is equivalent to the symmetrized Bregman divergence defined by $\psi$. 

Next we would like to present the bounds of $\<\nabla L^{(0)}_{n_0}(\hbeta) -\nabla L^{(0)}_{n_0}(\bbeta), \hbeta - \bbeta\>$ with and without Assumption \ref{asmp: merging} for $\mah$-Trans-GLM.

\begin{theorem}[Bound of a prediction error measure for $\mah$-Trans-GLM with Assumption \ref{asmp: merging}]\label{thm: prediction error with assumption}
	By imposing the same conditions in Theorem \ref{thm: l2 with assumption}, we have		
	\begin{align}
			\sup_{\bxi \in \Xi(s, h)}\<\nabla L^{(0)}_{n_0}(\hbeta) -\nabla L^{(0)}_{n_0}(\bbeta), \hbeta - \bbeta\> 
			&\lesssim s\left(\frac{\log p}{n_0}\right)^{1/2}\left(\frac{\log p}{\nah+n_0}\right)^{1/2} + h\left(\frac{\log p}{n_0}\right)^{1/2} \\
			&\quad + \left(\frac{s\log p}{n_0}\right)^{1/2}\left(\frac{\log p}{\nah+n_0}\right)^{1/4}h^{1/2}, \label{eq: prediction error bound 1 with asmp}
	\end{align}
		with probability at least $1-n_0^{-1}$.
\end{theorem}

\begin{remark}
	When $h \ll s\sqrt{\frac{\log p}{n_0}}$ and $\nah \gg n_0$, the upper bounds in (\rom{1})  are better than the classical Lasso bound $\mathcal{O}_p\left(\frac{s\log p}{n_0}\right)$ with target data \citep{loh2015regularized}. 
\end{remark}

\begin{theorem}[Bound of a prediction error measure for $\mah$-Trans-GLM without Assumption \ref{asmp: merging}]\label{thm: prediction error without assumption}
By imposing the same conditions in Theorem \ref{thm: l2 without assumption}, we have	\begin{align}
		\sup_{\bxi \in \Xi(s, h)}\<\nabla L^{(0)}_{n_0}(\hbeta) -\nabla L^{(0)}_{n_0}(\bbeta), \hbeta - \bbeta\> 
		&\lesssim s\left(\frac{\log p}{n_0}\right)^{1/2}\left(\frac{\log p}{\nah+n_0}\right)^{1/2} + sh\left(\frac{\log p}{n_0}\right)^{1/2}\\
		&\quad + \left(\frac{s\log p}{n_0}\right)^{1/2}\left(\frac{\log p}{\nah+n_0}\right)^{1/4}h^{1/2} \label{eq: prediction error bound 1 with asmp}
	\end{align}
	with probability at least $1-n_0^{-1}$.
\end{theorem}

\begin{remark}
	When $h \ll \sqrt{\frac{\log p}{n_0}}$ and $\nah \gg n_0$, the upper bounds in (\rom{1})  are better than the classical Lasso bound $\mathcal{O}_p\left(\frac{s\log p}{n_0}\right)$ with target data \citep{loh2015regularized}. 
\end{remark}

\subsection{Additional numerical results}

\subsubsection{More details about the implementation of numerical experiments}\label{subsubsec: implementation supp}
All experiments in this paper are conducted in R. The GLM Lasso is implemented via R package \texttt{glmnet} \citep{friedman2010regularization}. We summarize R codes for GLM transfer learning algorithms in a new R package \texttt{glmtrans}, which is available on CRAN. We use 10-fold cross-validation to choose the penalty parameter for na\"ive-Lasso and our GLM transfer learning algorithms. The largest $\lambda$ which achieves one standard error within the minimum cross-validation error will be chosen for the transferring step, which is sometimes called \textit{``lambda.1se''} \citep{friedman2010regularization}. To effectively debias the transferring step, we choose the lambda achieving minimal cross-validation error, which is often called \textit{``lambda.min''}. Since in transferable source detection, the first step is the same as the transferring step of $\{k\}$-Trans-GLM, therefore we keep the same setting as the transferring step in Algorithm \ref{algo: merging}, i.e. take \textit{``lambda.1se''}. And in Algorithm \ref{algo: a unknown}, we set the constant $C_0 = 2$. In the two-step transfer learning procedure of Algorithm \ref{algo: inf}, we use ``lambda.min'' in both transferring and debiasing steps.

In real-data studies, SVM with RBF kernel is implemented by package \texttt{e1071}, and decision trees are implemented through package \texttt{rpart}. We fit the random forest via package \texttt{randomForest}, and implement boosting trees through package \texttt{fastAdaboost}. The number of weak classifiers in boosting trees is set to be 50. Since the sample size of each state is small and some states have very imbalanced responses, we change the cross-validation folds from default 10 to 5 for all Lasso-based methods. All the other parameters are kept the same as the default settings.

\subsubsection{Transfer learning on $\mah$}\label{subsubsec: oracle supp}
In this section, we supplement more numerical results about the performance of Algorithm \ref{algo: merging} under different $h$ and $(\{n_k\}_{k=0}^K, p, s)$ settings. In addition to the previous $(\{n_k\}_{k=0}^K, p, s)$ setting studied in Section \ref{subsubsec: oracle}, the following two settings of $(\{n_k\}_{k=0}^K, p, s)$ are considered:
\begin{enumerate}[(i)]
	\item $n_k = 150$ for all $k = 0, \ldots, K$, $p = 1000$, $s =15$;
	\item $n_k = 200$ for all $k = 0, \ldots, K$, $p = 2000$, $s =20$.
\end{enumerate}

\begin{figure}[!h]
	\centering
	\includegraphics[width=\textwidth]{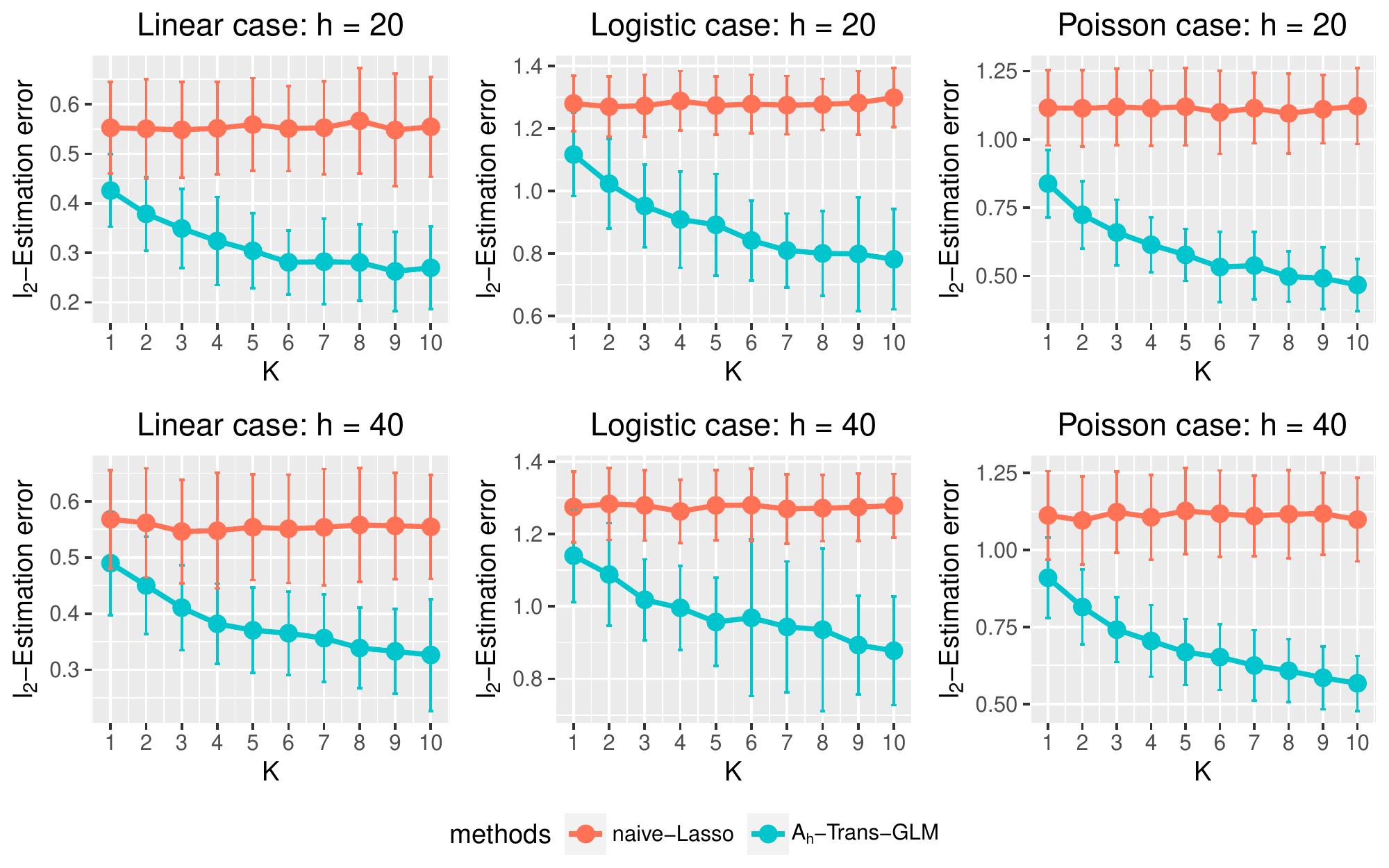}
	\caption{The average $\ell_2$-estimation of $\mah$-Trans-GLM and na\"ive-Lasso under linear, logistic and Poisson regression models with different settings of $h$ and $K$. $n_k = 150$ for all $k = 0, \ldots, K$, $p = 1000$, $s =15$. Error bars denote the standard deviations.}
	\label{fig: oracle2}
\end{figure}

\begin{figure}[!h]
	\centering
	\includegraphics[width=\textwidth]{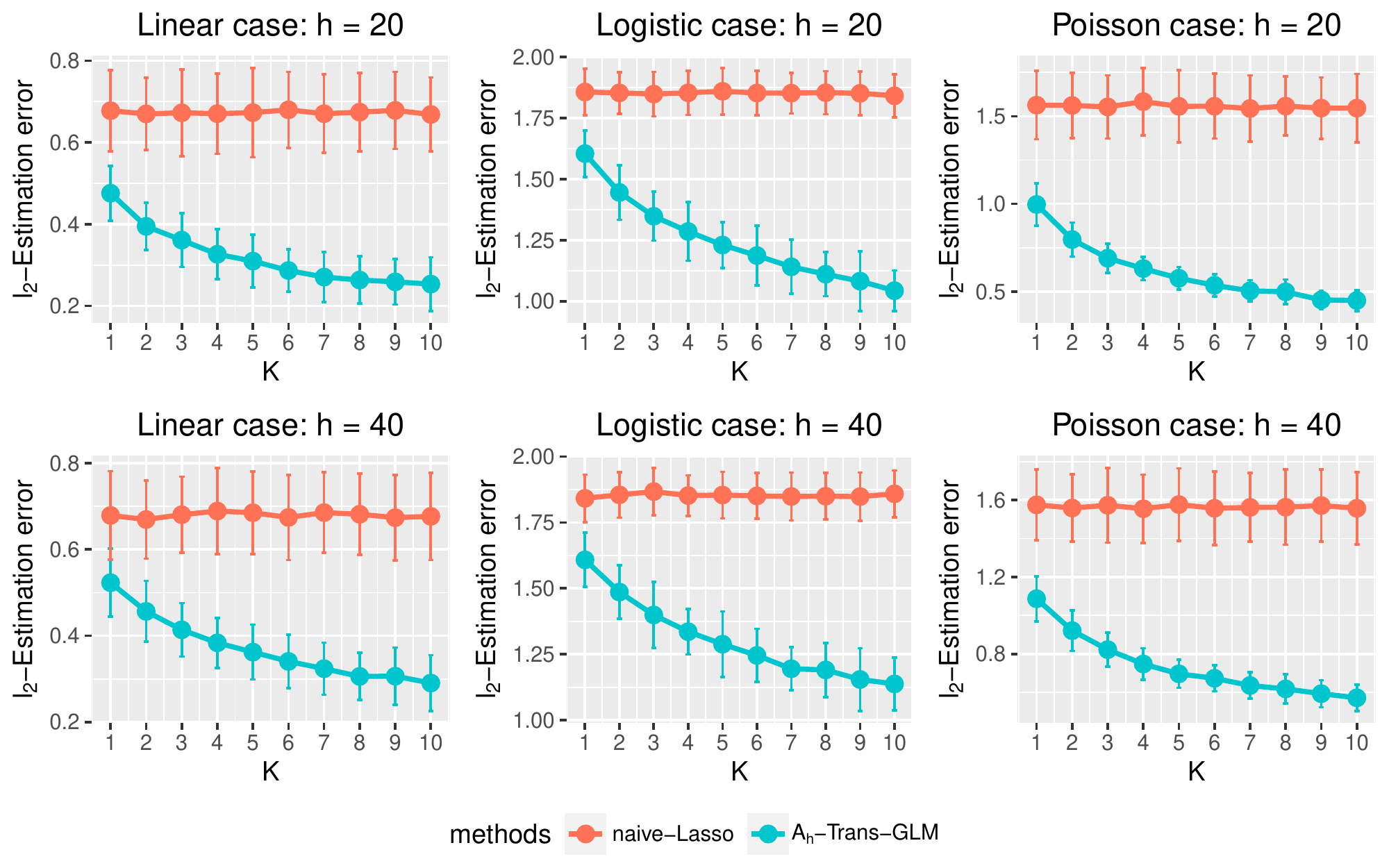}
	\caption{The average $\ell_2$-estimation of $\mah$-Trans-GLM and na\"ive-Lasso under linear, logistic, and Poisson regression models with different settings of $h$ and $K$. $n_k = 200$ for all $k = 0, \ldots, K$, $p = 2000$, $s =20$. Error bars denote the standard deviations.}
	\label{fig: oracle3}
\end{figure}

Given each $(\{n_k\}_{k=0}^K, p, s)$ setting, consider the same setting we use in Section \ref{subsubsec: oracle}. All the experiments are replicated 200 times and the average $\ell_2$-estimation errors of $\mah$-Trans-GLM and na\"ive-Lasso under linear, logistic and Poisson regression models are shown in Figure \ref{fig: oracle2} and \ref{fig: oracle3}.

The trends in Figures \ref{fig: oracle2} and \ref{fig: oracle3} are similar to that in Figure \ref{fig: oracle1}. As $K$ increases, the estimation error of $\Kah$-Trans-GLM continues declining and is lower that of na\"ive-Lasso on target data only.

\subsubsection{Transferable source detection}\label{subsubsec:source detection}

\begin{figure}[!h]
	\centering
	\includegraphics[width=\textwidth]{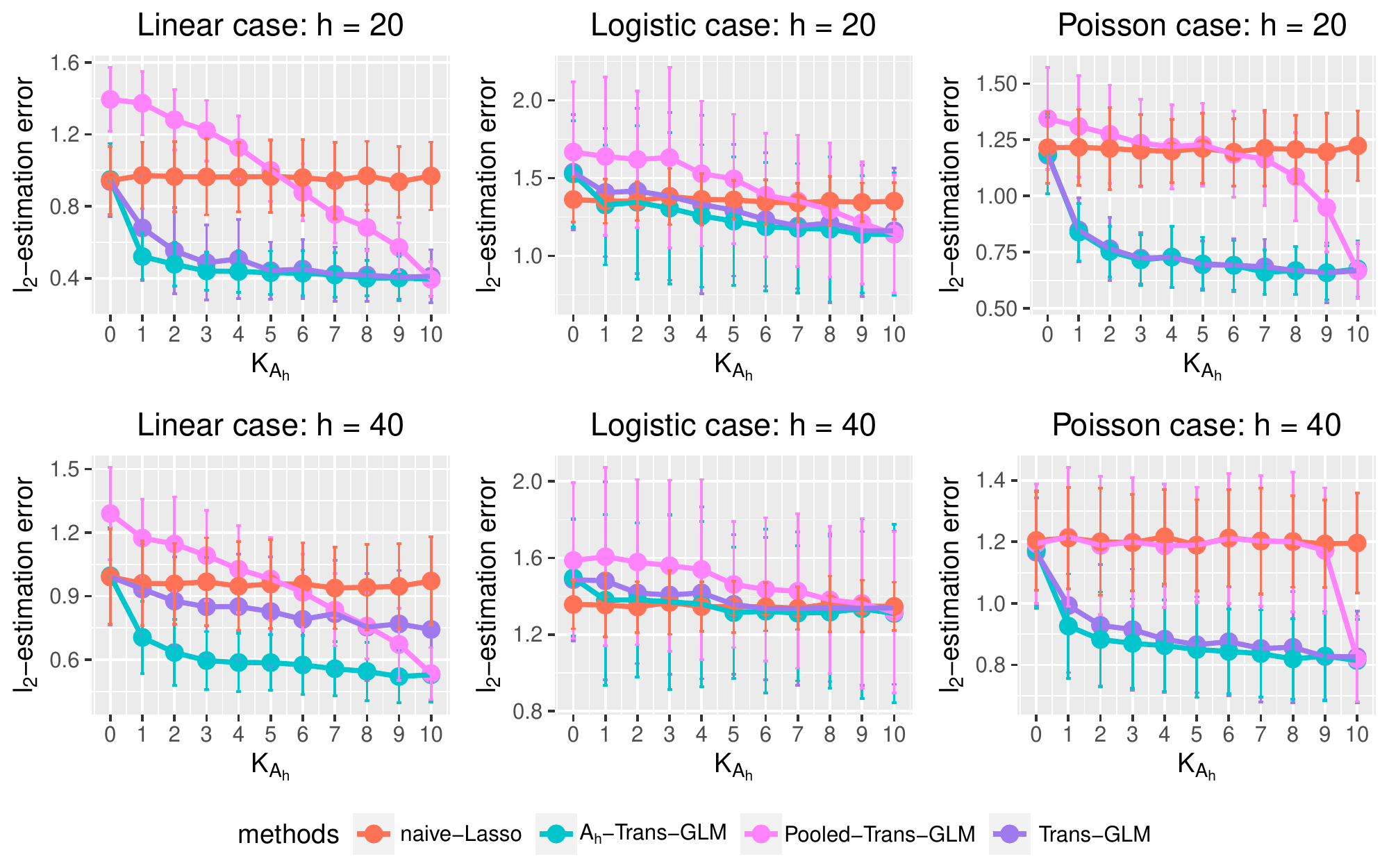}
	\caption{The average $\ell_2$-estimation error of various models with different settings of $h$ and $\Kah$ when $K = 10$. $n_k = 100$ for all $k = 0, \ldots, K$, $p = 500$, $s =10$. Error bars denote the standard deviations.}
	\label{fig: unknown1}
\end{figure} 

\begin{figure}[!h]
	\centering
	\includegraphics[width=\textwidth]{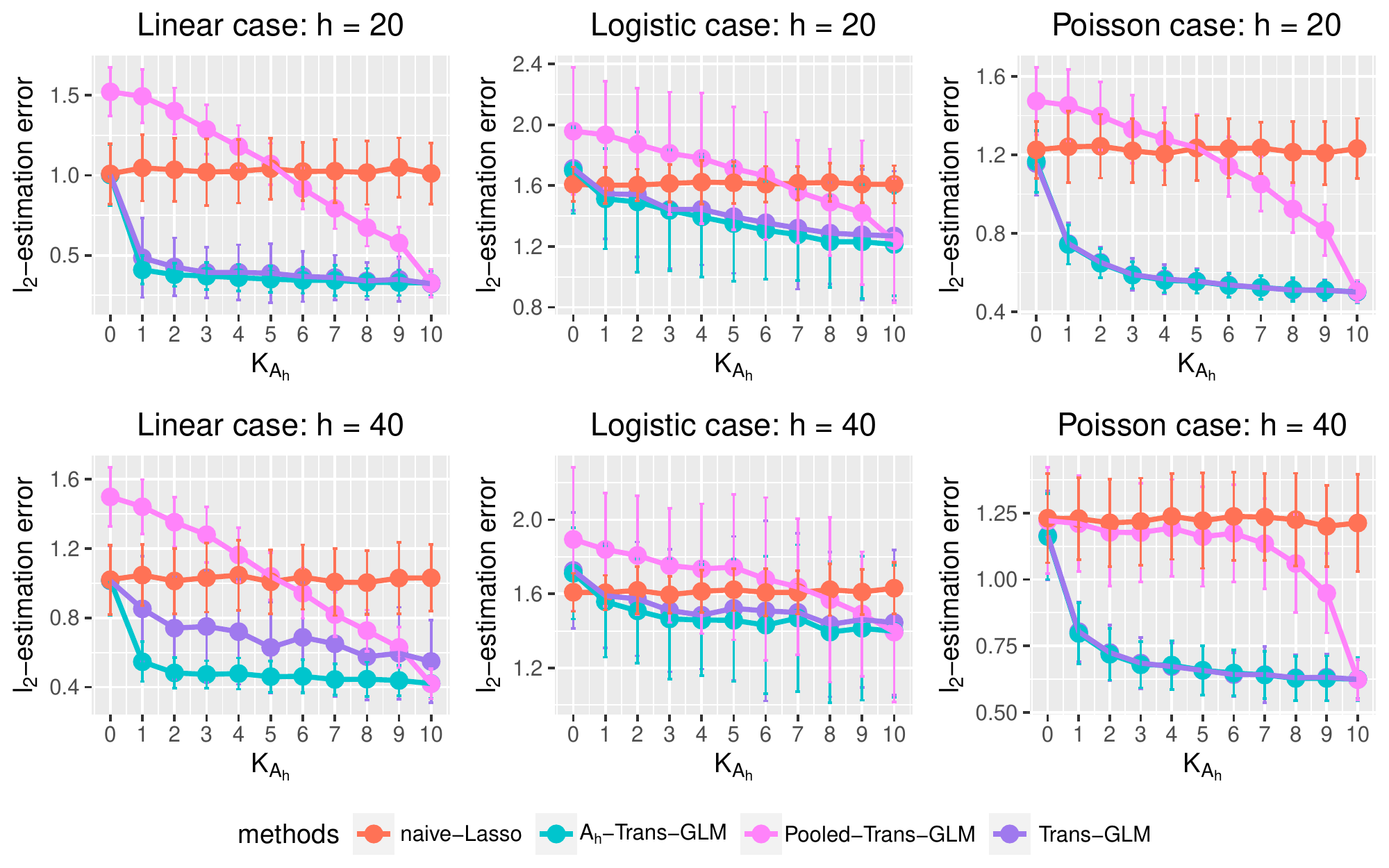}
	\caption{The average $\ell_2$-estimation error of various models with different settings of $h$ and $\Kah$ when $K = 10$. $n_k = 150$ for all $k = 0, \ldots, K$, $p = 1000$, $s =15$. Error bars denote the standard deviations.}
	\label{fig: unknown2}
\end{figure} 

In this section we supplement more experimental results in the case that some sources are not in the level-$h$ transferring set $\mah$. The model settings are the same as those in Section \ref{subsubsec: unknown simulation}. In addition to the setting used in Section \ref{subsubsec: unknown simulation}, two more ones are considered: 
\begin{enumerate}[(i)]
	\item $n_k = 100$ for all $k = 0, \ldots, K$, $p = 500$, $s = 10$;
	\item $n_k = 150$ for all $k = 0, \ldots, K$, $p = 1000$, $s = 15$.
\end{enumerate}

We vary the values of $|\Kah|$ and $h$, and repeat each setting for 200 times. The average $\ell_2$-estimation errors are summarized in Figures \ref{fig: unknown1} and \ref{fig: unknown2}.

Similar to Figure \ref{fig: unknown3}, it can be seen that $\mah$-Trans-GLM always achieves the best performance. Trans-GLM mimics the behavior of $\mah$-Trans-GLM very well in most cases, implying that the detection algorithm can accurately identify $\ma$. We also observe that for linear models and logistic regression models, when $h = 40$, there is a gap between the estimation error of $\mah$-Trans-GLM and Trans-GLM, meaning that when $h$ increases, Trans-GLM might begin missing sources in $\mah$ or wrongly including sources in $\mah$.

\begin{figure}[!h]
	\centering
	\includegraphics[width=\textwidth]{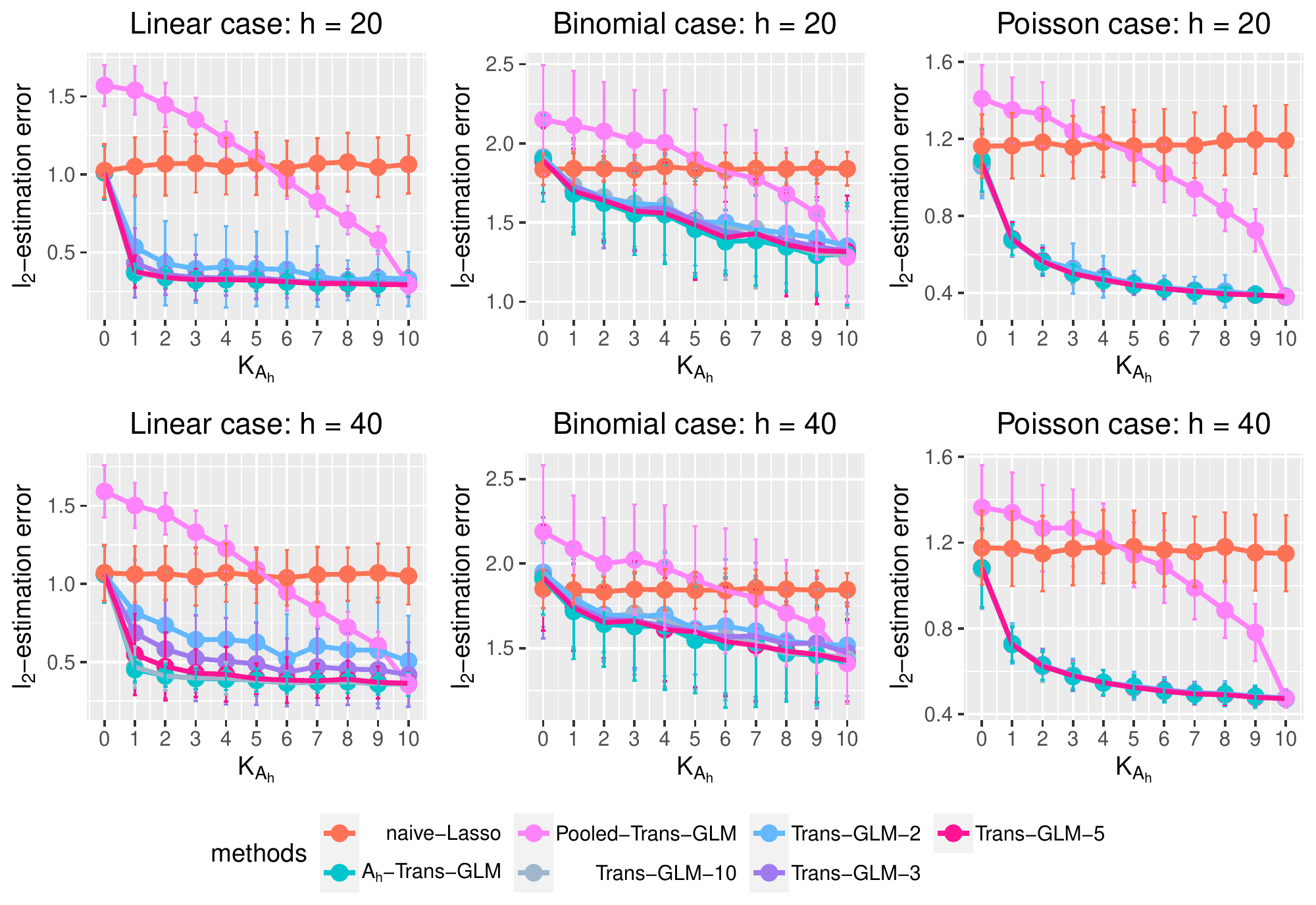}
	\caption{The average $\ell_2$-estimation error of various models with different settings of $h$ and $\Kah$ when $K = 10$. $n_k = 200$ for all $k = 0, \ldots, K$, $p = 2000$, $s =20$. Error bars denote the standard deviations. The numbers after Trans-GLM indicate the number of folds used in cross-validation procedure of Algorithm \ref{algo: a unknown} (steps 2-5)}
	\label{fig: unknown3_cv}
\end{figure} 

Furthermore, we pick the setting when $p=2000$ and try different numbers of folds in the cross-validation procedure of Algorithm \ref{algo: a unknown} (steps 2-5). The results are displayed in Figure \ref{fig: unknown3_cv}. The findings suggest that more cross-validation folds may lead to better performance of Trans-GLM. When the cross-validation folds are large, the detection is likely to be more accurate. In the meantime, this may cause more computational burdens. Therefore, we may choose a moderate fold number like 3 or 5 in practice to achieve a good trade-off between the accuracy and computational costs.

\subsubsection{Additional results of real data analysis}\label{subsubsec: election inf supp}

In this section, we aim to identify variables with significant effects for different targets in the real-data study (Section \ref{subsec:real data}), by applying Algorithm \ref{algo: inf}. Taking the randomness caused by the cross-validation procedure in algorithms, we repeat the experiment 500 times. In each replication, for each target state, we first run Algorithm \ref{algo: a unknown} to get a point estimate $\hbeta$ of the coefficient and the estimated informative source index set $\widehat{\ma}$. Then, we run Algorithm \ref{algo: inf} with $\hbeta$ and $\widehat{\ma}$ to get the significant variables under 90\% confidence level. Equivalently speaking, we identify coefficient components whose 90\% confidence interval (CI) does not cover zero and divided them into two parts based on the sign of CI center. Recall our recoding rule of the response: 0 denotes Republicans and 1 denotes Democrats. Therefore, given all other variables fixed, increasing the variable with a positive CI center gives rise to the chance of a county to vote Democrats. In contrast, increasing the variable with a negative CI center gives rise to the chance of a county to vote Republicans. The results are summarized in Figure \ref{fig: election_inf}, where we list all variables which are significant under 90\% level in at least 5\% of 500 replications. We use different colors and shapes to distinguish the effect of these variables (the sign of their CI centers). The $x$-axis and $y$-axis display the target state \footnote{Target state AR is removed because none of the variables are significant in 5\% of 500 replications} and the variable names. The description of some main effects is presented in Table \ref{tab: variable info election data}. It reveals that RHI825214 is positively significant in 6 of 8 target states, which means that when other variables are fixed, a larger White percentage in a county leads to a higher chance to vote Republicans. In opposition to this, EDU685213 is negatively significant in 4 of 8 target states, showing that when other variables are fixed, a larger Bachelor or higher degree holder percentage benefits Democrats under county-level. More interesting findings can be obtained from Figure \ref{fig: election_inf}, which are expected to provide some insights to better understand the election results.

\begin{figure}[!]
	\centering
	\includegraphics[width=\textwidth]{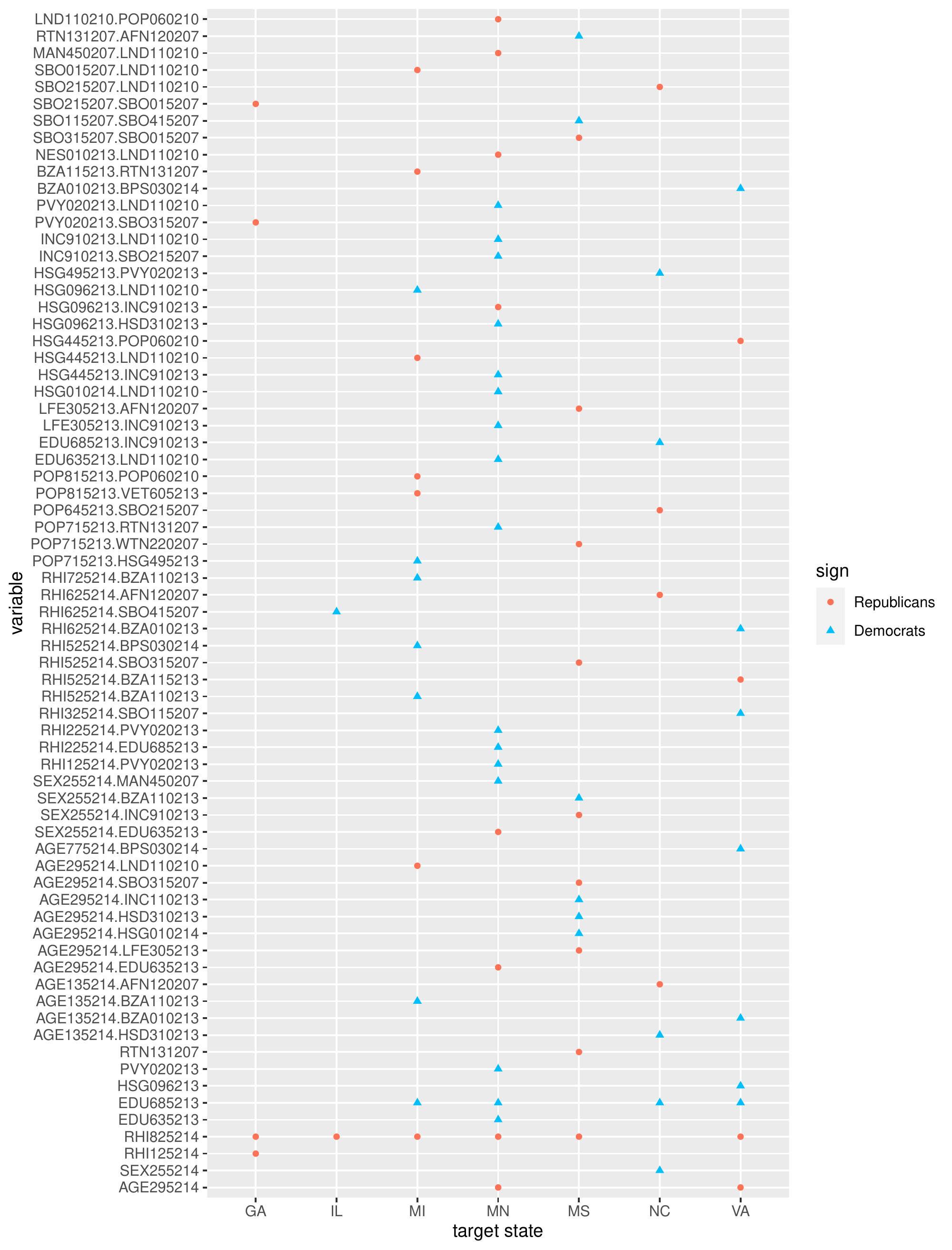}
	\caption{Variables that are significant under 90\% level confidence for different targets in at least 5\% of 500 replications. Provided by Algorithms \ref{algo: a unknown} and  \ref{algo: inf}.}
	\label{fig: election_inf}
\end{figure}

\begin{table}[!]
\renewcommand{\arraystretch}{0.65}
\begin{center}
\begin{tabular}{l|l}
\Xhline{1.25pt}
Variable name  & Description \\ 
\hline
AGE135214 &  Persons under 5 years, percent, 2014 \\ 
AGE295214 &  Persons under 18 years, percent, 2014 \\ 
AGE775214 &  Persons 65 years and over, percent, 2014 \\ 
SEX255214 &  Female persons, percent, 2014 \\ 
RHI125214 &  White alone, percent, 2014 \\ 
RHI225214 &  Black or African American alone, percent, 2014 \\ 
RHI325214 &  American Indian and Alaska Native alone, percent, 2014 \\ 
RHI525214 &  Native Hawaiian and Other Pacific Islander alone, percent, 2014 \\ 
RHI625214 &  Two or More Races, percent, 2014 \\ 
RHI725214 &  Hispanic or Latino, percent, 2014 \\ 
RHI825214 &  White alone, not Hispanic or Latino, percent, 2014 \\ 
POP715213 &  Living in same house 1 year \& over, percent, 2009-2013 \\ 
POP645213 &  Foreign born persons, percent, 2009-2013 \\ 
POP815213 &  Language other than English spoken at home, pct age 5+, 2009-2013 \\ 
EDU635213 &  High school graduate or higher, percent of persons age 25+, 2009-2013 \\ 
EDU685213 &  Bachelor's degree or higher, percent of persons age 25+, 2009-2013 \\ 
VET605213 &  Veterans, 2009-2013 \\ 
LFE305213 &  Mean travel time to work (minutes), workers age 16+, 2009-2013 \\ 
HSG010214 &  Housing units, 2014 \\ 
HSG445213 &  Homeownership rate, 2009-2013 \\ 
HSG096213 &  Housing units in multi-unit structures, percent, 2009-2013 \\ 
HSG495213 &  Median value of owner-occupied housing units, 2009-2013 \\ 
HSD310213 &  Persons per household, 2009-2013 \\ 
INC910213 &  Per capita money income in past 12 months (2013 dollars), 2009-2013 \\ 
INC110213 &  Median household income, 2009-2013 \\ 
PVY020213 &  Persons below poverty level, percent, 2009-2013 \\ 
BZA010213 &  Private nonfarm establishments, 2013 \\ 
BZA110213 &  Private nonfarm employment,  2013 \\ 
BZA115213 &  Private nonfarm employment, percent change, 2012-2013 \\ 
NES010213 &  Nonemployer establishments, 2013 \\ 
SBO315207 &  Black-owned firms, percent, 2007 \\ 
SBO115207 &  American Indian- and Alaska Native-owned firms, percent, 2007 \\ 
SBO215207 &  Asian-owned firms, percent, 2007 \\ 
SBO415207 &  Hispanic-owned firms, percent, 2007 \\ 
SBO015207 &  Women-owned firms, percent, 2007 \\ 
MAN450207 &  Manufacturers shipments, 2007 (\$1,000) \\ 
WTN220207 &  Merchant wholesaler sales, 2007 (\$1,000) \\ 
RTN131207 &  Retail sales per capita, 2007 \\ 
AFN120207 &  Accommodation and food services sales, 2007 (\$1,000) \\ 
BPS030214 &  Building permits, 2014 \\ 
LND110210 &  Land area in square miles, 2010 \\ 
POP060210 &  Population per square mile, 2010 \\ 
\Xhline{1.25pt}
\end{tabular}
\end{center}
\caption{Description of some variables in the original dataset.}
\label{tab: variable info election data}
\end{table}

\section{Proofs}\label{sec:proofs}
Define $\huah = \hwah - \bwah$ and $\mathcal{D} = \{(\bXk{k}, \byk{k})\}_{k \in \transet}$. In the following, we will use bolded $\bpsi'$ to represent the vector whose each component comes from the scalar function $\psi'$ with corresponding predictors. Denote 
\begin{align}
	\hat{L}(\bw, \md) &= -\frac{1}{\nah + n_0}\sum_{k \in \transet}(\byk{k})^T \bXk{k}\bw + \frac{1}{\nah + n_0}\sum_{k \in \transet}\sum_{i=1}^{n_k}\psi(\bw^T \bxk{k}_i), \\
	\nabla \hat{L}(\bw, \md) &= -\frac{1}{\nah + n_0}\sum_{k \in \transet}(\bXk{k})^T\byk{k} + \frac{1}{\nah + n_0}\sum_{k \in \transet}(\bXk{k})^T \bpsi'(\bw^T \bxk{k}), \\
	\delta \hat{L}(\bu, \md) &= \hat{L}(\bwah + \bu, \md) - \hat{L}(\bwah) - \nabla \hat{L}(\bwah)^T \bu.
\end{align}
Denote $\partial \onenorm{\bw}$ as the subgradient of $\onenorm{\bw}$ w.r.t. $\bw \in \mathbb{R}^p$, which falls between $-1$ and $1$.

For any $\bw \in \mathbb{R}^{n}$, denote $\bWk{k}_{\bw} = \diag\left(\sqrt{\psi''((\bxk{k}_1)^T\bw)}, \ldots, \sqrt{\psi''((\bxk{k}_{n_k})^T\bw)}\right)$ and $\bXk{k}_{\bw} = \bWk{k}_{\bw}\bXk{k}$. $\bXk{k}_{\bw, j}$ represents the $j$-th column of $\bXk{k}_{\bw}$ and $\bXk{k}_{\bw, -j}$ represents the matrix $\bXk{k}_{\bw}$ without the $j$-th column. Denote $\bXk{k}_{\bw, -j}$  as the submatrix without $j$-th column. $\bXk{k}_{\bw, j}$ represents the $j$-th column of $\bXk{k}_{\bw}$ without the diagonal $(j,j)$ elements. Denote $\bxk{k}_{\bw} = \sqrt{\psi''((\bxk{k})^T\bw)}\cdot \bxk{k}$. $\bxk{k}_{\bw, j}$ and $\bxk{k}_{\bw, -j}$ represent the $j$-th component of $\bxk{k}_{\bw}$ and the vector without $j$-th component, respectively. Define
\begin{equation}
	\bgammak{k}_j = \argmin_{\bgamma \in \mathbb{R}^{p-1}} \te\left[\bxk{k}_{\bwk{k}, j} - (\bxk{k}_{\bwk{k}, -j})^T\bgamma \right]^2 = \argmin_{\bgamma \in \mathbb{R}^{p-1}} \te\left\{\psi''((\bwk{k})^T\bxk{k})\cdot [\bxk{k}_{j} - (\bxk{k}_{-j})^T\bgamma]^2 \right\}.
\end{equation}
Also define $\bmetak{k}_{\bwk{k}, j} = \bXk{k}_{\bw, j} - \bXk{k}_{\bw, -j}\bgammak{k}_j$ and $(\tau^{(k)}_j)^2 = \te(\bmetak{k}_{\bwk{k}, j})^2$. And define
\begin{equation}
	\bgammaa_j = \argmin_{\bgamma \in \mathbb{R}^{p-1}} \left\{ \sum_{k \in \transet}\te \left[\psi''((\bwk{k})^T\bxk{k})\cdot [\bxk{k}_{j} - (\bxk{k}_{-j})^T\bgamma]^2\right] \right\}.
\end{equation}
\subsection{Some lemmas}

\begin{lemma}\label{lem: delta a}
	Under Assumptions \ref{asmp: convexity} and \ref{asmp: merging}, 
	\begin{equation}
		\onenorm{\bdeltaah} = \onenorm{\bwah - \bbeta} \leq C_1h,
	\end{equation}
	where $\bwah$ is defined by equation \eqref{eq: wa in population merging} and $C_1 \coloneqq  \sup_{k \in \transet}\onenorm{\widetilde{\bm{\Sigma}}_h^{-1}\widetilde{\bm{\Sigma}}_h^{(k)}} < \infty$.
\end{lemma}

\begin{lemma}\label{lem: re glm}
	Under Assumptions \ref{asmp: convexity} and \ref{asmp: x}, there exists some positive constants $\kappa_1$, $\kappa_2$, $C_3$ and $C_4$ such that,
	\begin{equation}
		\delta \hat{L}(\huah, \md) \geq \kappa_1 \twonorm{\bu}^2 - \kappa_1\kappa_2\sqrt{\frac{\log p}{\nah + n_0}}\onenorm{\bu}\twonorm{\bu}, \quad\forall \bu: \twonorm{\bu} \leq 1
	\end{equation}
	with probability at least $1-C_3\exp\{-C_4(\nah + n_0)\}$.
\end{lemma}


\begin{lemma}\label{lem: gamma a gamma 0 distance}
	Under Assumptions \ref{asmp: convexity}-\ref{asmp: second derivative} and Assumption \ref{asmp: inf}.(\rom{1})-(\rom{3}), 
	\begin{equation}
		\sup_{k,j} \onenorm{\bgammaah_j - \bgamma_j^{(0)}} \lesssim h_1.
	\end{equation}
\end{lemma}

\begin{lemma}\label{lem: inf lemma gamma a}
	Assume Assumptions \ref{asmp: convexity}-\ref{asmp: merging} and \ref{asmp: inf} hold. Let $\lambda_j \asymp \sqrt{\frac{\log p}{\nah+n
	_0}} + (\frac{1}{\sqrt{s}}(\frac{\log p}{n_0})^{1/4}h^{1/2})\wedge \frac{h}{\sqrt{s}}$ for any $j = 1, \ldots, p$. Suppose $h_1 \lesssim s^{-1/2}$. With probability at least $1-\Kah n_0^{-1}$,
	\begin{align}
		\twonorm{\hgammaa_j - \bgammaa_j}^2 &\lesssim h_1\left(\frac{\log p}{\nah+n_0}\right)^{1/2} + \Re_1^2 + \Re_1\frac{h_1}{\sqrt{s}},\\
		\onenorm{\hgammaa_j - \bgammaa_j} &\lesssim \sqrt{s}\Re_1 + s^{1/4}h_1^{1/2}\Re_1^{1/2} + h_1.
	\end{align}
\end{lemma}

\begin{lemma}\label{lem: inf lemma gamma delta}
	Let $\tlambda_j \asymp \sqrt{\frac{\log p}{n_0}} + \Re_1$. Impose the same conditions assumed by Lemma \ref{lem: inf lemma gamma a}. Suppose $h_1 \lesssim s_*^{-1/2}$. Then with probability at least $1-\Kah n_0^{-1}$,
	\begin{align}
		\twonorm{\hrho_j - \brho_j}^2 &\lesssim h_1\sqrt{\frac{\log p}{n_0}} + h_1\Re_1 + \Re_1^2,\\
		\onenorm{\hrho_j - \brho_j} &\lesssim h_1 + \Re_1.
	\end{align}
\end{lemma}

The proof of Lemma \ref{lem: delta a} will be presented later in Section \ref{subsubsec: proof of lemma 1}. Lemma \ref{lem: re glm} can be derived in the same spirit as the proof of Proposition 2 in the full-length version of \cite{negahban2009unified}, so we omit the full proof and only highlight the sketch in Section \ref{subsubsec: proof of lemma 2}. Lemma \ref{lem: gamma a gamma 0 distance} can be proved by following the same idea in the proof of Lemma \ref{lem: delta a}, therefore we omit its proof as well.

Also, it is important to point out that all the constants involved in the proofs of Theorem \ref{thm: l2 with assumption}-\ref{thm: prediction error without assumption} are independent with $\bxi = \{\bbeta, \{\bwk{k}\}_{k \in \ma}\} \in \Xi(s, h)$, therefore we can take the supremum over $\bxi \in \Xi(s, h)$ in the final conclusion without changing the rate.

\subsection{Proof of lemmas}

\subsubsection{Proof of Lemma \ref{lem: delta a}}\label{subsubsec: proof of lemma 1}
By definition, 
\begin{equation}
	\sum\limits_{k\in\{0\}\cup \ma}\alpha_k\te\left\{[\psi'((\bwah)^T\bxk{k})-\psi'((\bwk{k})^T\bxk{k})]\bxk{k}\right\} = 0,
\end{equation}
which implies
\begin{align}
	&\sum\limits_{k\in\{0\}\cup \ma}\alpha_k\te\left\{[\psi'((\bwah)^T\bxk{k})-\psi'(\bbeta^T\bxk{k})]\bxk{k}\right\} \\
	&= \sum\limits_{k \in \ma}\alpha_k\te\left\{[\psi'((\bwk{k})^T\bxk{k})-\psi'(\bbeta^T\bxk{k})]\bxk{k}\right\}.
\end{align}
By Taylor expansion,
\begin{align}
	&\sum\limits_{k\in\{0\}\cup \ma}\alpha_k\te\left[\int_0^1\psi''((\bwah)^T\bxk{k} + t(\bwah - \bbeta)^T\bxk{k})\bxk{k}(\bxk{k})^T\right](\bwah - \bbeta) \\
	&= \sum\limits_{k \in \ma}\alpha_k\te\left[\int_0^1\psi''((\bwk{k})^T\bxk{k} + t(\bwk{k} - \bbeta)^T\bxk{k})\bxk{k}(\bxk{k})^T\right](\bwk{k} - \bbeta).
\end{align}
Therefore, by Assumption \ref{asmp: merging}, $\onenorm{\bwah - \bbeta} \leq \sum\limits_{k \in \ma}\alpha_k \onenorm{\widetilde{\bm{\Sigma}}^{-1}_h\widetilde{\bm{\Sigma}}_h^{(k)}} \cdot \onenorm{\bwk{k} - \bbeta} \leq C_1h$.

\subsubsection{Proof of Lemma \ref{lem: re glm}}\label{subsubsec: proof of lemma 2}
By the second-order Taylor expansion, for some $t_{i}^{(k)} \in [0, 1]$,
\begin{align}
	\delta \hat{L}(\bu, \md) &= \hat{L}(\bwah + \bu, \md) - \hat{L}(\bwah) - \nabla \hat{L}^{(k)}(\bwah)^T \bu \\
	&= \frac{1}{\nah + n_0}\sum_{k \in \transet}\sum_{i=1}^{n_k}\left[\psi((\bwah + \bu)^T \bxk{k}_i) - \psi((\bwah)^T \bxk{k}_i) - \psi'((\bwah)^T \bxk{k}_i)\bu^T\bxk{k}_i \right] \\
	&= \frac{1}{\nah + n_0}\sum_{k \in \transet}\sum_{i=1}^{n_k} \psi''((\bwah)^T \bxk{k}_i+t_{i}^{(k)}\bu^T\bxk{k}_i)(\bu^T\bxk{k}_i)^2,
\end{align}
which is the counterpart of equation (63) in the full-length version of \cite{negahban2009unified}. Due to the independence of between $\bxk{k}_i$ for any $i$ and $k$, the arguments in \cite{negahban2009unified} directly follow.

\subsubsection{Proof of Lemma \ref{lem: inf lemma gamma a}}
Recall Theorem \ref{thm: l2 with assumption}, under the assumptions, we have
\begin{align}
	\twonorm{\hbeta-\bbeta} &\lesssim \sqrt{\frac{s\log p}{\na+n_0}} + \left[\left(\frac{\log p}{n_0}\right)^{1/4}h^{1/2}\right] \wedge h, \\
	\onenorm{\hbeta-\bbeta} &\lesssim s\sqrt{\frac{\log p}{\na+n_0}} + h,
\end{align}
with probability at least $1-n_0^{-1}$.
By basic inequality,
\begin{align}
	&\frac{1}{2(\nah+n_0)}\sum_{k \in \transet} \twonorm{\bXk{k}_{\hbeta, j}-\bXk{k}_{\hbeta, -j}\hgammaah_j}^2 + \lambda_j\onenorm{\hgammaah_j} \\
	&\leq \frac{1}{2(\nah+n_0)}\sum_{k \in \transet} \twonorm{\bXk{k}_{\hbeta, j}-\bXk{k}_{\hbeta, -j}\bgammaah_j}^2 + \lambda_j\onenorm{\bgammaah_j},
\end{align}
which implies
\begin{align}
	&\frac{1}{2(\nah+n_0)}\sum_{k \in \transet} \twonorm{\bXk{k}_{\hbeta, -j}(\hgammaah_j-\bgammaah_j)}^2 \\
	&\leq \frac{1}{\nah+n_0}\sum_{k \in \transet}\<\bWk{k}_{\hbeta}(\bWk{k}_{\bbeta})^{-1}\bmetak{k}_{\bbeta, j}, \bXk{k}_{\hbeta, -j}(\bgammaah_j - \hgammaah_j)\> \\
	&\quad + \frac{1}{\nah+n_0}\sum_{k \in \transet}\<\bXk{k}_{\hbeta, -j}(\bgammak{k}_j - \bgammaah_j), \bXk{k}_{\hbeta, -j}(\bgammaah_j - \hgammaah_j)\> \\
	&\quad + \lambda_j(\onenorm{\bgammaah_j} - \onenorm{\hgammaah_j}). \label{eq: basic ineq for gamma a}
\end{align}
Denote $\bm{\Lambda}^{(k)} = \diag(\{\psi''((\bxk{k}_i)^T\bbeta)\}_{i=1}^{n_k})$ and $\widehat{\bm{\Lambda}}^{(k)} = \diag(\{\psi''((\bxk{k}_i)^T\hbeta)\}_{i=1}^{n_k})$. 
\begin{align}
	&\norma{\frac{1}{\nah+n_0}\sum_{k \in \transet}\<\bXk{k}_{\hbeta, -j}(\bgammak{k}_j - \bgammaah_j), \bXk{k}_{\hbeta, -j}(\bgammaah_j - \hgammaah_j)} \\
	&\leq \frac{1}{\nah+n_0}\left[\infnorma{\sum_{k \in \transet} (\bXk{k}_{-j})^T\bm{\Lambda}^{(k)}\bXk{k}_{-j}(\bgammak{k}_j - \bgammaah_j)} \right. \\
	&\quad \left. + \infnorma{\sum_{k \in \transet} (\bXk{k}_{-j})^T(\widehat{\bm{\Lambda}}^{(k)} - \bm{\Lambda}^{(k)})\bXk{k}_{-j}(\bgammak{k}_j - \bgammaah_j)}\right]\cdot \onenorm{\bgammaah_j - \hgammaah_j}. \label{eq: gamma a rhs}
\end{align}
It's easy to see that each component of $\sum_{k \in \transet} (\bXk{k}_{-j})^T\bm{\Lambda}^{(k)}\bXk{k}_{-j}(\bgammak{k}_j - \bgammaah_j)$ is a zero-mean, sub-exponential variable with variance $C(\na+n_0)h^2$. By union bounds,
\begin{equation}
	\frac{1}{\nah+n_0}\infnorma{\sum_{k \in \transet} (\bXk{k}_{-j})^T\bm{\Lambda}^{(k)}\bXk{k}_{-j}(\bgammak{k}_j - \bgammaah_j)} \lesssim \sqrt{\frac{\log p}{\na+n_0}}h_1,
\end{equation}
with probability at least $1-p^{-1}$. In addition,
\begin{align}
	&\frac{1}{\nah+n_0}\infnorma{\sum_{k \in \transet} (\bXk{k}_{-j})^T(\widehat{\bm{\Lambda}}^{(k)} - \bm{\Lambda}^{(k)})\bXk{k}_{-j}(\bgammak{k}_j - \bgammaah_j)} \\ 
	&\leq \frac{1}{\nah+n_0} \maxnorma{\sum_{k \in \transet} (\bXk{k}_{-j})^T(\widehat{\bm{\Lambda}}^{(k)} - \bm{\Lambda}^{(k)})\bXk{k}_{-j}}\cdot \onenorm{\bgammak{k}_j - \bgammaah_j}.
\end{align}
For $j_1$, $j_2 \neq j$, 
\begin{align}
	&\frac{1}{\nah+n_0}\sum_{k \in \transet}\sum_{i=1}^{n_k}x_{i, j_1}^{(k)}x_{i, j_2}^{(k)}[\psi''((\bxk{k}_i)^T\hbeta) - \psi''((\bxk{k}_i)^T\bbeta)] \\
	&\leq \sqrt{\frac{1}{\nah+n_0}\sum_{k \in \transet}\sum_{i=1}^{n_k}(x_{i, j_1}^{(k)}x_{i, j_2}^{(k)})^2}\cdot \sqrt{\frac{1}{\nah+n_0}\sum_{k \in \transet}\sum_{i=1}^{n_k}\psi'''(\bm{a}_i^{(k)})[(\bxk{k}_i)^T(\hbeta - \bbeta)]^2} \\
	&\leq C\twonorm{\hbeta - \bbeta},
\end{align}
where the constant $C$ is the same for different $j_1$ and $j_2$, and $\bm{a}_i^{(k)}$ falls on the line segment between $(\bxk{k}_i)^T\hbeta$ and $(\bxk{k}_i)^T\bbeta$. Therefore, the right-hand side of \eqref{eq: gamma a rhs} can be upper bounded by
\begin{align}
	C\left(\twonorm{\hbeta - \bbeta} + \sqrt{\frac{\log p}{\nah+n_0}}\right)h_1\cdot \onenorm{\bgammaah_j - \hgammaah_j},
\end{align}
with probability at least $1-p^{-1}$. 

On the other hand,
\begin{align}
	&\frac{1}{\nah+n_0}\sum_{k \in \transet}\<\bWk{k}_{\hbeta}(\bWk{k}_{\bbeta})^{-1}\bmetak{k}_{\bbeta, j}, \bXk{k}_{\hbeta, -j}(\bgammaah_j - \hgammaah_j)\> \\
	&\leq \frac{1}{\nah+n_0}\infnorma{\sum_{k \in \transet}\bXk{k}_{\hbeta, -j}\bmetak{k}_{\bbeta, j}}\onenorm{\bgammaah_j - \hgammaah_j} \\
	&\quad + \frac{1}{\nah+n_0}\sum_{k \in \transet}\left[\frac{1}{4c_0}\twonorma{(\bWk{k}_{\hbeta}(\bWk{k}_{\bbeta})^{-1} - \bm{I})\bmetak{k}_{\bbeta, j}}^2 \right. \\
	&\hspace{4.6cm} \left.+ c_0\twonorma{\bXk{k}_{\hbeta, -j}(\bgammaah_j - \hgammaah_j)}^2 \right], \label{eq: wwinverse}
\end{align}
where $c_0$ is a positive constant smaller than $1/4$. Note that
\begin{align}
	&\frac{1}{\nah+n_0}\sum_{k \in \transet}\twonorma{(\bWk{k}_{\hbeta}(\bWk{k}_{\bbeta})^{-1} - \bm{I})\bmetak{k}_{\bbeta, j}}^2 \\
	&\leq \frac{1}{\nah+n_0}\sum_{k \in \transet}\sum_{i=1}^{n_k}\frac{\left(\sqrt{\psi''((\bxk{k}_i)^T\hbeta)} - \sqrt{\psi''((\bxk{k}_i)^T\bbeta)}\right)^2}{\psi''((\bxk{k}_i)^T\bbeta)}\cdot \norm{\eta^{(k)}_i}^2 \\
	&\lesssim \frac{1}{\nah+n_0}\sum_{k \in \transet}\sum_{i=1}^{n_k} [(\bxk{k}_i)^T(\hbeta-\bbeta)]^2, \\
	&\lesssim \sup_{k \in \transet}\twonorm{\hbeta-\bbeta}^2 \\
	&\lesssim \Re_1^2,
\end{align}
with probability at least $1-\Kah n_0^{-1}$, where we use Assumption \ref{asmp: inf}.(\rom{1}) to bound $\psi''((\bxk{k}_i)^T\bbeta)$. The last two inequalities follows because of the same reason as that of \eqref{eq: rsc upper version}. Plugging \eqref{eq: gamma a rhs} and \eqref{eq: wwinverse} and into \eqref{eq: basic ineq for gamma a},
\begin{align}
	&\frac{1}{4(\nah+n_0)}\sum_{k \in \transet} \twonorm{\bXk{k}_{\hbeta, -j}(\hgammaah_j-\bgammaah_j)}^2 \\
	&\leq C\left(\sqrt{\frac{\log p}{\nah+n_0}} + \Re_1\right)h_1 \cdot \onenorm{\hgammaah_j - \bgammaah_j} + \Re_1^2 \\
	&\quad  + \frac{C}{\nah+n_0}\infnorma{\sum_{k \in \transet}\bXk{k}_{\hbeta, -j}\bmetak{k}_{\bbeta, j}}\onenorm{\bgammaah_j - \hgammaah_j} + \lambda_j(\onenorm{\bgammaah_j} - \onenorm{\hgammaah_j})
\end{align}
with probability at least $1-\Kah n_0^{-1}$. Note that $\frac{1}{\nah+n_0}\infnorma{\sum_{k \in \transet}\bXk{k}_{\hbeta, -j}} \lesssim \sqrt{\frac{\log p}{\nah+n_0}}$ with probability at least $1-p^{-1}$. Therefore, if $\lambda_j \geq C'\Re_1/\sqrt{s}$ with some large $C'>0$, then since $h_1 \lesssim s^{-1/2}$, we have $\lambda_j > \frac{2C}{\nah+n_0}\infnorma{\sum_{k \in \transet}\bXk{k}_{\hbeta, -j}} + C\left(\sqrt{\frac{\log p}{\nah+n_0}} + \Re_1\right)h_1$, which implies
\begin{align}
	\frac{1}{4(\nah+n_0)}\sum_{k \in \transet} \twonorm{\bXk{k}_{\hbeta, -j}(\hgammaah_j-\bgammaah_j)}^2 &\leq \frac{3}{2}\lambda_j\onenorm{(\hgammaah_j -\bgammaah_j)_{S_j}} - \frac{1}{2}\lambda_j\onenorm{(\hgammaah_j -\bgammaah_j)_{S_j^c}}\\
	&\quad + \lambda_j\onenorm{(\bgammaah_j)_{S_j^c}} + C\Re_1^2, \label{eq: gamma a final ineq}
\end{align}
with probability at least $1-\Kah n_0^{-1}$. Therefore,
\begin{equation}
	\onenorm{\hgammaah_j -\bgammaah_j} \leq 4\sqrt{s}\twonorm{\hgammaah_j -\bgammaah_j} + C\sqrt{s}\Re_1 + h_1, \label{eq: gamma a l1 norm final}
\end{equation}
with probability at least $1-\Kah n_0^{-1}$.
Thus, we have either $\onenorm{\hgammaah_j -\bgammaah_j} \leq 2C\sqrt{s}\Re_1 + 2h_1 \lesssim \sqrt{s}$ or $\onenorm{\hgammaah_j -\bgammaah_j} \leq 8\sqrt{s}\twonorm{\hgammaah_j -\bgammaah_j}$. By similar arguments to get \eqref{eq: rsc upper version}, we obtain
\begin{align}
	\frac{1}{4(\nah+n_0)}\sum_{k \in \transet} \twonorm{\bXk{k}_{\hbeta, -j}(\hgammaah_j-\bgammaah_j)}^2 &\gtrsim \frac{1}{\nah+n_0}\sum_{k \in \transet} \twonorm{\bXk{k}_{-j}(\hgammaah_j-\bgammaah_j)}^2 \\
	&\gtrsim \twonorm{\hgammaah_j-\bgammaah_j}^2, \label{eq: gamma a final eigen}
\end{align}
with probability at least $1-\Kah n_0^{-1}$. Combine \eqref{eq: gamma a final ineq},  \eqref{eq: gamma a l1 norm final} and \eqref{eq: gamma a final eigen} to get the desired conclusions.

\subsubsection{Proof of Lemma \ref{lem: inf lemma gamma delta}}
Similar to \eqref{eq: basic ineq for gamma a}, by basic inequality, 
\begin{align}
	\frac{1}{2n_0}\twonorma{\bXk{0}_{\hbeta, -j}(\hrho_j - \brho_j)}^2 &\leq \frac{1}{n_0}\<\bWk{0}_{\hbeta}(\bWk{0}_{\bbeta})^{-1}\bmetak{0}_{\bbeta, j}, \bXk{0}_{\hbeta, -j}(\hrho_j - \brho_j)\> \\
	&\quad + \frac{1}{n_0}\<\bXk{0}_{\hbeta, -j}(\hrho_j - \brho_j), \bXk{0}_{\hbeta, -j}(\bgammaah_j-\hgammaah_j)\> \\
	&\quad + \tlambda_j(\onenorm{\brho_j} - \onenorm{\hrho_j}).
\end{align}

Similar to the analysis in the proof of Lemma \ref{lem: inf lemma gamma a},
\begin{align}
	\frac{1}{4n_0}\twonorma{\bXk{0}_{\hbeta, -j}(\hrho_j - \brho_j)}^2 &\leq C\left(\sqrt{\frac{\log p}{n_0}} + \Re_1\right)h_1\cdot \onenorm{\hrho_j - \brho_j} + C\twonorm{\hgammaah_j-\bgammaah_j}^2 \\
	&\quad + C\twonorm{\hbeta - \bbeta}^2 + \frac{C}{n_0}\infnorma{\bXk{0}_{\bbeta, -j}\bmetak{0}_j}\cdot \onenorm{\hrho_j - \brho_j} \\
	&\quad + \tlambda_j(\onenorm{\brho_j} - \onenorm{\hrho_j})
\end{align}
with probability at least $1-n_0^{-1}$. Note that $\frac{1}{n_0}\infnorm{\bXk{0}_{\bbeta, -j}\bmetak{0}_j} \lesssim \sqrt{\frac{\log p}{n_0}}$ with probability at least $1-p^{-1}$. Therefore for $\tlambda_j \geq C'\left(\sqrt{\frac{\log p}{n_0}} + \Re_1\right)$ with large enough $C' > 0$, we have $\tlambda_j \geq C\left(\sqrt{\frac{\log p}{n_0}} + \Re_1\right)h_1 + \frac{C}{n_0}\infnorm{\bXk{0}_{\bbeta, -j}\bmetak{0}_j}$. Then with probability at least $1-n_0^{-1}$,
\begin{equation}
	\frac{1}{4n_0}\twonorma{\bXk{0}_{\hbeta, -j}(\hrho_j - \brho_j)}^2 \leq  2\tlambda_j\onenorm{\brho_j} - \frac{1}{2}\tlambda_j\onenorm{\hrho_j-\brho_j} + C\twonorm{\hgammaah_j-\bgammaah_j}^2 + C\twonorm{\hbeta - \bbeta}^2,
\end{equation}
which leads to
\begin{align}
	\onenorm{\hrho_j-\brho_j} &\leq 4\onenorm{\brho_j} + C\tlambda_j^{-1}\twonorm{\hgammaah_j-\bgammaah_j}^2 + C\tlambda_j^{-1}\twonorm{\hbeta - \bbeta}^2 \\
	&\lesssim h_1 + \Re_1 + \left(\sqrt{\frac{n_0}{\log p}}\Re_1^2\right) \wedge \Re_1,
\end{align}
with probability at least $1-\Kah n_0^{-1}$.
Similar to the trick we used before, we can get
\begin{equation}
	\twonorm{\hrho_j-\brho_j}^2 \lesssim \tlambda_j h_1 + \twonorm{\hgammaah_j-\bgammaah_j}^2 + \twonorm{\hbeta - \bbeta}^2 \lesssim h_1\sqrt{\frac{\log p}{n_0}} + h_1\Re_1 + \Re_1^2
\end{equation}
with probability at least $1-\Kah n_0^{-1}$, which completes the proof.

\subsection{Proof of theorems}

\subsubsection{Proof of Theorem \ref{thm: l2 with assumption}}\label{subsubsec: proof of thm 1}
\underline{\textbf{Transferring step:}}	Define $\huah = \hwah - \bwah$ and $\mathcal{D} = \{(\bXk{k}, \byk{k})\}_{k \in \transet}$. We first claim that when $\lambda_{\bw} \geq 2\infnorm{\nabla L(\bwah, \md)}$, with probability at least $1-C_3\exp\{-C_4(\nah + n_0)\}$, it holds that
	\begin{equation}\label{eq: claim}
		\twonorm{\huah} \leq  8\kappa_2C_1h\sqrt{\frac{\log p}{\nah + n_0}} + 3\frac{\sqrt{s}}{\kappa_1}\lambda_{\bw} + 2\sqrt{\frac{C_1}{\kappa_1}h\lambda_{\bw}}.
	\end{equation}
	To see this, first by the definition of $\hwah$, H\"older inequality and Lemma \ref{lem: delta a}, we have
	\begin{align}
		\delta \hat{L}(\huah, \md) &\leq \lambda_{\bw}(\onenorm{\bwah_S} + \onenorm{\bwah_{S^c}}) - \lambda_{\bw}(\onenorm{\hwah_S} + \onenorm{\hwah_{S^c}}) + \nabla \hat{L}(\bw, \md)^T\huah \\
		&\leq  \lambda_{\bw}(\onenorm{\bwah_S} + \onenorm{\bwah_{S^c}}) - \lambda_{\bw}(\onenorm{\hwah_S} + \onenorm{\hwah_{S^c}}) + \frac{1}{2}\lambda_{\bw}\onenorm{\huah} \\
		&\leq \frac{3}{2}\lambda_{\bw}\onenorm{\huah_S} - \frac{1}{2}\lambda_{\bw}\onenorm{\huah_{S^c}} + 2\lambda_{\bw}\onenorm{\bwah_{S^c}} \\
		&\leq \frac{3}{2}\lambda_{\bw}\onenorm{\huah_S} - \frac{1}{2}\lambda_{\bw}\onenorm{\huah_{S^c}} + 2\lambda_{\bw}C_1h. \label{eq: posi c}
	\end{align}
	If the claim does not hold, consider $\mathds{C} = \left\{\bu: \frac{3}{2}\onenorm{\bu_S} - \frac{1}{2}\onenorm{\bu_{S^c}} + 2C_1h \geq 0\right\}$. Due to \eqref{eq: posi c} and the convexity of $\hat{L}$, $\huah \in \mathds{C}$. Then for any $t \in (0,1)$, it's easy to see that
	\begin{equation}
		\frac{1}{2}\onenorm{t\huah_{S^c}} = t\cdot \frac{1}{2}\onenorm{\huah_{S^c}} \leq t\cdot \left(\frac{3}{2}\onenorm{\huah_{S}} + 2C_{\bw}h\right) \leq \frac{3}{2}\onenorm{t\huah_{S}} + 2C_1h,
	\end{equation}
	implying that $t\huah \in \mathds{C}$. We could find some $t$ satisfying that $\twonorm{t\huah} > 8\kappa_2C_1h\sqrt{\frac{\log p}{\nah + n_0}} + 3\frac{\sqrt{s}}{\kappa_1}\lambda_{\bw} + 2\sqrt{\frac{C_1}{\kappa_1}h\lambda_{\bw}}$ and $\twonorm{t\huah} \leq 1$. Denote $\tilde{\bu}^{\mah} = t\huah$ and $F(\bu) = \hat{L}(\bwah + \bu, \md) - \hat{L}(\bwah) + \lambda_{\bw}(\onenorm{\bwah + \bu} - \onenorm{\bwah})$. Since $F(\bm{0}) = 0$ and $F(\huah) \leq 0$, by convexity,
	\begin{equation}\label{eq: f neg}
		F(\tilde{\bu}^{\mah}) = F(t\huah + (1-t)\bm{0}) \leq tF(\huah) \leq 0.
	\end{equation}
	However, by Lemma \ref{lem: re glm} and the same trick of \eqref{eq: posi c},
	\begin{align}
		F(\tilde{\bu}^{\mah}) &\geq \delta \hat{L}(\huah, \md) + \nabla \hat{L}(\bwah)^T \tilde{\bu}^{\mah} -\lambda_{\bw}\onenorm{\bwah} + \lambda_{\bw}\onenorm{\bwah + \tilde{\bu}^{\mah}} \\
		&\geq \kappa_1 \twonorm{\tilde{\bu}^{\mah}}^2 - \kappa_1\kappa_2\sqrt{\frac{\log p}{\nah + n_0}}\onenorm{\tilde{\bu}^{\mah}}\twonorm{\tilde{\bu}^{\mah}} - \frac{3}{2}\lambda_{\bw}\onenorm{\tilde{\bu}^{\mah}_S} + \frac{1}{2}\lambda_{\bw}\onenorm{\tilde{\bu}^{\mah}_{S^c}} - 2\lambda_{\bw}C_1h \\
		&\geq \kappa_1 \twonorm{\tilde{\bu}^{\mah}}^2 - \kappa_1\kappa_2\sqrt{\frac{\log p}{\nah + n_0}}\onenorm{\tilde{\bu}^{\mah}}\twonorm{\tilde{\bu}^{\mah}} - \frac{3}{2}\lambda_{\bw}\onenorm{\tilde{\bu}^{\mah}_S}  - 2\lambda_{\bw}C_1h.
	\end{align}
	Note that since $\tilde{\bu}^{\mah} \in \mathds{C}$, it holds that
	\begin{equation}
		\frac{1}{2}\onenorm{\tilde{\bu}^{\mah}} \leq 2\onenorm{\tilde{\bu}^{\mah}_S} + 2C_{\bw}h \leq 2\sqrt{s}\twonorm{\tilde{\bu}^{\mah}} + 2C_1h.
	\end{equation}
	When $\nah + n_0 > 16\kappa_2^2s\log p$, we have $2k_2\sqrt{\frac{s\log p}{\nah+n_0}} \leq \frac{1}{2}$. Then it follows
	\begin{equation}
		F(\tilde{\bu}^{\mah}) \geq \frac{1}{2}\kappa_1 \twonorm{\tilde{\bu}^{\mah}}^2 - \left[2\kappa_1\kappa_2\sqrt{\frac{\log p}{\nah + n_0}}C_{\bw}h + \frac{3}{2}\lambda_{\bw}\sqrt{s}\right]\twonorm{\tilde{\bu}^{\mah}} - 2\lambda_{\bw}C_1h > 0,
	\end{equation}
	which conflicts with \eqref{eq: f neg}. Therefore our claim at the beginning holds. 
	
	Next, let's prove $\infnorm{\nabla \hat{L}(\bwah)} \lesssim \sqrt{\frac{\log p}{\nah + n_0}}$ with probability at least $1-(\nah + n_0)^{-1}$. To see this, notice that
	\begin{align}
		\nabla \hat{L}(\bwah) &= \frac{1}{\nah + n_0}\sum_{k \in \transet}(\bXk{k})^T[-\byk{k} + \bpsi'(\bXk{k}\bwah)] \\
		&= \frac{1}{\nah + n_0}\sum_{k \in \transet}(\bXk{k})^T[-\byk{k} + \bpsi'(\bXk{k}\bbeta)]\\ 
		&\quad + \frac{1}{\nah + n_0}\sum_{k \in \transet}(\bXk{k})^T[-\bpsi'(\bXk{k}\bbeta) + \bpsi'(\bXk{k}\bwah)]. \label{eq: decomp nabla}
	\end{align}
	Following a similar idea in the proof of Lemma 6 in \cite{negahban2009unified}, under Assumptions \ref{asmp: convexity}-\ref{asmp: second derivative} and the fact $\nah \geq Cs\log p$, we can show that
	\begin{equation}
		\frac{1}{\nah + n_0}\infnorma{\sum_{k \in \transet}(\bXk{k})^T[-\byk{k} + \bpsi'(\bXk{k}\bbeta)]} \lesssim \sqrt{\frac{\log p}{\nah + n_0}},
	\end{equation}
	with probability at least $1-(\nah + n_0)^{-1}$.
	
	The remaining step is to bound the infinity norm of the second term in \eqref{eq: decomp nabla}. Denote $V_{ij}^{(k)} = x_{ij}^{(k)}[-\psi'((\bxk{k}_i)^T\bbeta) + \psi'((\bxk{k}_i)^T\bwah)]$. 
	Under Assumption \ref{asmp: second derivative}, by mean value theorem and Lemma \ref{lem: delta a}, there exists $v_{i}^{(k)} \in (0, 1)$ such that
			\begin{align}
				&\frac{1}{\nah+n_0}\sum_{k\in \transet}\sum_{i=1}^{n_k}V_{ij}^{(k)}\\
				&= \frac{1}{\nah+n_0}\sum_{k\in \transet}\sum_{i=1}^{n_k} \psi''((\bxk{k}_i)^T\bbeta+v_{i}^{(k)}(\bxk{k}_i)^T(\bwah-\bbeta))x_{ij}^{(k)}(\bxk{k}_i)^T(\bwah-\bbeta).
			\end{align}
			$\psi''((\bxk{k}_i)^T\bbeta+v_{i}^{(k)}(\bxk{k}_i)^T(\bwah-\bbeta))x_{ij}^{(k)}$ is $M_{\psi}^2\kappa_u^2$-subGaussian due to the almost sure boundedness of $\psi''$ in Assumption \ref{asmp: second derivative}. And $(\bxk{k}_i)^T(\bwah-\bbeta)$ is a $4C_1^2h^2$-subGaussian due to Lemma \ref{lem: delta a}. Then the multiplication is a $4C_1^2M_{\psi}^2\kappa_u^2h^2$-subexponential variable. By definition of $\bwah$, $\frac{1}{\nah+n_0}\sum_{k\in \transet}\sum_{i=1}^{n_k}V_{ij}^{(k)}$ has zero mean. Notice that the infinity norm of the second term in \eqref{eq: decomp nabla} equals to $\frac{1}{\nah+n_0}\sup_{j = 1,\ldots,p}\norm{\sum_{k\in \transet}\sum_{i=1}^{n_k}V_{ij}^{(k)}}$, by tail bounds of subexponential variables and union bounds, we have
			\begin{equation}
				\frac{1}{\nah+n_0}\sup_{j = 1,\ldots,p}\norma{\sum_{k\in \transet}\sum_{i=1}^{n_k}V_{ij}^{(k)}} \lesssim C_1M_{\psi}\kappa_u h \sqrt{\frac{\log p}{\nah+n_0}},
			\end{equation}
			with probability at least $1-(\nah + n_0)^{-1}$. Therefore $\infnorm{\nabla \hat{L}(\bwah)} \lesssim \sqrt{\frac{\log p}{\nah + n_0}}$ holds with probability at least $1-(\nah + n_0)^{-1}$. Plugging the rate into \eqref{eq: claim}, we have
	\begin{equation}
		\twonorm{\huah} \lesssim h\sqrt{\frac{\log p}{\nah + n_0}} + \sqrt{\frac{s\log p}{\nah + n_0}}  + \left(\frac{\log p}{\nah+n_0}\right)^{1/4}\sqrt{h}, \label{eq: ua 2 norm}
	\end{equation}
	with probability at least $1-(\nah + n_0)^{-1}$, when $\lambda_{\bw} \asymp C_{\bw}\sqrt{\frac{\log p}{\nah + n_0}}$ with $C_{\bw} > 0$ sufficiently large. Since $\huah \in \mathds{C}$, \eqref{eq: ua 2 norm} encloses
	\begin{equation}
		\onenorm{\huah} \lesssim s\sqrt{\frac{\log p}{\nah + n_0}}  + \left(\frac{\log p}{\nah+n_0}\right)^{1/4}\sqrt{sh} + h\left(1+\sqrt{\frac{s\log p}{\nah + n_0}}\right), \label{eq: ua 1 norm}
	\end{equation}
	with probability at least $1-(\nah + n_0)^{-1}$.
	
	\underline{\textbf{Debiasing step:}} 
Denote $\md^{(0)} = (\bXk{0}, \byk{0})$, $\hat{L}^{(0)}(\bw, \md^{(0)}) = -\frac{1}{n_0}(\byk{0})^T \bXk{0} \bw $\\ $+\frac{1}{n_0}\sum_{i=1}^{n_0}\psi((\bxk{0}_i)^T\bw)$, $\nabla \hat{L}^{(0)}(\bw, \md^{(0)}) 
	= -\frac{1}{n_0}(\bXk{0})^T \byk{0} + \frac{1}{n_0}(\bXk{0})^T\bpsi'(\bXk{0}\bw)$, $\bdeltaah = \bbeta - \bwah$, $\hbeta = \hwah + \hdeltaah$, $\hvah = \hdeltaah - \bdeltaah$, and $\delta \hat{L}^{(0)}(\bdelta, \md) = \hat{L}^{(0)}(\hwah + \bdelta, \md^{(0)}) - \hat{L}^{(0)}(\hwah + \bdeltaah, \md^{(0)}) - \nabla \hat{L}^{(0)}(\hwah + \bdeltaah, \md^{(0)})^T \hvah$. 
	
	Similar to \eqref{eq: posi c}, when $\lambda_{\bdelta} \geq 2\infnorm{\nabla \hat{L}^{(0)}(\bbeta, \md^{(0)})}$, we have
		\begin{align}
			\delta \hat{L}^{(0)}(\hdeltaah, \md) &\leq \lambda_{\bdelta}(\onenorm{\bdeltaah}-\onenorm{\hdeltaah}) - \nabla \hat{L}^{(0)}(\hwah + \bdeltaah, \md^{(0)})^T \hvah \\
			&\leq \lambda_{\bdelta}(2\onenorm{\bdeltaah} - \onenorm{\hvah}) + \infnorm{\nabla \hat{L}^{(0)}(\bbeta, \md^{(0)})}\onenorm{\hvah}\\
			&\quad - \left[\nabla \hat{L}^{(0)}(\hwah + \bdeltaah, \md^{(0)}) - \nabla \hat{L}^{(0)}(\bbeta, \md^{(0)})\right]^T\hvah\\
			&\leq 2\lambda_{\bdelta}\onenorm{\bdeltaah} - \frac{1}{2}\lambda_{\bdelta}\onenorm{\hvah} \\
			&\quad - \frac{1}{n_0}\left[\bpsi'((\bXk{0})^T(\hwah+\bdeltaah))-\bpsi'((\bXk{0})^T\bbeta)\right]^T\hvah\\
			&\leq 2\lambda_{\bdelta}\onenorm{\bdeltaah} - \frac{1}{2}\lambda_{\bdelta}\onenorm{\hvah} +\frac{1}{4c_0}M_{\psi}^2\cdot \frac{1}{n_0}\twonorm{\bXk{0}\huah}^2 + c_0 \cdot \frac{1}{n_0}\twonorm{\bXk{0}\hvah}^2, \label{eq: sec step right}
		\end{align}
		where $c_0 > 0$ is an enough small constant. The last inequality holds because
		\begin{align}
			&- \frac{1}{n_0}\left[\bpsi'((\bXk{0})^T(\hwah+\bdeltaah))-\bpsi'((\bXk{0})^T\bbeta)\right]^T\hvah \\
			&\quad = \frac{1}{n_0}(\huah)^T(\bXk{0})^T\Lambda^{(0)}\bXk{0}\hvah \\
			&\leq \frac{1}{4c_0}M_{\psi}^2\cdot \frac{1}{n_0}\twonorm{\bXk{0}\huah}^2 + c_0 \cdot \frac{1}{n_0}\twonorm{\bXk{0}\hvah}^2,
		\end{align}
		where $\Lambda^{(0)} = \text{diag}(\{\psi''((\bxk{0}_i)^T\bbeta + t_i(\bxk{0}_i)^T\huah)\}_{i=1}^{n_0})$ is a $n_0 \times n_0$ diagonal matrix and by Assumption \ref{asmp: second derivative}, $\maxnorm{\Lambda^{(0)}} \leq M_{\psi}$. Denote $\tilde{\bv}^{\mah} = t\hvah$ and $F^{(0)}(\bv) = \hat{L}^{(0)}(\hwah + \bdeltaah + \bv, \md^{(0)}) - \hat{L}^{(0)}(\hwah+\bdeltaah, \md^{(0)}) + \lambda_{\bdelta}(\onenorm{\bdeltaah + \bv} - \onenorm{\bdeltaah})$. Since $F(\bm{0}) = 0$ and $F^{(0)}(\hvah) \leq 0$, by convexity, for any $t \in (0, 1]$,
		\begin{equation}\label{basic ineq v 2}
			F^{(0)}(\tilde{\bv}^{\mah}) = F^{(0)}(t\hvah + (1-t)\bm{0}) \leq tF^{(0)}(\huah) \leq 0.
		\end{equation}
		We set $t \in (0, 1]$ such that $\twonorm{\tilde{\bv}^{\mah}} \leq 1$, which allows us to apply Proposition 1 in \cite{loh2015regularized} on $\tilde{\bv}^{\mah}$. By basic inequality \eqref{basic ineq v 2} and the same arguments in \eqref{eq: sec step right},
		\begin{align}
			c_1\twonorm{\tilde{\bv}^{\mah}}^2 - c_2\cdot \frac{\log p}{n_0}\onenorm{\tilde{\bv}^{\mah}}^2 &\leq  F^{(0)}(\tilde{\bv}^{\mah}) - \nabla \hat{L}^{(0)}(\hwah + \bdeltaah, \md^{(0)})^T \tilde{\bv}^{\mah} \\
			&\leq 2\lambda_{\bdelta}\onenorm{\bdeltaah} - \frac{1}{2}\lambda_{\bdelta}\onenorm{\tilde{\bv}^{\mah}} +\frac{1}{4c_0}M_{\psi}^2\cdot \frac{1}{n_0}\twonorm{\bXk{0}\huah}^2 \\
			&\quad + c_0 \cdot \frac{1}{n_0}\twonorm{\bXk{0}\tilde{\bv}^{\mah}}^2. \label{eq: inter ineq proof thm 1}
		\end{align}
		In the discussion of transferring step, we showed that $\onenorm{\huah_{S^c}} \leq 3\onenorm{\huah_{S}} + 4C_1h$. We leverage on this fact to bound $\frac{1}{n_0}\twonorm{\bXk{0}\huah}^2$ by $\twonorm{\bXk{0}\huah}^2$.
		
		If $3\onenorm{\huah_{S}} \geq 4C_1h$, we will have $\onenorm{\huah_{S^c}} \leq 6\onenorm{\huah_{S}}$. Then by Theorem 1.6 of \cite{zhou2009restricted},
		\begin{equation}\label{eq: rsc upper version}
			\frac{1}{n_0}\twonorm{\bXk{0}\huah}^2 \lesssim \twonorm{\huah}^2 \lesssim \frac{s\log p}{\na+n_0} + h\cdot \sqrt{\frac{\log p}{\na + n_0}},
		\end{equation}
		with probability at least $1-C(\na+n_0)^{-1} -C\exp\{-n_0\}$.
		
		If $3\onenorm{\huah_{S}} < 4C_1h$, we will have $\onenorm{\huah_{S^c}} \leq 8C_1h \leq \sqrt{s}$. Also $\twonorm{\huah} \leq 1$ with probability $1-C(\na+n_0)^{-1}$. Denote 
		\begin{align}
			\Pi_0(s) &= \{\bu \in \mathbb{R}^p: \twonorm{\bu} \leq 1, \zeronorm{\bu} \leq s\}, \\
			\Pi_1(s) &= \{\bu \in \mathbb{R}^p: \twonorm{\bu} \leq 1, \onenorm{\bu} \leq \sqrt{s}\}.
		\end{align}
		By Lemma 3.1 of \cite{plan2013one}, $\Pi_1(s) \subseteq 2\overline{\text{conv}}(\Pi_0(s))$, where $\overline{\text{conv}}(\Pi_0(s))$ is the closure of convex hull of $\Pi_0(s)$. Then following a similar argument in the proof of Theorem 2.4 in \cite{mendelson2008uniform} leads to \eqref{eq: rsc upper version} as well. 
		
		Next we want to bound $\frac{1}{n_0}\twonorm{\bXk{0}\tilde{\bv}^{\mah}}^2$ by $\twonorm{\tilde{\bv}^{\mah}}^2$. By another basic inequality, 
		\begin{align}
			0&\leq \hat{L}^{(0)}(\hbeta, \md^{(0)}) - \hat{L}^{(0)}(\bbeta, \md^{(0)}) - \nabla \hat{L}^{(0)}(\bbeta, \md^{(0)})^T (\hbeta-\bbeta) \\
			&\leq \lambda_{\bdelta}(\onenorm{\bbeta-\hwah} - \onenorm{\hdeltaah}) + \infnorm{\nabla \hat{L}^{(0)}(\bbeta, \md^{(0)})}\onenorm{\hbeta-\bbeta}\\
			&\leq \lambda_{\bdelta}(\onenorm{\bbeta-\hwah} - \onenorm{\hdeltaah}) + \frac{1}{2}\lambda_{\bdelta}\onenorm{\hbeta-\bbeta}\\
			&\leq \frac{3}{2}\lambda_{\bdelta}\onenorm{\bbeta-\hwah} - \frac{1}{2}\lambda_{\bdelta}\onenorm{\hdeltaah}\\
			&\leq \frac{3}{2}\lambda_{\bdelta}C_1h + \frac{3}{2}\lambda_{\bdelta}\onenorm{\hua} - \frac{1}{2}\lambda_{\bdelta}\onenorm{\hdeltaah},
		\end{align}
		which implies
		\begin{equation}
			\onenorm{\hvah} \leq \onenorm{\hdeltaah} + C_1h \leq 3\onenorm{\hua} + 4C_1h.
		\end{equation}
		Combined with \eqref{eq: ua 1 norm}, this leads to $\onenorm{\tilde{\bv}^{\mah}} \leq \onenorm{\hvah} \leq \sqrt{s}$ when $s\log p/(\nah+n_0)$ and $h$ are small enough. Due to the strict convexity, $\delta \hat{L}^{(0)}(\hdeltaah, \md) > 0$, leading to $\onenorm{\hdeltaah} \leq 3\onenorm{\huah} + 3h$. Then,
		\begin{equation}\label{eq: sec step left 1}
			\onenorm{\hvah} \leq \onenorm{\bbeta - \hwah} + \onenorm{\hdeltaah} \leq 4\onenorm{\huah} + 4h \leq \sqrt{s}.
		\end{equation}
		Similar to the analysis of the case $3\onenorm{\huah_{S}} < 4C_1h$ above, we can get
		\begin{equation}
			c_0 \cdot \frac{1}{n_0}\twonorm{\bXk{0}\tilde{\bv}^{\mah}}^2 \leq c_0 \cdot C\twonorm{\tilde{\bv}^{\mah}}^2,
		\end{equation}
		with probability at least $1- C\exp\{-n_0\}$.
		As long as $c_0 C < c_1/2$, by \eqref{eq: inter ineq proof thm 1}, we have
		\begin{align}
			c_1\twonorm{\tilde{\bv}^{\mah}}^2 - c_2\cdot \frac{\log p}{n_0}\onenorm{\tilde{\bv}^{\mah}}^2 &\leq 2\lambda_{\bdelta}\onenorm{\bdeltaah} - \frac{1}{2}\lambda_{\bdelta}\onenorm{\tilde{\bv}^{\mah}}  + C\frac{s\log p}{\nah+n_0} + Ch \sqrt{\frac{\log p}{\nah + n_0}} \\
			&\quad + c_1/2\twonorm{\tilde{\bv}^{\mah}}^2 \label{eq: important in proof of thm 1}
		\end{align}
		with probability at least $1-C'n_0^{-1}$.
		
		If $\lambda_{\bdelta}\onenorm{\bdeltaah} \leq C\frac{s\log p}{\nah+n_0} + Ch \sqrt{\frac{\log p}{\nah + n_0}}$:
		\begin{equation}
			\onenorm{\tilde{\bv}^{\mah}} \lesssim \left[\frac{s\log p}{\nah+n_0} + h \sqrt{\frac{\log p}{\nah + n_0}}\right] \cdot \sqrt{\frac{n_0}{\log p}} + \twonorm{\tilde{\bv}^{\mah}}^2.
		\end{equation}
		Since $\twonorm{\tilde{\bv}^{\mah}} \leq 1$, by \eqref{eq: important in proof of thm 1}, it holds 
		\begin{equation}
			\twonorm{\tilde{\bv}^{\mah}}^2 \lesssim \frac{s\log p}{\nah+n_0} + h\sqrt{\frac{\log p}{\nah + n_0}} \lesssim \frac{s\log p}{\nah+n_0} + \left[h\sqrt{\frac{\log p}{n_0}}\right] \wedge h^2,
		\end{equation}
		with probability at least $1-C'n_0^{-1}$.

		If $\lambda_{\bdelta}\onenorm{\bdeltaah} > C\frac{s\log p}{\nah+n_0} + Ch \sqrt{\frac{\log p}{\nah + n_0}}$: $\onenorm{\tilde{\bv}^{\mah}} \lesssim h + \twonorm{\tilde{\bv}^{\mah}}^2$, which leads to
		\begin{equation}
			\twonorm{\tilde{\bv}^{\mah}}^2 \lesssim 2\lambda_{\bdelta}\onenorm{\bdeltaah} - \frac{1}{2}\lambda_{\bdelta}\onenorm{\tilde{\bv}^{\mah}},
		\end{equation}
		implying $\onenorm{\tilde{\bv}^{\mah}} \leq 4\onenorm{\bdeltaah} \leq 4C_1h$. Besides, by plugging this result into \eqref{eq: important in proof of thm 1}, we have
		\begin{equation}
			\twonorm{\tilde{\bv}^{\mah}}^2 \lesssim \frac{s\log p}{\nah+n_0} + \left[h\sqrt{\frac{\log p}{n_0}}\right] \wedge h^2, \label{eq: va bound final}
		\end{equation}
		with probability at least $1-C'n_0^{-1}$.
		
		When $s\log p/(\nah+n_0)$ and $h$ is small enough, and because $h\sqrt{\frac{\log p}{n_0}} = o(1)$, the right hand side of \eqref{eq: va bound final} can be very small. This means that we showed $\twonorm{\tilde{\bv}^{\mah}} \leq c < 1$ with probability at least $1-C'n_0^{-1}$. Note that this result holds for any $t \in (0, 1]$ such that $\twonorm{\tilde{\bv}^{\mah}} \leq 1$. Now let's consider our interested vector $\hvah$. Suppose $\twonorm{\tilde{\bv}^{\mah}}^2 \geq C\frac{s\log p}{\nah+n_0} + C\left[h\sqrt{\frac{\log p}{n_0}}\right] \wedge h^2$ for some constant $C > 0$ with probability at least $C'n_0^{-1}$, then we can choose $t \in (0, 1]$ such that $c < \twonorm{\tilde{\bv}^{\mah}} \leq 1$, which is contradicted with the fact $\twonorm{\tilde{\bv}^{\mah}} \leq c$ with probability at least $1-C'n_0^{-1}$. Therefore 
		\begin{equation}
			\twonorm{\hvah}^2 \lesssim \frac{s\log p}{\nah+n_0} + \left[h\sqrt{\frac{\log p}{n_0}}\right] \wedge h^2 \lesssim 1,
		\end{equation}
		with probability at least $1-C'n_0^{-1}$.
		Then we can go over the analysis procedure of $\tilde{\bv}^{\mah}$ again with $\hvah$ (on the high-probability event $\twonorm{\hvah} \leq 1$) to get the $\ell_1$-bound
		\begin{equation}
			\onenorm{\hvah} \lesssim s\sqrt{\frac{\log p}{\nah+n_0}} + h,
		\end{equation}
		with probability at least $1-C'n_0^{-1}$.
		
		Finally, we connect the conclusions above with the upper bounds on $\twonorm{\huah}$ and $\onenorm{\huah}$, which completes the proof.

\subsubsection{Proof of Theorem \ref{thm: minimax}}
The analysis in the proof of Theorem B and Theorem 2 in \cite{li2021transfer}, which proved the $\ell_2$ minimax rate under linear models, can be directly extended to the GLMs. We only point out some key observations. The readers can refer to the supplements of their paper  for the full proof.

The proofs of Theorem B and Theorem 2 in \cite{li2021transfer} leverage on Fano's method, which involves Kullback-Leibler divergence. For $\bw_1$ and $\bw_2$ in $\mathbb{R}^p$, consider two GLMs where $y|\bx \sim \tp(y|\bx; \bw_1) = \rho(y)\exp\{y\bx^T\bw_1 - \psi(\bx^T\bw_1)\}$ and $y|\bx \sim \tp(y|\bx; \bw_2) = \rho(y)\exp\{y\bx^T\bw_2 - \psi(\bx^T\bw_2)\}$. Denote the corresponding joint distribution of $(\bx, y)$ as $f_{\bw_1}$ and $f_{\bw_2}$ and suppose marginal distributions of $\bx$ are the same. Then by definition and Assumption \ref{asmp: second derivative},
\begin{align}
	\textrm{KL}(f_{\bw_1}||f_{\bw_2}) &= \te_{\bx \sim f_{\bw_1}}\left[\psi''(\bx^T(t\bw_1+(1-t)\bw_2))\cdot [\bx^T(\bw_1-\bw_2)]^2\right] \\
	&\leq C\te_{\bx \sim f_{\bw_1}}[\bx^T(\bw_1-\bw_2)]^2],
\end{align}
for some constant $C>0$, if $\infnorm{\psi''} < \infty$ or $\infnorm{\bx} \leq U$ a.s. and $\onenorm{\bw_1 - \bw_2} \leq C'$ with some $C' > 0$. By using this fact, all the analysis of $\ell_2$-estimation error in their proof works out for GLM.

About the $\ell_1$-estimation error, we can make slight changes to make the argument work again. In case (\rom{1}) of the proof of Theorem 2 (or case (\rom{1}) of Theorem B) in \cite{li2021transfer}, replace the $\delta$-packing under $\ell_2$ norm with the $\delta$-packing under $\ell_1$ norm, and set $\delta_0 = c_0s\sqrt{\frac{\log p}{\na+n_0}}$. In case (\rom{2}-1), we can do the same and set $\delta_0 = c_0s\sqrt{\frac{\log p}{n_0}}$. In (\rom{2}-2), we suppose $s\sqrt{\log p/n_0} > h$. Set $\delta_0 = \bar{m}\sqrt{\frac{\log p}{n_0}} \asymp h$ where $\bar{m}$ is the integer part of $h\sqrt{n_0/\log p}$. In (\rom{2}-3), the same argument works and we set $\delta_0 = h$. So finally we can get the minimax $\ell_1$ rate $s\sqrt{\log p/(\na+n_0)} + [s\sqrt{\log p/n_0}]\wedge h$.

\subsubsection{Proof of Theorem \ref{thm: l2 without assumption}}
Throughout this proof, we denote $\bwah$ as any vector $\bw$ satisfying $\onenorm{\bw - \bbeta} \leq h$. Such a $\bwah$ indeed exists, e.g. $\bwah = \sum_{k \in \transet} \alpha_k \bwk{k}$. Note that $\bwah$ here does not necessarily enjoy the moment condition \eqref{eq: wa in population merging}, which will bring more bias. This is the price we have to pay for relaxing Assumption \ref{asmp: merging}. Other notations are defined the same as in the proof of Theorem \ref{thm: l2 with assumption}.

The main idea of the proof is similar to that in proof of Theorem \ref{thm: l2 with assumption}. We only highlight the different parts here and do not dig into all the details.

First, the claim in \eqref{eq: claim} still holds here, i.e. when $\lambda_{\bw} \geq 2\infnorm{\nabla \hat{L}(\bwah, \md)}$, with probability at least $1-C_3\exp\{-C_4(\nah + n_0)\}$, it holds that
	\begin{equation}\label{eq: ua without asump}
		\twonorm{\huah} \leq  8\kappa_2C_{\bw}h\sqrt{\frac{\log p}{\nah + n_0}} + 3\frac{\sqrt{s}}{\kappa_1}\lambda_{\bw} + 2\sqrt{\frac{1}{\kappa_1}h\lambda_{\bw}}.
	\end{equation}
	Via the decomposition in \eqref{eq: decomp nabla}, $\infnorm{\nabla L(\bwah, \md)}$ can be bounded by two parts where the first part has rate $\mathcal{O}_p\left(\sqrt{\frac{\log p}{\nah+n_0}}\right)$. Denote $V_{ij}^{(k)} = x_{ij}^{(k)}[-\psi'((\bxk{k}_i)^T\bwk{k}) + \psi'((\bxk{k}_i)^T\bwah)]$.
	\begin{align}
		&\frac{1}{\nah+n_0}\sum_{k\in \transet}\sum_{i=1}^{n_k}V_{ij}^{(k)}\\
		&= \frac{1}{\nah+n_0}\sum_{k\in \transet}\sum_{i=1}^{n_k} \psi''((\bxk{k}_i)^T\bwk{k}+v_{i}^{(k)}(\bxk{k}_i)^T(\bwah-\bwk{k}))x_{ij}^{(k)}(\bxk{k}_i)^T(\bwah-\bwk{k}).
	\end{align}
	Similar as before, the multiplication $\psi''((\bxk{k}_i)^T\bwk{k}+v_{i}^{(k)}(\bxk{k}_i)^T(\bwah-\bwk{k}))x_{ij}^{(k)}(\bxk{k}_i)^T(\bwah-\bwk{k})$ is a $C_1^2M_{\psi}^2\kappa_u^2h^2$-subexponential variable. And by Cauchy-Schwarz inequality and subGaussian properties \citep{vershynin2018high},
	\begin{equation}
		\te\norm{V_{ij}^{(k)}} \leq (\te\norm{x_{ij}^{(k)}}^2)^{1/2}(\te[(\bxk{k}_i)^T(\bwah-\bwk{k})]^2)^{1/2} \lesssim \kappa_u h.
	\end{equation}
	Therefore, by tail bounds of sub-exponential variables and union bounds, we have
			\begin{equation}
				\frac{1}{\nah+n_0}\sup_{j = 1,\ldots,p}\norma{\sum_{k\in \transet}\sum_{i=1}^{n_k}V_{ij}^{(k)}} \lesssim \kappa_u h + C_1M_{\psi}\kappa_u h \sqrt{\frac{\log p}{\nah+n_0}},
			\end{equation}
	with probability at least $1-(\nah+n_0)^{-1}\vee p^{-1}$, which implies that $\infnorm{\nabla \hat{L}(\bwah)} \lesssim \sqrt{\frac{\log p}{\nah + n_0}} + h$ with probability at least $1-(\nah+n_0)^{-1}\vee p^{-1}$.	Let $\lambda_{\bw} = C_{\bw}\left(\sqrt{\frac{\log p}{\nah + n_0}} + h\right)$ and $\lambda_{\bw} \geq 2\infnorm{\nabla L(\bwah, \md)}$ in high probability. Plugging it into \eqref{eq: ua without asump}, we get
	\begin{equation}\label{eq: detection cite}
		\twonorm{\huah} \lesssim \sqrt{\frac{s\log p}{\nah+n_0}} + \sqrt{s}h + \sqrt{h}\left(\frac{\log p}{\nah+n_0}\right)^{1/4},
	\end{equation}
	with probability at least $1-(\nah+n_0)^{-1}\vee p^{-1}$. Similarly, we can obtain the $\ell_1$-error bound
	\begin{equation}
		\onenorm{\huah} \lesssim s\sqrt{\frac{\log p}{\nah+n_0}} + sh + \sqrt{sh}\left(\frac{\log p}{\nah+n_0}\right)^{1/4},\label{eq: ua l1 bound without assumption}
	\end{equation}
	with probability at least $1-(\nah+n_0)^{-1}\vee p^{-1}$. Next, consider the debiasing step. Similar as the proof of Theorem \ref{thm: l2 with assumption}, let $\lambda_{\bdelta} =  C_{\bdelta}\sqrt{\frac{\log p}{n_0}}$ which satisfies $\lambda_{\bdelta} \geq 2\infnorm{\nabla L^{(0)}(\bbeta, \md^{(0)})}$ in high probability to get
		\begin{align}
			\twonorm{\hbeta - \bbeta}^2 &\lesssim \frac{s\log p}{\nah+n_0} + sh^2 + \left[h\sqrt{\frac{\log p}{n_0}}\right] \wedge h^2, \\
			\onenorm{\hbeta - \bbeta} &\lesssim s\sqrt{\frac{\log p}{\na+n_0}} + sh + \sqrt{sh}\cdot \left(\frac{\log p}{\nah+n_0}\right)^{1/4},
		\end{align}
	with probability at least $1-Cn_0^{-1}$.

\subsubsection{Proof of Theorem \ref{thm: a detection}}

\begin{figure}[!h]
\centering
\begin{tikzpicture}
  \matrix (m) [matrix of math nodes, row sep=6em, column sep=8em]
    { L_0(\bbeta)  & \hat{L}_0^{[r]}(\bbeta) & \hat{L}_0^{[r]}(\hbeta^{(0)[r]})  \\
       L_0(\bbetak{k}) & \hat{L}_0^{[r]}(\bbetak{k}) & \hat{L}_0^{[r]}(\hbeta^{(k)[r]}) \\ };
  { [start chain]
    \chainin (m-2-1);
    { [start branch=A] \chainin (m-1-1)
        [join={node[right,labeled] {\textup{Assumption \ref{asmp: idenfity a}}}}];}}
   {[start chain]
    \chainin (m-1-2);
     {[start branch=A]\chainin (m-1-1)[join={node[above,labeled] {\Gamma_2}}];}
    { [start branch=B]\chainin (m-1-3) [join={node[above,labeled] {\Gamma_1}}];}
   }
  { [start chain] 
  \chainin (m-2-3);
    \chainin (m-2-2) [join={node[above,labeled] {\Gamma_1}}];
    \chainin (m-2-1) [join={node[above,labeled] {\Gamma_2}}];}
    {[start chain] 
    \chainin (m-1-3);
    \chainin (m-2-3) [join={node[left,labeled] {\textup{our goal}}}];}
\end{tikzpicture}
\caption{The idea behind Assumptions \ref{asmp: merging}, \ref{asmp: idenfity a} and Theorem \ref{thm: a detection}.}
\label{fig: decomp}
\end{figure}
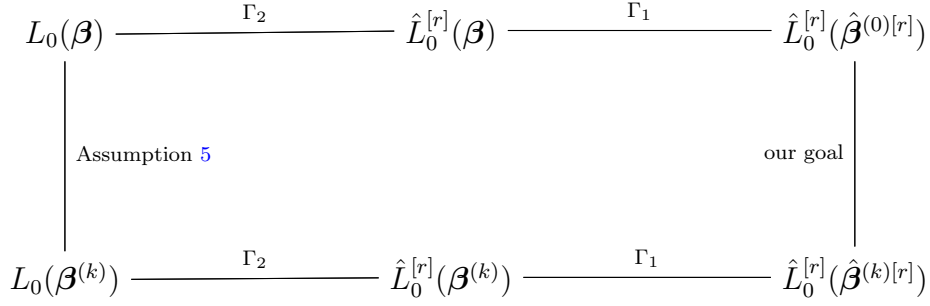

\begin{lemma}\label{lem: L a 0}
	Under Assumptions \ref{asmp: convexity}-\ref{asmp: second derivative}, $\sup\limits_{k \in \ma}L_0(\bbetak{k}) - L_0(\bbeta) \lesssim \twonorm{\bbetak{k} - \bbeta}^2 \lesssim h^2$.
\end{lemma}

\begin{proof}[Proof of Lemma \ref{lem: L a 0}]
	Note that the term involving $\te[\rho(\yk{0})]$ is canceled when taking the difference, therefore we drop that term and consider $L_0(\bw) = -\te[\psi'(\bw^T\bxk{0})\bw^T\bxk{0}] + \te[\psi(\bw^T\bxk{0})]$. Since $\nabla L_0(\bbeta) = \bm{0}$ and $\nabla^2 L_0(\bw) = \te[\psi''(\bw^T\bxk{0})\bxk{0}(\bxk{0})^T]$. By mean-theorem, $\exists t^{(k)}\in (0,1)$, such that
	\begin{equation}
		L_0(\bbetak{k}) - L_0(\bbeta) = (\bbetak{k} - \bbeta)^T \te[\psi''(\bbeta^T\bxk{0} + t^{(k)}(\bbetak{k} - \bbeta)^T\bxk{0})\bxk{0}(\bxk{0})^T](\bbetak{k} - \bbeta).
	\end{equation}
	Under Assumption \ref{asmp: second derivative}.(\rom{1}):
	\begin{equation}
		L_0(\bbetak{k}) - L_0(\bbeta) \leq M\te[((\bxk{0})^T(\bbetak{k}-\bbeta))^2] \lesssim \twonorm{\bbetak{k}-\bbeta}^2.
	\end{equation}
	Under Assumption \ref{asmp: second derivative}.(\rom{2}), by Cauchy-Schwarz inequality and the subGaussian moment bound \citep{vershynin2018high}:
	\begin{equation}
		L_0(\bbetak{k}) - L_0(\bbeta) \leq \left[\te\left(\max_{z:\norm{z}\leq 1}\psi''((\bxk{0})^T\bbeta + z)\right)^2\right]^{1/2}[\te((\bxk{0})^T(\bbetak{k}-\bbeta))^4]^{1/4} \lesssim \twonorm{\bbetak{k}-\bbeta}^2.
	\end{equation}
	The second half inequality automatically holds since $\bbetak{k}$ is a linear combination of $\bbeta$ and $\bwk{k}$. And it's easy to see that all the constants appearing in the inequalities are uniform for all $k \in \ma$, which completes the proof.
\end{proof}

Next, we prove Theorem \ref{thm: a detection}. We have
\begin{align}
	\hat{\sigma} &= \sqrt{\sum_{r=1}^3(\hat{L}_0^{[r]}(\hat{\bbeta}^{(0)[r]}) - L_0(\bbeta))^2/3} \leq \sqrt{\frac{2}{3}}\cdot \sum_{r=1}^3 \norm{\hat{L}_0^{[r]}(\hat{\bbeta}^{(0)[r]}) - L_0(\bbeta)}  \\
	&\leq \sqrt{\frac{2}{3}}\cdot\sum_{r=1}^3 \left[\norm{\hat{L}_0^{[r]}(\hat{\bbeta}^{(0)[r]}) - \hat{L}_0^{[r]}(\bbeta)} + \norm{\hat{L}_0^{[r]}(\bbeta) - L_0(\bbeta)}\right] \\
	&\lesssim \zeta\left(\Gamma_1^{(0)} + \Gamma_2^{(0)}\right),
\end{align}
with probability at least $1-g_1^{(0)}(\zeta) - g_2^{(0)}(\zeta)$.
As Figure \ref{fig: decomp} shows, by Lemma \ref{lem: L a 0}, for $k\in \ma$, there holds
\begin{align}
	\sup_r  \norm{\hat{L}_0^{[r]}(\hat{\bbeta}^{(k)[r]}) - \hat{L}_0^{[r]}(\hbeta^{(0)[r]})}&\leq 
	2\sup_r \norm{\hat{L}_0^{[r]}(\hat{\bbeta}^{(k)[r]}) - \hat{L}_0^{[r]}(\bbeta^{(k)[r]})} \\
	&\quad+ \sup_r \norm{\hat{L}_0^{[r]}(\bbeta^{(k)}) -\hat{L}_0^{[r]}(\bbeta)- L_0(\bbeta^{(k)})+L_0(\bbeta)} \\
	&\quad + \norm{L_0(\bbeta^{(k)}) - L_0(\bbeta)} \\
	&\lesssim \zeta\left(\Gamma_1^{(k)} + \Gamma_2^{(k)} + h^2\right) \\
	&\leq C_0(\hat{\sigma} \vee 0.01), \label{eq: l0 first eq}
\end{align}
simultaneously with probability at least $1-|\ma|\max_{k \in \ma}[g_1^{(k)}(\zeta) + g_2^{(k)}(\zeta)]$ for sufficiently small $\zeta > 0$ when $\min_{k \in \ma}n_k$ and $n_0$ go to infinity since $\Gamma_1^{(k)} + \Gamma_2^{(k)} + h^2 = o(1)$.
On the other hand, by Assumption \ref{asmp: idenfity a} and the fact $\nabla L_0(\bbeta) = \bm{0}$, for $k\in \mac$, 
\begin{align}
	&\inf_r\hat{L}_0^{[r]}(\hat{\bbeta}^{(k)[r]}) - \hat{L}_0^{[r]}(\hbeta^{(0)[r]}) \\
	&\geq \norm{L_0(\bbeta^{(k)}) - L_0(\bbeta)} - \Upsilon_1^{(k)} - \zeta\Gamma_1^{(k)} - \zeta\Gamma_2^{(k)} \\
	&=  \twonorm{\bbetak{k} - \bbeta}^2 \cdot\lambda_{\min}\left(\int_0^1\nabla^2 L_0(\bbeta + t(\bbetak{k}-\bbeta))dt\right) - \Upsilon_1^{(k)} - \zeta\Gamma_1^{(k)} - \zeta\Gamma_2^{(k)}  \\
	&\geq \twonorm{\bbetak{k} - \bbeta}^2 \cdot\underline{\lambda}  - \Upsilon_1^{(k)} - \zeta\Gamma_1^{(k)} - \zeta\Gamma_2^{(k)} \\
	&\geq C_1^2\left[(\Gamma_1^{(0)} + \Gamma_2^{(0)}) \vee 1\right] - \zeta\Gamma_1^{(k)} - \zeta\Gamma_2^{(k)} \\
	&> C_0(\hat{\sigma} \vee 0.01),
\end{align}
simultaneously with probability at least $1-|\mac|\max_{k \in \mac}[g_1^{(k)}(C_0^{-1}) + g_2^{(k)}(C_0^{-1})]-|\mac|\max_{k \in \mac}\allowbreak[g_1^{(k)}(\zeta) + g_2^{(k)}(\zeta)]$. It entails
\begin{align}
	\tp(\widehat{\mathcal{A}} \neq \mathcal{A}_h) & \leq \tp\left(\bigcup_{k \in \ma}\left\{\hat{L}_0^{(k)} - \hat{L}_0^{(0)} > C_0(\hat{\sigma} \vee 0.01) \right\}\bigcup \bigcup_{k \in \mac}\left\{\hat{L}_0^{(k)} - \hat{L}_0^{(0)} \leq C_0(\hat{\sigma} \vee 0.01) \right\}\right) \\
	&\leq \sum_{k \in \ma}\tp\left(\hat{L}_0^{(k)} - \hat{L}_0^{(0)} > C_0(\hat{\sigma} \vee 0.01)\right) + \sum_{k \in \mac}\tp\left(\hat{L}_0^{(k)} - \hat{L}_0^{(0)} \leq C_0(\hat{\sigma} \vee 0.01) \right) \\
	&\leq \sum_{k \in \ma}\tp\left(\inf_r\norm{\hat{L}_0^{[r]}(\hat{\bbeta}^{(k)[r]}) - \hat{L}_0^{[r]}(\hbeta^{(0)[r]})} > C_0(\hat{\sigma} \vee 0.01)\right) \\
	&\quad + \sum_{k \in \mac}\tp\left(\sup_r\norm{\hat{L}_0^{[r]}(\hat{\bbeta}^{(k)[r]}) - \hat{L}_0^{[r]}(\hbeta^{(0)[r]})} \leq C_0(\hat{\sigma} \vee 0.01) \right) \\
	&\leq |\ma|\max_{k \in \ma}[g_1^{(k)}(\zeta) + g_2^{(k)}(\zeta)] +|\mac|\max_{k \in \mac}[g_1^{(k)}(C_0^{-1}) + g_2^{(k)}(C_0^{-1})].
\end{align}
For any $\delta > 0$, there exist constants $C'(\delta)$ and $\zeta' > 0$ such that when $C_0 = C'(\delta)$, $K\max_{k \in \mac}[g_1^{(k)}(C_0^{-1}) + g_2^{(k)}(C_0^{-1})] \leq \delta/2$, $K\max_{k \in \ma}[g_1^{(k)}(\zeta') + g_2^{(k)}(\zeta')] < \delta/2$ and $C_1^2\big[(\Gamma_1^{(0)} + \Gamma_2^{(0)}) \vee 1\big] > \zeta'\Gamma_1^{(k)} + \zeta'\Gamma_2^{(k)}$. On the other hand, there exists $N = N(\delta) > 0$, such that when $\min_{k \in \transet} n_k > N(\delta)$, $\zeta'\left(\Gamma_1^{(k)} + \Gamma_2^{(k)} + h^2\right)/[C'(\delta)\cdot(\hat{\sigma} \vee 0.01)]$ is sufficiently small to make \eqref{eq: l0 first eq} hold. 

In summary, for any $\delta > 0$, there exist constants $C'(\delta)$ and $N = N(\delta) > 0$ such that when $C_0 = C'(\delta)$ and $\min_{k \in \transet} n_k > N(\delta)$, $\tp(\widehat{\mathcal{A}} \neq \mathcal{A}_h) \leq \delta$.

\subsubsection{Proof of Theorem \ref{thm: inf}}
Denote $\Re_1 = \sqrt{\frac{s\log p}{\na+n_0}} +  \left[h^{1/2}\left(\frac{\log p}{n_0}\right)^{1/4}\right] \wedge h$. First we present the following proposition.

\begin{proposition}\label{prop: gamma rate}
	Assume Assumptions \ref{asmp: convexity}-\ref{asmp: merging} and \ref{asmp: inf} hold (except (\rom{5})). Let $\lambda_j \asymp \sqrt{\frac{\log p}{\nah+n_0}} + \left[\frac{1}{\sqrt{s}}\cdot \left(\frac{\log p}{n_0}\right)^{1/4}h^{1/2}\right] \wedge \frac{h}{\sqrt{s}}$ for any $j = 1, \ldots, p$, and $\tlambda_j \asymp \sqrt{\frac{\log p}{n_0}} + \Re_1$. Then with probability at least $1-\Kah n_0^{-1}$,
	\begin{align}
	\twonorm{\hbgammak{0}_j - \bgammak{0}_j}^2 &\lesssim h_1\sqrt{\frac{\log p}{n_0}} + h_1\Re_1 + \Re_1^2, \label{eq: prop 2 eq 1}\\
	\onenorm{\hbgammak{0}_j - \bgammak{0}_j} &\lesssim \sqrt{s}\Re_1 + s^{1/4}\Re_1^{1/2}h_1^{1/2} + h_1, \label{eq: prop 2 eq 2}\\
	\norm{\htau^2_j - \tau_j^2} \vee \norm{\htau^{-2}_j - \tau_j^{-2}} &\lesssim \Re_1 + h^{1/2}_1\left(\frac{\log p}{n_0}\right)^{1/4} + h_1^{1/2}\Re_1^{1/2}+h_{\max}, \label{eq: prop 2 eq 3}\\
\end{align}
\end{proposition}

\eqref{eq: prop 2 eq 1} and \eqref{eq: prop 2 eq 2} can be proved by combining conclusions in Lemmas \ref{lem: inf lemma gamma a} and \ref{lem: inf lemma gamma delta}. The proof of \eqref{eq: prop 2 eq 3} can be found in Section \ref{subsubsec: proof of prop 2}.

Next we will leverage Proposition \ref{prop: gamma rate} to show \eqref{eq: inf clt} and \eqref{eq: inf var est bound}. First, we notice that
\begin{align}
	\hat{\bm{b}} &= \hbeta +\frac{1}{n_0}\hTheta(\bXk{0})^T[\bm{Y}^{(0)} - \bpsi'(\bXk{0}\hbeta)] \\
	&= \hbeta +\frac{1}{n_0}\hTheta(\bXk{0})^T[\bm{Y}^{(0)} - \bpsi'(\bXk{0}\bbeta)] + \frac{1}{n_0}\hTheta(\bXk{0})^T[\bpsi'(\bXk{0}\hbeta) - \bpsi'(\bXk{0}\bbeta)] \\
	&= \hbeta +\frac{1}{n_0}\hTheta(\bXk{0})^T[\bm{Y}^{(0)} - \bpsi'(\bXk{0}\bbeta)] + \frac{1}{n_0}\hTheta\hSigmak{0}_{\bbeta}(\bbeta - \hbeta) \\
	&\quad + \frac{1}{n_0}\hTheta(\bXk{0})^T\diag(\{\psi''(\tilde{u}_i^{(0)})\}) - \psi''((\bxk{0}_i)^T\bbeta)\}_{i=1}^{n_0})\bXk{0}(\bbeta - \hbeta).
\end{align}
where $\tilde{u}_i^{(0)}$ falls on the line between $(\bxk{0}_i)^T\bbeta$ and $(\bxk{0}_i)^T\hbeta$. Therefore for any $j = 1, \ldots, p$,
\begin{align}
	\hb_j - \beta_j &= \underbrace{\left[\bm{e}_j - (\hTheta_j)^T\hSigmak{0}_{\bbeta}\right]^T(\bbeta - \hbeta)}_\text{(1)} + \underbrace{\frac{1}{n_0}(\hTheta_j)^T(\bXk{0})^T[\bm{Y}^{(0)} - \bpsi'(\bXk{0}\bbeta)]}_\text{(2)} \\
	&\quad + \underbrace{\frac{1}{n_0}(\hTheta_j)^T(\bXk{0})^T\diag(\{\psi''(\tilde{u}_i^{(0)}) - \psi''((\bxk{0}_i)^T\bbeta)\}_{i=1}^{n_0})\bXk{0}(\bbeta - \hbeta)}_\text{(3)},
\end{align}
For convenience, we write the $j$-th row of the matrix
\begin{equation}
	\begin{pmatrix}
		1 &-\hgamma_{1,2} &\hdots  &-\hgamma_{1,p} \\
		-\hgamma_{2,1} &1 &\hdots  &-\hgamma_{2,p} \\
		\vdots &\vdots &\ddots &\vdots \\
		-\hgamma_{p,1} &-\hgamma_{p,2} &\hdots &-\hgamma_{p,p}
	\end{pmatrix}
\end{equation}
as $(\hbgammak{0}_j)^{\dagger}$ and the $j$-th row of matrix
\begin{equation}
	\begin{pmatrix}
		1 &-\bgamma_{1,2} &\hdots  &-\bgamma_{1,p} \\
		-\bgamma_{2,1} &1 &\hdots  &-\bgamma_{2,p} \\
		\vdots &\vdots &\ddots &\vdots \\
		-\bgamma_{p,1} &-\bgamma_{p,2} &\hdots &-\bgamma_{p,p}
	\end{pmatrix}
\end{equation}
as $(\bgammak{0}_j)^{\dagger}$. 

For (\rom{1}), notice that
\begin{align}
	&\infnorma{\hSigmak{0}_{\bbeta}\hTheta_j - \bm{e}_j} \\
	&= \infnorma{\hSigmak{0}_{\bbeta}\hTheta_j - \bSigmak{0}_{\bbeta}\bTheta_j} \\
	&= \infnorma{\frac{1}{n_0}\sum_{i=1}^{n_0}\bxk{0}_i\psi''((\bxk{0}_i)^T\bbeta)(\bxk{0}_i)^T\frac{(\hbgammak{0}_j)^{\dagger}}{\htau^2_j} - \te\left[\bxk{0}_i\psi''((\bxk{0}_i)^T\bbeta)(\bxk{0}_i)^T\frac{(\bgammak{0}_j)^{\dagger}}{\tau^2_j}\right]}\\
	&= \norma{\frac{1}{n_0}\sum_{i=1}^{n_0}\bxk{0}_i\psi''((\bxk{0}_i)^T\bbeta)(\bxk{0}_i)^T\frac{(\hbgammak{0}_j)^{\dagger}}{\htau^2_j} - \frac{1}{n_0}\sum_{i=1}^{n_0}\bxk{0}_i\psi''((\bxk{0}_i)^T\bbeta)(\bxk{0}_i)^T\frac{(\bgammak{0}_j)^{\dagger}}{\tau^2_j}} \\
	&\quad + \norma{\frac{1}{n_0}\sum_{i=1}^{n_0}\bxk{0}_i\psi''((\bxk{0}_i)^T\bbeta)(\bxk{0}_i)^T\frac{(\bgammak{0}_j)^{\dagger}}{\tau^2_j} - \te\left[\bxk{0}_i\psi''((\bxk{0}_i)^T\bbeta)(\bxk{0}_i)^T\frac{(\bgammak{0}_j)^{\dagger}}{\tau^2_j}\right]} \\
	&\lesssim \twonorm{\hbgammak{0}_j - \bgammak{0}_j} + \norm{\htau^{-2}_j-\tau^{-2}_j} + n_0^{-1/2} \\
	&\lesssim \Re_1 + h^{1/2}_1\left(\frac{\log p}{n_0}\right)^{1/4} + h_1^{1/2}\Re_1^{1/2} + h_{\max}+ n_0^{-1/2},
\end{align}
implying that
\begin{align}
	\text{(1)} &\leq \infnorma{\hSigmak{0}_{\bbeta}\hTheta_j - \bm{e}_j}\onenorm{\hbeta - \bbeta} \\
	&\lesssim_p \left[\Re_1 + h^{1/2}_1\left(\frac{\log p}{n_0}\right)^{1/4} + h_1^{1/2}\Re_1^{1/2} +h_{\max} + n_0^{-1/2}\right]\cdot \left(s\sqrt{\frac{\log p}{\nah+n_0}}+h\right) \\
	&= \smallo(n_0^{-1/2}), \label{eq: clt term 1}
\end{align}
under Assumption \ref{asmp: inf}.

For (3),
\begin{equation}
	\text{(3)} \leq \frac{1}{n_0}\infnorma{\bXk{0}\hTheta_j}\onenorma{\diag(\{\psi''(\tilde{u}_i^{(0)}) - \psi''((\bxk{0}_i)^T\bbeta)\}_{i=1}^{n_0})\bXk{0}(\bbeta - \hbeta)},
\end{equation}
where $\norm{[\psi''(\tilde{u}_i^{(0)}) - \psi''((\bxk{0}_i)^T\bbeta)]\cdot (\bxk{0}_i)^T(\bbeta - \hbeta)}\leq \psi'''(t\tilde{u}_i^{(0)} + (1-t)(\bxk{0}_i)^T\bbeta)\cdot\norm{\tilde{u}_i^{(0)}-(\bxk{0}_i)^T\bbeta}\norm{(\bxk{0}_i)^T(\bbeta - \hbeta)} \lesssim \norm{(\bxk{0}_i)^T(\bbeta - \hbeta)}^2$. Therefore by similar arguments to get \eqref{eq: rsc upper version}, we have
\begin{equation}\label{eq: clt term 3}
	\text{(3)} \lesssim_p (1 + \onenorm{\hTheta_j - \bTheta_j})\cdot \twonorm{\hbeta - \bbeta}^2 \lesssim_p \Re_1^2 = \smallo(n_0^{-1/2}) 
\end{equation}
under Assumption \ref{asmp: inf}.

For (2), we can see that
\begin{equation}\label{eq: two terms decomp theta}
	\text{(2)} = \frac{1}{n_0}\hTheta_j^T(\bXk{0})^T[\bm{Y}^{(0)} - \bpsi'(\bXk{0}\bbeta)] = \frac{1}{n_0}\sum_{i=1}^{n_0}\hTheta_j^T\bxk{0}_i[y^{(0)}_i - \psi'((\bxk{0}_i)^T\bbeta)],
\end{equation}
where $\te\{(\hTheta_j)^T\bxk{0}_i[y^{(0)}_i - \psi'((\bxk{0}_i)^T\bbeta)]\} = 0$ and 
\begin{align}
	\te\{\hTheta_j^T\bxk{0}_i[y^{(0)}_i - \psi'((\bxk{0}_i)^T\bbeta)]\}^2 &= \te\{(\hTheta_j^T\bxk{0}_i)^2\psi''((\bxk{0}_i)^T\bbeta)\} \\
	&= \bTheta_j^T \te\left[\bxk{0}(\bxk{0})^T\psi''((\bxk{0})^T\bbeta)\right]\bTheta_j \\
	&= \bTheta_j^T \bSigma_{\bbeta}^{(0)}\bTheta_j \\
	&\lesssim  \te\left(\bTheta_j^T\bxk{0}\right)^2 \\
	&< \infty,
\end{align}
by Assumption \ref{asmp: inf}. Thus by Lindeberg's conditions,
\begin{equation}\label{eq: clt term 2}
	\sqrt{n_0}\cdot (\text{2}) \xrightarrow{d} \mathcal{N}(0, \bTheta_{j,j}).
\end{equation}
On the other hand,
\begin{equation}
	\norma{\hTheta_j^T\hSigma_{\hbeta}\hTheta_j - \bTheta_{j,j}} \leq \norma{\hTheta_j^T(\hSigma_{\hbeta}\hTheta_j - \bSigmak{0}_{\bbeta}\bTheta_{j})} + \norma{(\hTheta_j - \bTheta_j)\bSigmak{0}_{\bbeta}\bTheta_{j}}.
\end{equation}
Note that
\begin{align}
	&\infnorma{\hSigma_{\hbeta}\hTheta_j - \bSigmak{0}_{\bbeta}\bTheta_{j}} \\
	&= \infnorma{\hSigma_{\hbeta}\frac{(\hbgammak{0}_j)^{\dagger}}{\htau_j^2} - \bSigmak{0}_{\bbeta}\frac{(\bgammak{0}_j)^{\dagger}}{\tau_j^2}}\\
	&\leq \infnorma{\frac{1}{\nah+n_0}\sum_{i,k}\left[\bxk{k}_i\psi''(\hbeta^T\bxk{k}_i)(\bxk{k}_i)^T\frac{(\hbgammak{0}_j)^{\dagger}}{\htau_j^2}\right]-\frac{1}{\nah+n_0}\sum_{i,k}\left[\bxk{k}_i\psi''(\bbeta^T\bxk{k}_i)(\bxk{k}_i)^T\frac{(\bgammak{0}_j)^{\dagger}}{\tau_j^2}\right]} \\
	&\quad + \frac{1}{\nah+n_0}\infnorma{\sum_{i,k}\left[\bxk{k}_i\psi''(\bbeta^T\bxk{k}_i)(\bxk{k}_i)^T\frac{(\bgammak{0}_j)^{\dagger}}{\tau_j^2}\right]-\sum_{k \in \transet}n_k\te\left[\bxk{k}_i\psi''(\bbeta^T\bxk{k}_i)(\bxk{k}_i)^T\frac{(\bgammak{0}_j)^{\dagger}}{\tau_j^2}\right]}\\
	&\quad + \infnorma{\sum_{k \in \transet}\frac{n_k}{\nah+n_0}\te\left[\bxk{k}_i\psi''(\bbeta^T\bxk{k}_i)(\bxk{k}_i)^T\frac{(\bgammak{0}_j)^{\dagger}}{\tau_j^2}-\bxk{0}_i\psi''(\bbeta^T\bxk{0}_i)(\bxk{0}_i)^T\frac{(\bgammak{0}_j)^{\dagger}}{\tau_j^2}\right]} \\
	&\lesssim \underbrace{\infnorma{\frac{1}{\nah+n_0}\sum_{i,k}\bxk{k}_i\left[\psi''(\hbeta^T\bxk{k}_i) - \psi''(\bbeta^T\bxk{k}_i)\right](\bxk{k}_i)^T\frac{(\hbgammak{0}_j)^{\dagger}}{\htau_j^2}}}_\text{(4)} \\
	&\quad + \underbrace{\infnorma{\frac{1}{\nah+n_0}\sum_{i,k}\bxk{k}_i\psi''(\bbeta^T\bxk{k}_i)(\bxk{k}_i)^T[(\hgammak{0}_j)^{\dagger} - (\bgammak{0}_j)^{\dagger}]\htau_j^{-2}}}_\text{(5)} \\
	&\quad + \underbrace{\infnorma{\frac{1}{\nah+n_0}\sum_{i,k}\bxk{k}_i\psi''(\bbeta^T\bxk{k}_i)\cdot (\bxk{k}_i)^T(\hgammak{0}_j)^{\dagger}(\htau^2_j - \tau^2_j)}}_\text{(6)} \\
	&\quad + \underbrace{\frac{1}{\nah+n_0}\infnorma{\sum_{i,k}\left[\bxk{k}_i\psi''(\bbeta^T\bxk{k}_i)(\bxk{k}_i)^T\frac{(\bgammak{0}_j)^{\dagger}}{\tau_j^2}\right]-\sum_{k \in \transet}n_k\te\left[\bxk{k}_i\psi''(\bbeta^T\bxk{k}_i)(\bxk{k}_i)^T\frac{(\bgammak{0}_j)^{\dagger}}{\tau_j^2}\right]}}_\text{(7)}\\
	&\quad + \underbrace{\infnorma{\sum_{k \in \transet}\frac{n_k}{\nah+n_0}\te\left[\bxk{k}_i\psi''(\bbeta^T\bxk{k}_i)(\bxk{k}_i)^T\frac{(\bgammak{0}_j)^{\dagger}}{\tau_j^2}-\bxk{0}_i\psi''(\bbeta^T\bxk{0}_i)(\bxk{0}_i)^T\frac{(\bgammak{0}_j)^{\dagger}}{\tau_j^2}\right]}}_\text{(8)}.
\end{align}

It's easy to see that
\begin{align}
	\text{(4)} &\lesssim \infnorma{\frac{1}{\nah+n_0}\sum_{i,k}\bxk{k}_i\cdot \psi'''(t_i^{(k)}(\bxk{k}_i)^T(\hbeta - \bbeta) + (\bxk{k}_i)^T\bbeta)\cdot (\bxk{k}_i)^T(\hbeta - \bbeta)\cdot(\bxk{k}_i)^T\frac{(\hbgammak{0}_j)^{\dagger}}{\htau_j^2}} \\
	&\lesssim \sqrt{\frac{1}{\nah+n_0}\sum_{k, i}\norm{(\bxk{k}_i)^T(\hbeta - \bbeta)}^2} \\
	&\lesssim \twonorm{\hbeta - \bbeta} \\
	&\lesssim \Re_1,
\end{align}
with probability at least $1-n_0^{-1}$.

Similarly,
\begin{align}
	\text{(5)} &\lesssim \twonorm{\hbgammak{0}_j- \bgammak{0}_j} + \norm{\htau_j^{-2} - \tau_j^{-2}} \lesssim \Re_1 + h^{1/2}_*\left(\frac{\log p}{n_0}\right)^{1/4} + h_1^{1/2}\Re_1^{1/2} + h_{\max}\\
	\text{(6)} &\lesssim \norm{\htau_j^{-2} - \tau_j^{-2}} \lesssim \Re_1 + h^{1/2}_*\left(\frac{\log p}{n_0}\right)^{1/4} + h_1^{1/2}\Re_1^{1/2}+ h_{\max}\\
	\text{(7)} &\lesssim (\nah+n_0)^{-1/2},\\
	\text{(8)} &\leq \sup_{k \in \mah}\maxnorm{\bSigmak{k}_{\bbeta}-\bSigmak{0}_{\bbeta}}\cdot \onenorma{\frac{(\bgammak{0}_j)^{\dagger}}{\tau_j^2}} \lesssim h_{\max},
\end{align}
with probability at least $1-\Kah n_0^{-1}$. Besides,
\begin{equation}
	\onenorm{\hTheta_j} \leq \onenorm{\bTheta_j} + \onenorm{\hTheta_j - \bTheta_j} \lesssim \sqrt{s}.
\end{equation}
with probability at least $1-\Kah n_0^{-1}$.
Therefore,
\begin{align}
	\norma{\hTheta_j^T(\hSigma_{\hbeta}\hTheta_j - \bSigmak{0}_{\bbeta}\bTheta_{j})} &\leq \onenorm{\hTheta_j}\infnorma{\hSigma_{\hbeta}\hTheta_j - \bSigmak{0}_{\bbeta}\bTheta_{j}} \\
	&\lesssim \sqrt{s}\left[\Re_1 + h^{1/2}_*\left(\frac{\log p}{n_0}\right)^{1/4} + h_1^{1/2}\Re_1^{1/2} + h_{\max}\right],
\end{align}
with probability at least $1-\Kah n_0^{-1}$. Similarly, we can obtain the same upper bound for the second term $\norma{(\hTheta_j - \bTheta_j)\bSigmak{0}_{\bbeta}\bTheta_{j}}$ in \eqref{eq: two terms decomp theta}. Then finally,
\begin{equation}
	\norma{\hTheta_j^T\hSigma_{\hbeta}\hTheta_j - \bTheta_{j,j}} \lesssim \sqrt{s}\left[\Re_1 + h^{1/2}_*\left(\frac{\log p}{n_0}\right)^{1/4} + h_1^{1/2}\Re_1^{1/2} + h_{\max}\right] = \smallo(1),
\end{equation}
with probability at least $1-\Kah n_0^{-1}$. Then by Slutsky's lemma, equations \eqref{eq: clt term 1}, \eqref{eq: clt term 2} and \eqref{eq: clt term 3},
\begin{equation}
	\frac{\sqrt{n_0}(\hb_j - \beta_j)}{\sqrt{\hTheta_j^T\hSigma_{\hbeta}\hTheta_j}} \xrightarrow{d} \mathcal{N}(0, 1),
\end{equation}
which completes the proof.

\subsubsection{Proof of Theorem \ref{thm: prediction error with assumption}}
Assume Assumption \ref{asmp: second derivative}.(\rom{1}) holds or Assumption \ref{asmp: second derivative}.(\rom{2}) with $h \leq C'U^{-1}\bar{U}$ for some $C' > 0$ holds. When $\lambda_{\bdelta} = C_{\bdelta}\sqrt{\frac{\log p}{n_0}}$, where $C_{\bdelta} > 0$ is a sufficiently large constant: By the definition of $\hbeta$, we have $\nabla L_{n_0}^{(0)}(\hbeta) + \lambda_{\bdelta}\cdot \partial \onenorm{\hbeta - \hwah} = \bm{0}$. Then by H\"older inequality,
\begin{align}
	\<\nabla L_{n_0}^{(0)}(\hbeta) - \nabla L_{n_0}^{(0)}(\bbeta), \hbeta - \bbeta\> &= \<-\lambda_{\bdelta}\cdot \partial \onenorm{\hbeta - \hwah}- \nabla L_{n_0}^{(0)}(\bbeta), \hbeta - \bbeta\> \\
	&\leq \lambda_{\bdelta}\onenorm{\hbeta - \bbeta} + \infnorm{\nabla L_{n_0}^{(0)}(\bbeta)}\onenorm{\hbeta - \bbeta}. \label{eq: holder kkt}
\end{align}
Considering the fact that $\infnorm{\nabla L_{n_0}^{(0)}(\bbeta)} \lesssim \sqrt{\frac{\log p}{n_0}}$ with probability at least $1-n_0^{-1}$ (Lemma 6 in \cite{negahban2009unified}) and the upper bound of $\onenorm{\hbeta - \bbeta}$ we prove in Theorem \ref{thm: l2 with assumption}, the desired upper bound of $\<\nabla L_{n_0}^{(0)}(\hbeta) - \nabla L_{n_0}^{(0)}(\bbeta), \hbeta - \bbeta\>$ follows.

\subsubsection{Proof of Theorem \ref{thm: prediction error without assumption}}
Assume Assumption \ref{asmp: second derivative}.(\rom{1}) holds or Assumption \ref{asmp: second derivative}.(\rom{2}) with $h \leq C'U^{-1}\bar{U}$ for some $C' > 0$ holds. If we take $\lambda_{\bdelta} = C_{\bdelta}\sqrt{\frac{\log p}{n_0}}$, where $C_{\bdelta} > 0$ is a sufficiently large constant: Similar to \eqref{eq: holder kkt}, we can obtain
\begin{equation}
	\<\nabla L_{n_0}^{(0)}(\hbeta) - \nabla L_{n_0}^{(0)}(\bbeta), \hbeta - \bbeta\> \leq \lambda_{\bdelta}\onenorm{\hbeta - \bbeta} + \infnorm{\nabla L_{n_0}^{(0)}(\bbeta)}\onenorm{\hbeta - \bbeta}. 
\end{equation}
To bound $\onenorm{\hbeta - \bbeta}$, it suffices to combine \eqref{eq: sec step left 1} and the upper bound of $\onenorm{\huah}$ in \eqref{eq: ua l1 bound without assumption}. Then the final upper bound follows.

\subsection{Proof of propositions}

\subsubsection{Proof of Proposition \ref{prop: gamma_1 gamma_2}}
The rate of $\Gamma_1$ can be derived from the following Lemma \ref{lem: gamma_1 rate} and the union bound, together with the tail inequality \eqref{eq: detection cite}. The rate of $\Gamma_2$ comes from the following Lemma \ref{lem: gamma_2 rate}. 
\begin{lemma}\label{lem: gamma_1 rate}
	Under the same assumptions as Theorem \ref{thm: a detection}, we have the following conclusions:
	\begin{enumerate}[(i)]
		\item For logistic regression model: 
		\begin{align}
			\sup_r\norm{\hat{L}_0^{[r]}(\hbeta^{(k)[r]})-\hat{L}_0^{[r]}(\bbetak{k})} &\leq C\left(1+\sqrt{\frac{1}{n_0}}\cdot \zeta\right)\cdot \sup_r \twonorm{\hbeta^{(k)[r]} - \bbetak{k}}, k\in \ma,\\
			\sup_r\norm{\hat{L}_0^{[r]}(\hbeta^{(k)[r]})-\hat{L}_0^{[r]}(\bbetak{k})} &\leq C\left(1+\sqrt{\frac{1}{n_0}}\cdot \zeta\right)\cdot \sup_r \twonorm{\hbeta^{(k)[r]} - \bbetak{k}}, k\in \mac,\\
			\sup_r\norm{\hat{L}_0^{[r]}(\hbeta^{(0)[r]})-\hat{L}_0^{[r]}(\bbeta)} &\leq C \sup_r\twonorm{\hbeta^{(0)[r]} - \bbeta}\cdot (1+\zeta),
		\end{align}
		with probability at least $1-\exp\{-\zeta^2\}$.
		\item For linear model:
		\begin{align}
			\sup_r\norm{\hat{L}_0^{[r]}(\hbeta^{(k)[r]})-\hat{L}_0^{[r]}(\bbetak{k})} &\leq C \left(1+\sqrt{\frac{1}{n_0}}\cdot \zeta\right)\cdot(\twonorm{\bwk{k}}\vee\twonorm{\bbeta}) \cdot  \twonorm{\hbeta^{(k)[r]} - \bbetak{k}}, k\in \ma,\\
			\sup_r\norm{\hat{L}_0^{[r]}(\hbeta^{(k)[r]})-\hat{L}_0^{[r]}(\bbetak{k})} &\leq C \left(1+\sqrt{\frac{1}{n_0}}\cdot \zeta\right)\cdot\twonorm{\bbetak{k}}\cdot  \twonorm{\hbeta^{(k)[r]} - \bbetak{k}}, k\in \mac,\\
			\sup_r\norm{\hat{L}_0^{[r]}(\hbeta^{(0)[r]})-\hat{L}_0^{[r]}(\bbeta)} &\leq C \twonorm{\bbeta}\cdot \twonorm{\hbeta^{(0)[r]} - \bbeta}\cdot (1+\zeta),
		\end{align}
		with probability at least $1-\exp\{-\zeta^2\}$.
		\item For Poisson regression model with $\sup_k \infnorm{\bxk{k}} \leq U$ a.s.:
		\begin{align}
			\sup_r\norm{\hat{L}_0^{[r]}(\hbeta^{(k)[r]})-\hat{L}_0^{[r]}(\bbetak{k})} &\leq C \left(1+\sqrt{\frac{1}{n_0}}\cdot \zeta\right)\cdot\exp\left(U(\onenorm{\bwk{k}}\vee \onenorm{\bbeta})\right) \twonorm{\hbeta^{(k)[r]} - \bbetak{k}}, k\in \ma,\\
			\sup_r\norm{\hat{L}_0^{[r]}(\hbeta^{(k)[r]})-\hat{L}_0^{[r]}(\bbetak{k})} &\leq C \left(1+\sqrt{\frac{1}{n_0}}\cdot \zeta\right)\cdot\exp\left(U\onenorm{\bbetak{k}}\right) \twonorm{\hbeta^{(k)[r]} - \bbetak{k}}, k\in \mac,\\
			\norm{\hat{L}_0^{[r]}(\hbeta^{(0)[r]})-\hat{L}_0^{[r]}(\bbeta)} &\lesssim \exp\left(U\onenorm{\bbeta}\right)\cdot \twonorm{\hbeta^{(0)[r]} - \bbeta}\cdot (1+\zeta),
		\end{align}
		with probability at least $1-\exp\{-\zeta^2\}$.
	\end{enumerate}
\end{lemma}
\begin{remark}
	It's important to point out that based on Algorithm \ref{algo: a unknown}, the randomness of $\hat{L}_0^{[r]}$, $\hbeta^{(k)[r]}$ ($k \neq 0$), and $\hbeta^{(0)[r]}$ is independent. Here $\hbeta^{(k)[r]}$ and $\hbeta^{(0)[r]}$ are regarded as fixed and we only consider the randomness from $\hat{L}_0^{[r]}$.
\end{remark}

\begin{proof}[Proof of Lemma \ref{lem: gamma_1 rate}]
For convenience, we assume $n_0$ is divisible by 3. Note that the term involving $\sum_{i=1}^{n_0/3}\rho(y^{(0)[r]}_i)$ is canceled when we take the difference between $\hat{L}_0^{[r]}(\hbeta^{(k)[r]})$ and $\hat{L}_0^{[r]}(\bbetak{k})$. So in the following, we drop that term from the definition of  $\hat{L}_0^{[r]}$ in equation \eqref{eq: l0 hat def}. We only prove the bound for $\norm{\hat{L}_0(\hbeta^{(k)[r]})-\hat{L}_0(\bbetak{k})}$ when $k \in \ma$. The cases that $k = 0$ or $k \in \mac$ can be similarly discussed. Besides, according to the proof of Theorem \ref{thm: l2 without assumption}, when $k \in \ma$, we define $\bbetak{k} = \frac{2n_0/3}{2n_0/3+n_k}\bbeta + \frac{n_k}{2n_0/3+n_k}\bwk{k}$, which gives us the final results shown in Proposition \ref{prop: gamma_1 gamma_2}.
	\begin{enumerate}[(i)]
		\item For logistic regression model, notice that
		\begin{align}
			\norm{\hat{L}_0(\hbeta^{(k)[r]})-\hat{L}_0(\bbetak{k})} &\leq \frac{1}{n_0/3}\norm{(\by^{(0)[r]})^T \bX^{(0)[r]}(\hbeta^{(k)[r]}-\bbetak{k})} \\
			&\quad + \frac{1}{n_0/3}\norma{\sum_{i=1}^{n_0/3}[\psi((\bx^{(0)[r]}_i)^T\hbeta^{(k)[r]})-\psi((\bx^{(0)[r]}_i)^T\bbetak{k})]}.
		\end{align}
		For the first term on the right-hand side, it holds that
		\begin{equation}
			\frac{1}{n_0/3}\norm{(\by^{(0)[r]})^T \bX^{(0)[r]}(\hbeta^{(k)[r]}-\bbetak{k})} \leq \frac{1}{n_0/3}\sum_{i=1}^{n_0/3}\norm{(\bx_i^{(0)[r]})^T(\hbeta^{(k)[r]}-\bbetak{k})},
		\end{equation}
		where $\frac{1}{n_0/3}\sum_{i=1}^{n_0/3}\norm{(\bx_i^{(0)[r]})^T(\hbetak{k}-\bbetak{k})}$ is a $\frac{1}{n_0/3}\twonorm{\hbetak{k}-\bbetak{k}}^2$ sub-Gaussian with mean less than $C\twonorm{\hbetak{k}-\bbetak{k}}$, where $C>0$ is a uniform constant,
		implying that
		\begin{equation}
			\frac{1}{n_0/3}\norm{(\by^{(0)[r]})^T \bX^{(0)[r]}(\hbeta^{(k)[r]}-\bbetak{k})}  \lesssim \left(1+\sqrt{\frac{1}{n_0}}\cdot \zeta\right)\cdot\twonorm{\hbeta^{(k)[r]}-\bbetak{k}},
		\end{equation}
		with probability at least $1-\exp\{-\zeta^2\}$.
		On the other hand, the second term can be bounded by $\frac{C}{n_0/3}\sum_{i=1}^{n_0/3}\norm{(\bx_i^{(0)[r]})^T(\hbeta^{(k)[r]}-\bbetak{k})}$, which can be similarly bounded as the first term, leading to the desired conclusion.
		\item For linear model, note that $y_i^{(0)[r]} = \bbeta^T\bx^{(0)[r]}_i + \epsilon^{(0)[r]}_i$ and $\psi(u) = u^2/2$, leading to
		\begin{align}
			\norm{\hat{L}_0(\hbeta^{(k)[r]}) - \hat{L}_0(\bbetak{k})} &\leq \frac{1}{n_0/3}\norma{\sum_{i=1}^{n_0/3}\epsilon_i^{(0)[r]}(\bx^{(0)[r]}_i)^T(\hbeta^{(k)[r]}-\bbetak{k})} \\
			&\quad + \frac{1}{n_0/3}\norma{\sum_{i=1}^{n_0/3}(\bx^{(0)[r]}_i)^T \bbeta \cdot (\bx^{(0)[r]}_i)^T(\hbeta^{(k)[r]}-\bbetak{k})} \\
			&\quad + \frac{1}{n_0/3}\norma{\sum_{i=1}^{n_0/3}(\bx^{(0)[r]}_i)^T(\hbeta^{(k)[r]}+\bbeta^{(k)})(\hbeta^{(k)[r]}-\bbeta^{(k)})^T\bx^{(0)[r]}_i}.
		\end{align}
		It is easy to see that $\frac{1}{n_0/3}\sum_{i=1}^{n_0/3}\epsilon_i^{(0)}(\bx^{(0)[r]}_i)^T(\hbeta^{(k)[r]}-\bbetak{k})$ is $\frac{1}{n_0}\twonorm{\hbeta^{(k)[r]}-\bbetak{k}}$-subGaussian with zero mean, while $\frac{1}{n_0/3}\sum_{i=1}^{n_0/3}(\bx^{(0)[r]}_i)^T \bbeta \cdot (\bx^{(0)[r]}_i)^T(\hbeta^{(k)[r]}-\bbetak{k})$ and $\frac{1}{n_0/3}\sum_{i=1}^{n_0/3}(\bx^{(0)[r]}_i)^T(\hbeta^{(k)[r]}+\bbetak{k})(\hbeta^{(k)[r]}-\bbetak{k})^T\bx^{(0)[r]}_i$ are $\frac{1}{n_0/3}\twonorm{\bbetak{k}}\twonorm{\hbeta^{(k)[r]}-\bbetak{k}}$-subexponential with mean at most $C\twonorm{\bbetak{k}}\twonorm{\hbeta^{(k)[r]}-\bbetak{k}}$ \citep{vershynin2018high}, where $C>0$ is a uniform constant, then by tail bounds and union bounds, the conclusion follows.
		\item For Poisson regression model with a.s. bounded covariates, it holds that
		\begin{align}
			\norm{\hat{L}_0^{[r]}(\hbeta^{(k)[r]}) - \hat{L}_0(\bbetak{k})} &\leq \frac{1}{n_0/3}\norma{\sum_{i=1}^{n_0/3}y_i^{(0)[r]}(\bx^{(0)[r]}_i)^T(\hbeta^{(k)[r]}-\bbetak{k})} \\
			&\quad + \frac{1}{n_0/3}\norma{\sum_{i=1}^{n_0/3}[e^{(\bx^{(0)[r]}_i)^T \hbeta^{(k)[r]}}-e^{(\bx^{(0)[r]}_i)^T \bbetak{k}}]}.
		\end{align}
		Conditioning on $\bX^{(0)[r]}$, we know that $y_i^{(0)[r]}(\bx^{(0)[r]}_i)^T(\hbeta^{(k)[r]}-\bbetak{k}) \sim \textup{Poisson}(e^{(\bx^{(0)[r]}_i)^T\bbeta}) \cdot (\bx^{(0)[r]}_i)^T(\hbeta^{(k)[r]}-\bbetak{k})$. By the fact that $\textup{Poisson}(e^{(\bx^{(0)[r]}_i)^T\bbeta})$ is a $e^{(\bx^{(0)[r]}_i)^T\bbeta}$-subexponential given $\bxk{0}_i$, we have
		\begin{align}
			&\frac{1}{n_0/3}\norma{\sum_{i=1}^{n_0/3}y_i^{(0)}(\bx^{(0)[r]}_i)^T(\hbeta^{(k)[r]}-\bbetak{k})} \lesssim \frac{1}{n_0/3}\norma{\sum_{i=1}^{n_0/3}e^{(\bx^{(0)[r]}_i)^T\bbeta}(\bx^{(0)[r]}_i)^T(\hbeta^{(k)[r]}-\bbetak{k})} \\
			&\hspace{4cm} + \frac{1}{\sqrt{n_0/3}}\max_i e^{(\bx^{(0)[r]}_i)^T\bbeta}\sqrt{\sum_{i=1}^{n_0/3}\norm{(\bx^{(0)[r]}_i)^T(\hbeta^{(k)[r]}-\bbetak{k})}^2}\cdot \zeta,
		\end{align}
		with probability at least $1-\exp\{-\zeta^2\}$.
		By H\"older inequality, the first term on the right hand side can be bounded by a $e^{2U\onenorm{\bbeta}}\twonorm{\hbeta^{(k)[r]}-\bbetak{k}}^2$-subexponential with mean at most $Ce^{2U\onenorm{\bbeta}}\twonorm{\hbeta^{(k)[r]}-\bbetak{k}}^2$, leading to
		\begin{equation}
			\frac{1}{n_0/3}\norma{\sum_{i=1}^{n_0/3}e^{(\bx^{(0)[r]}_i)^T\bbeta}(\bx^{(0)[r]}_i)^T(\hbeta^{(k)[r]}-\bbetak{k})} \lesssim \left(1+\sqrt{\frac{1}{n_0}}\cdot \zeta\right)\cdot e^{U\onenorm{\bbeta}}\twonorm{\hbeta^{(k)[r]}-\bbetak{k}},
		\end{equation}
		with probability at least $1-\exp\{-\zeta^2\}$.
		On the other hand, by applying Bernstein inequality (Theorem 2.8.2 in \cite{vershynin2018high}) as well as union bounds, we have
		\begin{align}
			\frac{1}{\sqrt{n_0/3}}\max_i e^{(\bx^{(0)[r]}_i)^T\bbeta}\sqrt{\sum_{i=1}^{n_0/3}\norm{(\bx^{(0)[r]}_i)^T(\hbeta^{(k)[r]}-\bbetak{k})}^2} &\lesssim_p \left(1+\sqrt{\frac{1}{n_0}}\cdot \zeta\right)\\
			&\quad \cdot e^{U\onenorm{\bbeta}}\sup_k\twonorm{\hbeta^{(k)[r]}-\bbetak{k}},
		\end{align}
		with probability at least $1-\exp\{-\zeta^2\}$.
		Summarizing the conclusions before, we obtain the desired conclusion.
		\end{enumerate}
\end{proof}

\begin{lemma}\label{lem: gamma_2 rate}
		Under the same assumptions as Theorem \ref{thm: a detection}, we have the following conclusions:
	\begin{enumerate}[(i)]
		\item For logistic regression model: 
		\begin{equation}
			\norm{\hat{L}_0^{[r]}(\bbeta) - L_0(\bbeta)} \vee \norm{\hat{L}_0^{[r]}(\bbetak{k})-\hat{L}_0^{[r]}(\bbeta)-L_0(\bbetak{k})+L_0(\bbeta)} \leq C\sqrt{\frac{1}{n_0}}\cdot \twonorm{\bwk{k}}\cdot \zeta,
		\end{equation}
		with probability at least $1-\exp\{-\zeta^2\}$.
		\item For linear model:
		\begin{align}
			&\norm{\hat{L}_0^{[r]}(\bbeta) - L_0(\bbeta)} \vee \norm{\hat{L}_0^{[r]}(\bbetak{k})-\hat{L}_0(\bbeta)-L_0(\bbetak{k})+L_0(\bbeta)} \\
			&\quad \leq C\sqrt{\frac{1}{n_0}}\cdot (\twonorm{\bwk{k}}^2 \vee \twonorm{\bwk{k}})\cdot \zeta,
		\end{align}
		with probability at least $1-\exp\{-\zeta^2\}$.
		\item For Poisson regression model with $\sup_k \infnorm{\bxk{k}} \leq U$ a.s.:
		\begin{align}
			&\norm{\hat{L}_0^{[r]}(\bbeta) - L_0(\bbeta)} \vee \norm{\hat{L}_0^{[r]}(\bbetak{k})-\hat{L}_0^{[r]}(\bbeta)-L_0(\bbetak{k})+L_0(\bbeta)} \\
			&\leq C\sqrt{\frac{1}{n_0}}\exp\left(U\onenorm{\bwk{k}}\right) \left[1 
		+ \twonorm{\bwk{k}} + U\onenorm{\bwk{k}} \right]\cdot \zeta,
		\end{align}
		with probability at least $1-\zeta^{-2}$.
	\end{enumerate}
\end{lemma}

\begin{proof}[Proof of Lemma \ref{lem: gamma_2 rate}]
	Similar to the proofs of Lemmas \ref{lem: L a 0} and \ref{lem: gamma_1 rate}, the terms involving $\sum_{i=1}^{n_0/3}\rho(y_i^{(0)[r]})$ and $\te[\rho(y^{(0)})]$ are canceled when taking the difference. Therefore without loss of generality, to prove the rate of $\sup_{k}\norm{\hat{L}_0^{[r]}(\bbeta^{(k)})-\hat{L}_0^{[r]}(\bbeta)-L_0(\bbetak{k})+L_0(\bbeta)}$, throughout this proof, we discard these terms and consider
	\begin{align}
		L_0(\bw) &= -\te[\psi'(\bbeta^T\bxk{0})\bw^T\bxk{0}] + \te[\psi(\bw^T\bxk{0})], \\
		\hat{L}_0^{[r]}(\bw) &= -\frac{1}{n_0/3}\sum_{i=1}^{n_0/3}y_i^{(0)[r]}\bbeta^T\bx_i^{(0)[r]} + \frac{1}{n_0/3}\sum_{i=1}^{n_0/3}\psi(\bw^T\bx^{(0)[r]}_i).
	\end{align}
	\begin{enumerate}[(i)]
		\item For logistic regression model:
		\begin{align}
			&\sup_k\norma{\hat{L}_0^{[r]}(\bbetak{k}) - L_0(\bbetak{k})} \\
			&\leq \frac{1}{n_0/3}\norma{\sum_{i=1}^{n_0/3} [y_i^{(0)[r]} - \psi'(\bbeta^T\bx_i^{(0)[r]})](\bbetak{k})^T\bx_i^{(0)[r]}} \\
			&\quad + \frac{1}{n_0/3}\norma{\sum_{i=1}^{n_0/3}\left[\psi'(\bbeta^T\bx_i^{(0)[r]})(\bbetak{k})^T\bx_i^{(0)[r]} - \te[\psi'(\bbeta^T\bxk{0})(\bbetak{k})^T\bxk{0}]\right]} \\
			&\quad + \frac{1}{n_0/3}\norma{\sum_{i=1}^{n_0/3} \left[\psi((\bbetak{k})^T\bx_i^{(0)[r]}) - \te\psi((\bbetak{k})^T\bx_i^{(0)[r]}) \right]}
		\end{align}
		Since $\frac{1}{n_0/3}\sum_{i=1}^{n_0/3} [y_i^{(0)[r]} - \psi'(\bbeta^T\bx_i^{(0)[r]})](\bbetak{k})^T\bx_i^{(0)[r]}$ is a zero-mean $\twonorm{\bbetak{k}}^2$-subexponential variable, we have
		\begin{equation}\label{eq: prop 1 term 1}
			\frac{1}{n_0/3}\norma{\sum_{i=1}^{n_0/3} [y_i^{(0)[r]} - \psi'(\bbeta^T\bx_i^{(0)[r]})](\bbetak{k})^T\bx_i^{(0)[r]}} \lesssim \sqrt{\frac{1}{n_0}}\twonorm{\bwk{k}}\cdot \zeta,
		\end{equation}
		with probability at least $1-\exp\{-\zeta^2\}$.
		For the second term, since $\psi'$ is bounded, $\psi'(\bbeta^T\bx_i^{(0)[r]})(\bbetak{k})^T\bx_i^{(0)[r]}$ are i.i.d. $\twonorm{\bbetak{k}}^2$-subexponential, leading to
		\begin{equation}\label{eq: prop 1 term 2}
			\frac{1}{n_0/3}\norma{\sum_{i=1}^{n_0/3}\left[\psi'(\bbeta^T\bx_i^{(0)[r]})(\bbetak{k})^T\bx_i^{(0)[r]} - \te[\psi'(\bbeta^T\bxk{0})(\bbetak{k})^T\bxk{0}]\right]} \lesssim \sqrt{\frac{1}{n_0}}\twonorm{\bwk{k}}\cdot \zeta,
		\end{equation}
		with probability at least $1-\exp\{-\zeta^2\}$.
		For the last term, consider $g(u_1^{[r]}, \ldots, u_{n_0/3}^{[r]}) = \sum_{i=1}^{n_0/3}\psi(u_i^{[r]})$, where $u_i^{[r]} = (\bx_i^{(0)[r]})^T\bbetak{k}$ is an i.i.d. $\twonorm{\bbetak{k}}^2$-subGaussian. Since $\psi$ is 1-Lipschitz under $\ell_1$-norm, by Theorem 1 in \cite{kontorovich2014concentration} and union bounds,
		\begin{equation}\label{eq: prop 1 term 3}
			\frac{1}{n_0/3}\norma{\sum_{i=1}^{n_0/3} \left[\psi((\bbetak{k})^T\bx_i^{(0)[r]}) - \te\psi((\bbetak{k})^T\bx^{(0)}) \right]} \lesssim \sqrt{\frac{1}{n_0}}\cdot  \twonorm{\bwk{k}} \cdot \zeta,
		\end{equation}
		with probability at least $1-\exp\{-\zeta^2\}$.
		By \eqref{eq: prop 1 term 1}, \eqref{eq: prop 1 term 2} and \eqref{eq: prop 1 term 3}, the conclusion follows.
		\item For linear model: recall that $y_i^{(0)} = (\bx^{(0)[r]}_i)^T\bbetak{0} + \epsilon_i^{(0)[r]}$, then
		\begin{align}
			&\norma{\hat{L}_0^{[r]}(\bbetak{k}) - L_0(\bbetak{k})} \\
			&\leq \frac{1}{n_0/3}\norma{\sum_{i=1}^{n_0/3}\epsilon_i^{(0)[r]}(\bx^{(0)[r]}_i)^T\bbeta}\\
			&\quad + \frac{1}{n_0/3}\norma{\sum_{i=1}^{n_0/3}(\bx_i^{(0)[r]})^T\bbeta\cdot(\bx_i^{(0)[r]})^T\bbetak{k}-\te[(\bx^{(0)})^T\bbeta\cdot(\bx^{(0)})^T\bbetak{k}]} \\
			&\quad + \frac{1}{2n_0/3}\norma{\sum_{i=1}^{n_0/3}[(\bx_i^{(0)[r]})^T\bbetak{k}]^2 - \te[(\bx^{(0)})^T\bbetak{k}]^2}.
		\end{align}
		By subexponential tail bounds, we have
		\begin{equation}
			\frac{1}{n_0/3}\norma{\sum_{i=1}^{n_0/3}\epsilon_i^{(0)[r]}(\bx_i^{(0)[r]})^T\bbeta} \lesssim \sqrt{\frac{1}{n_0}}\cdot \twonorm{\bbetak{k}} \cdot \zeta \lesssim \sqrt{\frac{1}{n_0}}\cdot \twonorm{\bwk{k}} \cdot \zeta,
		\end{equation}
		\begin{align}
			 &\frac{1}{n_0/3}\norma{\sum_{i=1}^{n_0/3}(\bx^{(0)[r]}_i)^T\bbetak{0}\cdot(\bx^{(0)[r]}_i)^T\bbetak{k}-\te[(\bx^{(0)})^T\bbetak{0}\cdot(\bxk{0})^T\bbetak{k}]}\\
			  &\lesssim \sqrt{\frac{1}{n_0}}\cdot \twonorm{\bbeta}\sup_k\twonorm{\bbetak{k}} \cdot \zeta \lesssim \sqrt{\frac{1}{n_0}}\cdot \twonorm{\bbeta}\sup_k\twonorm{\bwk{k}},
		\end{align}
		\begin{align}
			 \frac{1}{2n_0/3}\norma{\sum_{i=1}^{n_0/3}[(\bx^{(0)[r]}_i)^T\bbetak{k}]^2 - \te[(\bx^{(0)})^T\bbetak{k}]^2} &\lesssim \sqrt{\frac{1}{n_0}}\cdot \twonorm{\bbetak{k}}^2 \cdot \zeta,
		\end{align}
		with probability at least $1-\exp\{-\zeta^2\}$, which leads to the desired conclusion.
		\item For Poisson regression model: similar to the logistic regression model, it holds that
		\begin{align}
			&\norma{\hat{L}_0^{[r]}(\bbetak{k}) - L_0(\bbetak{k})} \\
			&\leq \frac{1}{n_0/3}\norma{\sum_{i=1}^{n_0/3} [y_i^{(0)[r]} - e^{\bbeta^T\bx_i^{(0)[r]}}](\bbetak{k})^T\bx_i^{(0)[r]}} \\
			&\quad + \frac{1}{n_0/3}\norma{\sum_{i=1}^{n_0/3}\left[e^{\bbeta^T\bx_i^{(0)[r]}}(\bbetak{k})^T\bx_i^{(0)[r]} - \te[e^{\bbeta^T\bxk{0}}(\bbetak{k})^T\bxk{0}]\right]} \\
			&\quad + \frac{1}{n_0/3}\norma{\sum_{i=1}^{n_0/3} \left[e^{(\bbetak{k})^T\bx_i^{(0)[r]}} - \te e^{(\bbetak{k})^T\bx^{(0)}} \right]}.
		\end{align}
		For the first term on the right-hand side, because $[y_i^{(0)[r]} - e^{\bbeta^T\bx_i^{(0)[r]}}](\bbetak{k})^T\bx_i^{(0)[r]}$ is an i.i.d. zero-mean $e^{U\onenorm{\bbeta}}\twonorm{\bbetak{k}}$-subexponential variable, it follows
		\begin{equation}
			\frac{1}{n_0/3}\norma{\sum_{i=1}^{n_0/3} [y_i^{(0)[r]} - e^{\bbeta^T\bx_i^{(0)[r]}}](\bbetak{k})^T\bx_i^{(0)[r]}}	 \lesssim \sqrt{\frac{1}{n_0}}e^{U\onenorm{\bbeta}}\twonorm{\bbetak{k}}\cdot \zeta,	
		\end{equation}
		with probability at least $1-\exp\{-\zeta^2\}$.
		For the last term on the right-hand side, note that $\exp\{(\bbetak{k})^T\bx_i^{(0)[r]}\}$ is bounded by $\exp\{U\onenorm{\bbetak{k}}\}$. Therefore by tail probability,
		\begin{equation}
			 \frac{1}{n_0/3}\norma{\sum_{i=1}^{n_0/3} \left[e^{(\bbetak{k})^T\bx_i^{(0)[r]}} - \te e^{(\bbetak{k})^T\bx^{(0)}} \right]} \lesssim  \exp\{U\onenorm{\bbetak{k}}\}\cdot \sqrt{\frac{1}{n_0}}\cdot \zeta,
		\end{equation}
		with probability at least $1-\exp\{-\zeta^2\}$. Denote $u_i^{(k)[r]} = (\bbetak{k})^T\bx^{(0)[r]}_i$.
		Finally, to bound the second term on the right-hand side, we follow the same idea to get
		\begin{align}
			\frac{1}{n_0/3}\norma{\sum_{i=1}^{n_0/3}u_i^{(k)[r]}e^{u_i^{(0)[r]}} - \te u_i^{(k)[r]}e^{u_i^{(0)[r]}}} \lesssim U\onenorm{\bbetak{k}} \cdot \exp\{U\onenorm{\bbetak{k}}\}\cdot \sqrt{\frac{1}{n_0}}\cdot \zeta,
		\end{align}
		with probability at least $1-\exp\{-\zeta^2\}$.
		By combining all the conclusions above, we obtain the desired bound.
		
		The remaining task is to calculate the rate of $\norm{\hat{L}_0^{[r]}(\bbeta) - L_0(\bbeta)}$ under three scenarios. For logistic case, since $\rho = 0$, the calculation in (\rom{1}) naturally follows and the same bound can be derived. For the linear case, we only have to show that
		\begin{equation}
			\norma{\frac{1}{n_0/3}\sum_{i=1}^{n_0/3}(y_i^{(0)[r]})^2 - \te (y^{(0)})^2} \lesssim \sqrt{\frac{1}{n_0}}\cdot(\twonorm{\bbetak{k}}^2 \vee \twonorm{\bbetak{k}}),
		\end{equation}
		with probability at least $1-\exp\{-n_0\}$.
		This can be easily checked by considering $y_i^{(0)[r]} = \bbeta^T\bx_i^{(0)[r]}+\epsilon_i^{(0)[r]}$ and applying subexponential tail bounds. For Poisson regression model, notice that
		\begin{equation}
			\var(\log (y_i^{(0)[r]}!)) \leq \te[(y_i^{(0)[r]})^2\log^2 y_i^{(0)[r]}] \leq \sqrt{\te[(y_i^{(0)[r]})^4]}\sqrt{\te(\log^4 y_i^{(0)[r]})}.
		\end{equation}
		Due to moment bounds of subexponential variables and Jensen's inequality:
		\begin{align}
			\te[(y_i^{(0)[r]})^4] &\lesssim \te_{\bx}\left[\te_{y^{(0)}|\bx}(y_i^{(0)[r]})^4\right] \leq \exp(4U\onenorm{\bbeta}),\\
			\te(\log^4 y_i^{(0)[r]}) &\leq \log^4 \te y_i^{(0)[r]} \lesssim U^4\onenorm{\bbeta}^4.
		\end{align}
		Then by Chebyshev inequality and union bounds, it's straightforward to prove that
		\begin{equation}
			\norma{\frac{1}{n_0/3}\sum_{i=1}^{n_0/3}y_i^{(0)[r]} - \te y^{(0)}} \lesssim_p \sqrt{\frac{1}{n_0}}U\onenorm{\bbeta}\exp(U\onenorm{\bbeta})\cdot \zeta,
		\end{equation}
		with probability at least $1-\zeta^{-2}$, which completes our proof.
		\end{enumerate}
\end{proof}

\subsubsection{Proof of Proposition \ref{prop: gamma rate}}\label{subsubsec: proof of prop 2}
As analyzed in the proof of Theorem \ref{thm: inf}, it remains to show \eqref{eq: prop 2 eq 3}. Recall that $\htau_j^2 = \hSigma_{\hbeta, j, j} - \hSigma_{\hbeta, j, -j}\hbgammak{0}_j = (\nah+n_0)^{-1}\sum_{k \in \transet}(\bXk{k}_j)^T\diag(\{\psi''((\bxk{k}_i)^T\hbeta)\}_{i=1}^{n_k})\bXk{k}_j - (\nah+n_0)^{-1}\sum_{k \in \transet}(\bXk{k}_j)^T\diag(\{\psi''((\bxk{k}_i)^T\hbeta)\}_{i=1}^{n_k})\bXk{k}_{-j}\hbgammak{0}_j$, $\quad\tau_j^2 = \bSigmak{0}_{\bbeta, j, j} - \bSigmak{0}_{\bbeta, j, -j}\bgammak{0}_j = \allowbreak \te[(x_j^{(0)})^2\psi''((\bxk{0})^T\bbeta)] - \te[x_j^{(0)}\psi''((\bxk{0})^T\bbeta)(x_{-j}^{(0)})^T]\bgammak{0}_j = (\bTheta_{j, j})^{-1}$ stays away from zero because of Assumption \ref{asmp: inf}.(\rom{3}). And we have
\begin{align}
	&\htau_j^2 - \tau_j^2 \\
	&= \underbrace{\frac{1}{\nah+n_0}\sum_{i, k}(x_{ij}^{(k)})^2[\psi''((\bxk{k}_i)^T\hbeta) - \psi''((\bxk{k}_i)^T\bbeta)]}_\text{(1)} \\
	&\quad + \underbrace{\frac{1}{\nah+n_0}\sum_{i, k}(x_{ij}^{(k)})^2\psi''((\bxk{k}_i)^T\bbeta) - \sum_{k \in \transet}\frac{n_k}{\nah+n_0}\te[(x_{ij}^{(k)})^2\psi''((\bxk{k}_i)^T\bbeta)]}_\text{(2)} \\
	&\quad + \underbrace{\sum_{k \in \transet}\frac{n_k}{\nah+n_0}\left\{\te[(x_{ij}^{(k)})^2\psi''((\bxk{k}_i)^T\bbeta)]-\te[(x_{ij}^{(0)})^2\psi''((\bxk{0}_i)^T\bbeta)]\right\}}_\text{(3)}\\
	&\quad + \underbrace{\frac{1}{\nah+n_0}\sum_{k \in \transet} (\bXk{k}_j)^T\diag(\{\psi''((\bxk{k}_i)^T\bbeta) - \psi''((\bxk{k}_i)^T\hbeta)\}_{i=1}^{n_k})\bXk{k}_{-j}\hbgammak{0}_j}_\text{(4)} \\
	&\quad - \underbrace{\frac{1}{\nah+n_0}\sum_{k \in \transet}(\bXk{k}_j)^T\diag(\{\psi''((\bxk{k}_i)^T\bbeta)\}_{i=1}^{n_k})\bXk{k}_{-j}(\hbgammak{0}_j - \bgammak{0}_j)}_\text{(5)} \\
	&\quad + \underbrace{\frac{1}{\nah+n_0}\sum_{k \in \transet}\left\{-(\bXk{k}_j)^T\diag(\{\psi''((\bxk{k}_i)^T\bbeta)\}_{i=1}^{n_k})\bXk{k}_{-j}\bgammak{0}_j + \te[x_j^{(k)}\psi''((\bxk{k})^T\bbeta)(x_{-j}^{(k)})^T]\bgammak{0}_j\right\}}_\text{(6)} \\
	&\quad + \underbrace{\sum_{k \in \transet}\frac{n_k}{\nah+n_0}\left\{-\te[x_j^{(k)}\psi''((\bxk{k})^T\bbeta)(x_{-j}^{(k)})^T] + \te[x_j^{(0)}\psi''((\bxk{0})^T\bbeta)(x_{-j}^{(0)})^T]\right\}\bgammak{0}_j}_\text{(7)}.
\end{align}
And we have the following control for each term:
\begin{align}
	\norm{\text{(1)}} &\leq \frac{1}{\nah+n_0}\norma{\sum_{i,k}(x_{ij}^{(k)})^2\psi'''((\bxk{k}_i)^T\bbeta + t(\bxk{k}_i)^T(\hbeta-\bbeta)) (\bxk{k}_i)^T(\hbeta-\bbeta)} \\
	&\leq \sqrt{\frac{1}{\nah+n_0}\sum_{i,k}(x_{ij}^{(k)})^4[\psi'''((\bxk{k}_i)^T\bbeta + t(\bxk{k}_i)^T(\hbeta-\bbeta))]^2}\cdot \sqrt{\frac{1}{\nah+n_0}\sum_{i,k}[(\bxk{k}_i)^T(\hbeta-\bbeta))]^2} \\
	&\lesssim \Re_1,
\end{align}
with probability at least $1-n_0^{-1}$.

Since $\{(x_{ij}^{(k)})^2\psi''((\bxk{k}_i)^T\bbeta)\}_{i,k}$ are independent sub-Gaussian variables with finite variance, by concentration inequality, $\norm{\text{(2)}} \lesssim (\nah+n_0)^{-1/2}$ with probability $1-\exp\{-(\nah+n_0)\}$. Similarly, $\norm{\text{(6)}} \lesssim (\nah+n_0)^{-1/2}$ with probability $1-\exp\{-(\nah+n_0)\}$. 
\begin{equation}
	\norm{\text{(3)}} \leq \sup_{k \in \mah}\norma{\bSigmak{k}_{\bbeta,j,j} - \bSigmak{0}_{\bbeta,j,j}} \leq h_{\max}.
\end{equation}

\begin{align}
	\norm{\text{(4)}} &\leq \norma{\frac{1}{\nah+n_0} \sum_{i,k}x_{ij}^{(k)}(\bxk{k}_{-j})^T\hbgammak{0}_j\cdot[\psi''((\bxk{k}_i)^T\bbeta) - \psi''((\bxk{k}_i)^T\hbeta)]} \\
	&\leq \norma{\frac{1}{\nah+n_0} \sum_{i,k}x_{ij}^{(k)}(\bxk{k}_{-j})^T\bgammak{0}_j\cdot \psi'''((\bxk{k}_i)^T\bbeta + t(\bxk{k}_i)^T(\hbeta-\bbeta))\cdot (\bxk{k}_i)^T(\hbeta-\bbeta)} \\
	&\quad + \norma{\frac{1}{\nah+n_0} \sum_{i,k}x_{ij}^{(k)}(\bxk{k}_{-j})^T(\hbgammak{0}_j-\bgammak{0}_j)\cdot \psi'''((\bxk{k}_i)^T\bbeta + t(\bxk{k}_i)^T(\hbeta-\bbeta))\cdot (\bxk{k}_i)^T(\hbeta-\bbeta)} \\
	&\lesssim \Re_1,
\end{align}
with probability at least $1-n_0^{-1}$.

\begin{align}
	\norm{\text{(5)}} &\leq \norma{\frac{1}{\nah+n_0} \sum_{i,k}x_{ij}^{(k)}\psi''((\bxk{k}_i)^T\bbeta)\cdot (\bxk{k}_{i,-j})^T(\hbgammak{0}_j-\bgammak{0}_j)} \\
	&\lesssim \sqrt{\frac{1}{\nah+n_0} \sum_{i,k}[(\bxk{k}_{i,-j})^T(\hbgammak{0}_j-\bgammak{0}_j)]^2} \\
	&\lesssim  h_1^{1/2}\left(\frac{\log p}{n_0}\right)^{1/4} + h_1^{1/2}\Re_1^{1/2} + \Re_1
\end{align}
with probability at least $1-\Kah n_0^{-1}$ by \eqref{eq: prop 2 eq 1}. 

\begin{equation}
	\norm{\text{(7)}} \lesssim  \sup_{k \in \mah}\norma{(\bSigmak{k}_{\bbeta, -j, j}-\bSigmak{0}_{\bbeta, -j, j})\bgammak{0}_j} \lesssim h_{\max}.
\end{equation}

Combine all the inequalities above to finish the proof of \eqref{eq: prop 2 eq 3}. Note that the bound of $\norm{\htau_j^{-2} - \tau_j^{-2}}$ follows because $\inf_{j}\tau_j^2 = \mathcal{O}(1)$.

\end{document}